%% file: main.tex
\documentclass{article}

\usepackage[preprint]{ProbNum26} 

\usepackage[round]{natbib}

\input{preambles/packages}
\input{preambles/defns}
\input{preambles/tikz}

\usepackage{hyperref}
\usepackage[capitalise,nameinlink]{cleveref}
\crefname{thm}{Theorem}{Theorems}
\crefname{defn}{Definition}{Definitions}

\newcommand{\ml}[2][] {\todo[inline,backgroundcolor=green!20!white, #1]{(Mykola) #2}}

\newcommand{\wouterk}[2][] {\todo[inline,backgroundcolor=red!20!white, #1]{(WouterK) #2}}

\probnumtitle{Composing Non-Conjugate Factor Graphs with Closed-Form Variational Inference}

\probnumauthors{%
\name{Mykola Lukashchuk}%
\affiliation{Eindhoven University of Technology, the Netherlands}
\and%
\name{Kyrylo Yemets}%
\affiliation{Lviv Polytechnic National University, Lviv, Ukraine}
\and%
\name{Wouter M. Kouw}%
\affiliation{ Eindhoven University of Technology, the Netherlands}
\and%
\name{Dmitry Bagaev}%
\affiliation{Eindhoven University of Technology, the Netherlands}
\and%
\name{{\.{I}}smail {\c{S}en\"{o}z}}%
\affiliation{Lazy Dynamics, Utrecht, the Netherlands}
\and%
\name{Jeff Beck}%
\affiliation{Great Sky, Boulder, Colorado, USA}
\and%
\name{Bert de Vries}%
\affiliation{
  Eindhoven University of Technology, the Netherlands}
}

\probnumabstract{%
Stacking probabilistic building blocks into deeper architectures typically breaks closed-form inference. We show that closed-form inference can be preserved. We identify five factor-graph primitives: a bilinear factor, an exponential link, a Gamma prior, a Gaussian likelihood, and an equality node, and prove that any model composed from them admits closed-form variational message passing. The construction works because each primitive preserves a small set of message families: under mean-field factorization, messages on Gaussian variables remain Gaussian and messages on precision variables remain Gamma, while the only non-conjugate interface, the exponential link, remains tractable through the Gaussian moment-generating function and the sufficient statistics of the Gamma family. We demonstrate composition at increasing depth, from static ensembles through input-dependent gating to split-branch routing, and show that stacking routing layers encodes arbitrary decision trees, establishing universal function approximation with closed-form inference. Applied to ensemble time-series forecasting, the framework yields a Bayesian mixture of experts in which gating functions are inferred rather than learned, providing calibrated uncertainty over expert selection across five benchmark datasets.
}

\let\probnumoriginalmaketitle\probnummaketitle
\renewcommand{\probnummaketitle}{%
    \probnumoriginalmaketitle
    \begin{center}
        \small\sffamily
        \textbf{Accepted to ProbNum 2026.}
        OpenReview record will be accessible after the camera-ready deadline:
        \url{https://openreview.net/forum?id=fi4qX94LI8}.
    \end{center}
    \vspace*{2mm}
}

\begin{document}

\section{Introduction}\label{sec:intro}
\input{sections/new_introduction}

\section{The Alphabet}\label{sec:alphabet}
\input{sections/alphabet}

\section{The Grammar}\label{sec:grammar}
\input{sections/grammar}

\section{The Runtime}\label{sec:runtime}
\input{sections/runtime}

\section{Expressiveness}\label{sec:expressiveness}
\input{sections/expressiveness}

\section{Application}\label{sec:application}
\input{sections/application}

\section{Related Work}\label{sec:discussion}
\input{sections/new_discussion}

\section{Conclusion}\label{sec:conclusion}
\input{sections/conclusion}

\section*{Acknowledgements}

We gratefully acknowledge financial support by the Dutch Ministry of Economic Affairs (PPS funding), by the Dutch Research Council (NWO) and by hearing aid manufacturer GN Hearing, under contracts \\ TKI-HTSM/21.0161/2112P09 (project: Auto-AR) and KICH3.LTP.20.006 (Project: ROBUST).

\FloatBarrier

\bibliography{references}

\appendix

\section{Forney-style Factor Graphs and Message Passing}\label{app:ffg}
\input{sections/appendix_ffg}


\section{XOR Encoding Parameters}\label{app:xor}
\input{sections/appendix_xor}

\section{Comparison with Neural Universal Approximation}\label{app:neural}
\input{sections/appendix_neural}

\section{Models in GraphPPL.jl}\label{app:graphppl}
\input{sections/appendix_graphppl}

\section{Experimental Details}\label{app:experimental}
\input{sections/appendix_experimental}

\section{Full Results}
\input{sections/full_results}



\end{document}

%% file: preambles/packages.tex
\usepackage{amssymb}
\usepackage[T1]{fontenc}
\usepackage[utf8]{inputenc}

\usepackage[american]{babel}
\usepackage{multirow}

\usepackage{mathtools} 
\usepackage{bm} 
\usepackage{amsthm}

\newtheorem{proposition}{Proposition}
\newtheorem{corollary}{Corollary}
\usepackage[disable]{todonotes}
\usepackage{soul}
\usepackage{subcaption}
\usepackage{placeins}
\usepackage[frozencache]{minted}
\setminted[julia]{
  fontsize=\scriptsize,
  frame=lines,
  framesep=2pt,
  breaklines=true,
  numbersep=5pt,
  xleftmargin=1.5em,
}

%% file: preambles/defns.tex
\newcommand\cut[1]{}

\newcommand{\Exp}[2]{\mathbb{E}_{#1}\sqr{#2}}
\newcommand{\Ent}[1]{\mathbb{H}\sqr{#1}}

\newcommand{\squishlist}{
   \begin{list}{$\bullet$}
    { \setlength{\itemsep}{0pt}      \setlength{\parsep}{3pt}
      \setlength{\topsep}{3pt}       \setlength{\partopsep}{0pt}
      \setlength{\leftmargin}{1.5em} \setlength{\labelwidth}{1em}
      \setlength{\labelsep}{0.5em} } }

\newcommand{\squishlisttwo}{
   \begin{list}{$\bullet$}
    { \setlength{\itemsep}{0pt}    \setlength{\parsep}{0pt}
      \setlength{\topsep}{0pt}     \setlength{\partopsep}{0pt}
      \setlength{\leftmargin}{2em} \setlength{\labelwidth}{1.5em}
      \setlength{\labelsep}{0.5em} } }

\newcommand{\squishend}{
    \end{list}  }

{}
\newtheorem{thm}{Theorem}{}
{}
{}
{}
\newtheorem{defn}{Definition}{}

\newcommand{\sqr}[1]{\left[#1\right]}

\newcommand{\gauss}{\mbox{${\cal N}$}}
\newcommand{\gammad}{\mbox{${\cal G}$}}

\newcommand{\myvec}[1]{\mbox{$\mathbf{#1}$}}
\newcommand{\myvecsym}[1]{\mbox{$\boldsymbol{#1}$}}

\newcommand{\veta}{\mbox{$\myvecsym{\eta}$}}
\newcommand{\vgamma}{\mbox{$\myvecsym{\gamma}$}}
\newcommand{\vmu}{\mbox{$\myvecsym{\mu}$}}

\newcommand{\vphi}{\mbox{$\myvecsym{\phi}$}}

\newcommand{\vtau}{\mbox{$\myvecsym{\tau}$}}

\newcommand{\vm}{\mbox{$\myvec{m}$}}

\newcommand{\vs}{\mbox{$\myvec{s}$}}

\newcommand{\vu}{\mbox{$\myvec{u}$}}
\newcommand{\vv}{\mbox{$\myvec{v}$}}
\newcommand{\vw}{\mbox{$\myvec{w}$}}

\newcommand{\vx}{\mbox{$\myvec{x}$}}

\newcommand{\vy}{\mbox{$\myvec{y}$}}

\newcommand{\vz}{\mbox{$\myvec{z}$}}

\newcommand{\vF}{\mbox{$\myvec{F}$}}

\newcommand{\vT}{\mbox{$\myvec{T}$}}

\newcommand{\vV}{\mbox{$\myvec{V}$}}

\newcommand{\be}{\begin{equation}}
\newcommand{\ee}{\end{equation}}
\newcommand{\bea}{\begin{eqnarray}}
\newcommand{\eea}{\end{eqnarray}}
\newcommand{\beaa}{\begin{eqnarray*}}
\newcommand{\eeaa}{\end{eqnarray*}}

\DeclareMathOperator*{\argmin}{arg\,min}

\newcommand{\R}{\mathbb{R}}
\newcommand{\mI}{\mbox{$\myvec{I}$}} 

%% file: preambles/tikz.tex
\usepackage{pgfplots}
\usepackage{pgfplotstable}
\pgfplotsset{compat=newest}
\usetikzlibrary{intersections}

\usetikzlibrary{positioning}
\usepgfplotslibrary{external}
\usepgfplotslibrary{fillbetween}
\pgfplotsset{compat=1.5}
\usetikzlibrary{pgfplots.groupplots}
\usetikzlibrary{shapes.geometric,backgrounds}
\usetikzlibrary{calc,arrows.meta,positioning}
\usepackage{adjustbox}
\usetikzlibrary{backgrounds}
\pgfdeclarelayer{background}
\pgfsetlayers{background,main}

\tikzset{
    every node/.style={font=\sffamily\small},
    main node/.style={thick,circle,draw,font=\sffamily\Large}
}

\usepackage{xifthen}
\usepackage{tkz-euclide}
\usetikzlibrary{positioning}
\usepgfplotslibrary{fillbetween}
\usepgfplotslibrary{groupplots}
\usetikzlibrary{shapes.geometric,backgrounds}
\usepackage[outline]{contour}

\usetikzlibrary{external}
\pgfdeclarelayer{bg}    
\pgfsetlayers{bg,main}  

\usetikzlibrary{calc, arrows, fit, positioning, patterns, decorations.pathreplacing, shapes}
\tikzstyle{dash} = [dashed, -latex,>=latex]
\tikzstyle{line} = [draw, -latex,>=latex]
\tikzstyle{smallbox} = [draw, minimum size=5.0mm]
\tikzstyle{box} = [draw, minimum size=7.0mm]
\tikzstyle{bigbox} = [draw, minimum size=10.0mm]
\tikzstyle{rectangle} = [draw, minimum width=10.0mm, minimum height=20.0mm]
\tikzstyle{switch} = [trapezium, trapezium angle=120, draw, rotate=90,  inner ysep=5pt, outer sep=5pt,
minimum height=7mm, minimum width=7mm]
\tikzstyle{roundbox} = [draw, circle, inner sep=0pt, minimum size=3mm]
\tikzstyle{clamped} = [draw, fill=darkgrey, minimum size=0.15cm]
\tikzstyle{msgcircle} = [shape=circle, draw, inner sep=0pt, minimum size=4mm, fill=white, font=\scriptsize]
\tikzstyle{darkmsgcircle} = [shape=circle, draw, inner sep=0pt, minimum size=4mm, fill=darkgrey, text=white, font=\scriptsize]
\tikzstyle{redmsgcircle} = [shape=circle, draw=red, inner sep=0pt, minimum size=4mm, text=red, font=\scriptsize]
\tikzstyle{reddarkmsgcircle} = [shape=circle, draw=red, inner sep=0pt, minimum size=4mm, fill=red, text=white, font=\scriptsize]
\tikzstyle{msgdoublecircle} = [shape=circle, double, double distance=1.5pt, draw, inner sep=0pt, minimum size=5mm, fill=white]
\tikzstyle{darkmsgdoublecircle} = [shape=circle, double, double distance=1.5pt, draw, inner sep=0pt, minimum size=5mm, fill=darkgrey, text=white, font=\bfseries]

\newcommand{\msg}[6]{
      \ifthenelse{\isin{#1}{left} \AND \isin{#2}{down}}{
            \coordinate (anchor) at ($({#3})!{#5}!({#4})$);
            \node[xshift=-6.0mm] at (anchor) {#6};
            \node[xshift=-1.0mm] at (anchor) {$\downarrow$};
      }{}
      \ifthenelse{\isin{#1}{right} \AND \isin{#2}{down}}{
            \coordinate (anchor) at ($({#3})!{#5}!({#4})$);
            \node[xshift=6.0mm] at (anchor) {#6};
            \node[xshift=1.0mm] at (anchor) {$\downarrow$};
      }{}

      \ifthenelse{\isin{#1}{down} \AND \isin{#2}{right}}{
            \coordinate (anchor) at ($({#3})!{#5}!({#4})$);
            \node[ yshift=-4.0mm] at (anchor) {#6};
            \node[yshift=-1.0mm] at (anchor) {$\rightarrow$};
      }{}
      \ifthenelse{\isin{#1}{up} \AND \isin{#2}{right}}{
            \coordinate (anchor) at ($({#3})!{#5}!({#4})$);
            \node[ yshift=4.0mm] at (anchor) {#6};
            \node[yshift=1.0mm] at (anchor) {$\rightarrow$};
      }{}

      \ifthenelse{\isin{#1}{down} \AND \isin{#2}{left}}{
            \coordinate (anchor) at ($({#3})!{#5}!({#4})$);
            \node[ yshift=-4.0mm] at (anchor) {#6};
            \node[yshift=-1.0mm] at (anchor) {$\leftarrow$};
      }{}
      \ifthenelse{\isin{#1}{up} \AND \isin{#2}{left}}{
            \coordinate (anchor) at ($({#3})!{#5}!({#4})$);
            \node[ yshift=4.0mm] at (anchor) {#6};
            \node[yshift=1.0mm] at (anchor) {$\leftarrow$};
      }{}

      \ifthenelse{\isin{#1}{left} \AND \isin{#2}{up}}{
            \coordinate (anchor) at ($({#3})!{#5}!({#4})$);
            \node[ xshift=-6.0mm] at (anchor) {#6};
            \node[xshift=-1.0mm] at (anchor) {$\uparrow$};
      }{}
      \ifthenelse{\isin{#1}{right} \AND \isin{#2}{up}}{
            \coordinate (anchor) at ($({#3})!{#5}!({#4})$);
            \node[ xshift=6.0mm] at (anchor) {#6};
            \node[xshift=1.0mm] at (anchor) {$\uparrow$};
      }{}
}

\newcommand{\darkmsg}[6]{
      \ifthenelse{\isin{#1}{left} \AND \isin{#2}{down}}{
            \coordinate (anchor) at ($({#3})!{#5}!({#4})$);
            \node[darkmsgcircle, xshift=-5.5mm] at (anchor) {#6};
            \node[xshift=-1.5mm] at (anchor) {$\downarrow$};
      }{}
      \ifthenelse{\isin{#1}{right} \AND \isin{#2}{down}}{
            \coordinate (anchor) at ($({#3})!{#5}!({#4})$);
            \node[darkmsgcircle, xshift=5.5mm] at (anchor) {#6};
            \node[xshift=1.5mm] at (anchor) {$\downarrow$};
      }{}

      \ifthenelse{\isin{#1}{down} \AND \isin{#2}{right}}{
            \coordinate (anchor) at ($({#3})!{#5}!({#4})$);
            \node[darkmsgcircle, yshift=-6.0mm] at (anchor) {#6};
            \node[yshift=-2.0mm] at (anchor) {$\rightarrow$};
      }{}
      \ifthenelse{\isin{#1}{up} \AND \isin{#2}{right}}{
            \coordinate (anchor) at ($({#3})!{#5}!({#4})$);
            \node[darkmsgcircle, yshift=6.0mm] at (anchor) {#6};
            \node[yshift=2.0mm] at (anchor) {$\rightarrow$};
      }{}

      \ifthenelse{\isin{#1}{down} \AND \isin{#2}{left}}{
            \coordinate (anchor) at ($({#3})!{#5}!({#4})$);
            \node[darkmsgcircle, yshift=-6.0mm] at (anchor) {#6};
            \node[yshift=-2.0mm] at (anchor) {$\leftarrow$};
      }{}
      \ifthenelse{\isin{#1}{up} \AND \isin{#2}{left}}{
            \coordinate (anchor) at ($({#3})!{#5}!({#4})$);
            \node[darkmsgcircle, yshift=6.0mm] at (anchor) {#6};
            \node[yshift=2.0mm] at (anchor) {$\leftarrow$};
      }{}

      \ifthenelse{\isin{#1}{left} \AND \isin{#2}{up}}{
            \coordinate (anchor) at ($({#3})!{#5}!({#4})$);
            \node[darkmsgcircle, xshift=-5.5mm] at (anchor) {#6};
            \node[xshift=-1.5mm] at (anchor) {$\uparrow$};
      }{}
      \ifthenelse{\isin{#1}{right} \AND \isin{#2}{up}}{
            \coordinate (anchor) at ($({#3})!{#5}!({#4})$);
            \node[darkmsgcircle, xshift=5.5mm] at (anchor) {#6};
            \node[xshift=1.5mm] at (anchor) {$\uparrow$};
      }{}
}

\newcommand{\bwmsg}[6]{
      \ifthenelse{\isin{#1}{left} \AND \isin{#2}{down}}{
            \coordinate (anchor) at ($({#3})!{#5}!({#4})$);
            \node[msgdoublecircle, xshift=-5.5mm] at (anchor) {#6};
            \node[xshift=-1.5mm] at (anchor) {$\downarrow$};
      }{}
      \ifthenelse{\isin{#1}{right} \AND \isin{#2}{down}}{
            \coordinate (anchor) at ($({#3})!{#5}!({#4})$);
            \node[msgdoublecircle, xshift=5.5mm] at (anchor) {#6};
            \node[xshift=1.5mm] at (anchor) {$\downarrow$};
      }{}

      \ifthenelse{\isin{#1}{down} \AND \isin{#2}{right}}{
            \coordinate (anchor) at ($({#3})!{#5}!({#4})$);
            \node[msgdoublecircle, yshift=-6.0mm] at (anchor) {#6};
            \node[yshift=-2.0mm] at (anchor) {$\rightarrow$};
      }{}
      \ifthenelse{\isin{#1}{up} \AND \isin{#2}{right}}{
            \coordinate (anchor) at ($({#3})!{#5}!({#4})$);
            \node[msgdoublecircle, yshift=6.0mm] at (anchor) {#6};
            \node[yshift=2.0mm] at (anchor) {$\rightarrow$};
      }{}

      \ifthenelse{\isin{#1}{down} \AND \isin{#2}{left}}{
            \coordinate (anchor) at ($({#3})!{#5}!({#4})$);
            \node[msgdoublecircle, yshift=-6.0mm] at (anchor) {#6};
            \node[yshift=-2.0mm] at (anchor) {$\leftarrow$};
      }{}
      \ifthenelse{\isin{#1}{up} \AND \isin{#2}{left}}{
            \coordinate (anchor) at ($({#3})!{#5}!({#4})$);
            \node[msgdoublecircle, yshift=6.0mm] at (anchor) {#6};
            \node[yshift=2.0mm] at (anchor) {$\leftarrow$};
      }{}

      \ifthenelse{\isin{#1}{left} \AND \isin{#2}{up}}{
            \coordinate (anchor) at ($({#3})!{#5}!({#4})$);
            \node[msgdoublecircle, xshift=-5.5mm] at (anchor) {#6};
            \node[xshift=-1.5mm] at (anchor) {$\uparrow$};
      }{}
      \ifthenelse{\isin{#1}{right} \AND \isin{#2}{up}}{
            \coordinate (anchor) at ($({#3})!{#5}!({#4})$);
            \node[msgdoublecircle, xshift=5.5mm] at (anchor) {#6};
            \node[xshift=1.5mm] at (anchor) {$\uparrow$};
      }{}
}

\newcommand{\bwdarkmsg}[6]{
      \ifthenelse{\isin{#1}{left} \AND \isin{#2}{down}}{
            \coordinate (anchor) at ($({#3})!{#5}!({#4})$);
            \node[darkmsgdoublecircle, xshift=-5.5mm] at (anchor) {#6};
            \node[xshift=-1.5mm] at (anchor) {$\downarrow$};
      }{}
      \ifthenelse{\isin{#1}{right} \AND \isin{#2}{down}}{
            \coordinate (anchor) at ($({#3})!{#5}!({#4})$);
            \node[darkmsgdoublecircle, xshift=5.5mm] at (anchor) {#6};
            \node[xshift=1.5mm] at (anchor) {$\downarrow$};
      }{}

      \ifthenelse{\isin{#1}{down} \AND \isin{#2}{right}}{
            \coordinate (anchor) at ($({#3})!{#5}!({#4})$);
            \node[darkmsgdoublecircle, yshift=-6.0mm] at (anchor) {#6};
            \node[yshift=-2.0mm] at (anchor) {$\rightarrow$};
      }{}
      \ifthenelse{\isin{#1}{up} \AND \isin{#2}{right}}{
            \coordinate (anchor) at ($({#3})!{#5}!({#4})$);
            \node[darkmsgdoublecircle, yshift=6.0mm] at (anchor) {#6};
            \node[yshift=2.0mm] at (anchor) {$\rightarrow$};
      }{}

      \ifthenelse{\isin{#1}{down} \AND \isin{#2}{left}}{
            \coordinate (anchor) at ($({#3})!{#5}!({#4})$);
            \node[darkmsgdoublecircle, yshift=-6.0mm] at (anchor) {#6};
            \node[yshift=-2.0mm] at (anchor) {$\leftarrow$};
      }{}
      \ifthenelse{\isin{#1}{up} \AND \isin{#2}{left}}{
            \coordinate (anchor) at ($({#3})!{#5}!({#4})$);
            \node[darkmsgdoublecircle, yshift=6.0mm] at (anchor) {#6};
            \node[yshift=2.0mm] at (anchor) {$\leftarrow$};
      }{}

      \ifthenelse{\isin{#1}{left} \AND \isin{#2}{up}}{
            \coordinate (anchor) at ($({#3})!{#5}!({#4})$);
            \node[darkmsgdoublecircle, xshift=-5.5mm] at (anchor) {#6};
            \node[xshift=-1.5mm] at (anchor) {$\uparrow$};
      }{}
      \ifthenelse{\isin{#1}{right} \AND \isin{#2}{up}}{
            \coordinate (anchor) at ($({#3})!{#5}!({#4})$);
            \node[darkmsgdoublecircle, xshift=5.5mm] at (anchor) {#6};
            \node[xshift=1.5mm] at (anchor) {$\uparrow$};
      }{}
}

\newcommand{\redbackmsg}[6]{
      \ifthenelse{\isin{#1}{left} \AND \isin{#2}{down}}{
            \coordinate (anchor) at ($({#3})!{#5}!({#4})$);
            \node[reddarkmsgcircle, xshift=-5mm] at (anchor) {#6};
            \node[xshift=-1.5mm] at (anchor) {$\downarrow$};
      }{}
      \ifthenelse{\isin{#1}{right} \AND \isin{#2}{down}}{
            \coordinate (anchor) at ($({#3})!{#5}!({#4})$);
            \node[reddarkmsgcircle, xshift=5mm] at (anchor) {#6};
            \node[xshift=1.5mm] at (anchor) {$\downarrow$};
      }{}

      \ifthenelse{\isin{#1}{down} \AND \isin{#2}{right}}{
            \coordinate (anchor) at ($({#3})!{#5}!({#4})$);
            \node[reddarkmsgcircle, yshift=-5.0mm] at (anchor) {#6};
            \node[yshift=-2.0mm] at (anchor) {$\rightarrow$};
      }{}
      \ifthenelse{\isin{#1}{up} \AND \isin{#2}{right}}{
            \coordinate (anchor) at ($({#3})!{#5}!({#4})$);
            \node[reddarkmsgcircle, yshift=5.0mm] at (anchor) {#6};
            \node[yshift=2.0mm] at (anchor) {$\rightarrow$};
      }{}

      \ifthenelse{\isin{#1}{down} \AND \isin{#2}{left}}{
            \coordinate (anchor) at ($({#3})!{#5}!({#4})$);
            \node[reddarkmsgcircle, yshift=-5.0mm] at (anchor) {#6};
            \node[yshift=-2.0mm] at (anchor) {$\leftarrow$};
      }{}
      \ifthenelse{\isin{#1}{up} \AND \isin{#2}{left}}{
            \coordinate (anchor) at ($({#3})!{#5}!({#4})$);
            \node[reddarkmsgcircle, yshift=5.0mm] at (anchor) {#6};
            \node[yshift=2.0mm] at (anchor) {$\leftarrow$};
      }{}

      \ifthenelse{\isin{#1}{left} \AND \isin{#2}{up}}{
            \coordinate (anchor) at ($({#3})!{#5}!({#4})$);
            \node[reddarkmsgcircle, xshift=-5.0mm] at (anchor) {#6};
            \node[xshift=-1.5mm] at (anchor) {$\uparrow$};
      }{}
      \ifthenelse{\isin{#1}{right} \AND \isin{#2}{up}}{
            \coordinate (anchor) at ($({#3})!{#5}!({#4})$);
            \node[reddarkmsgcircle, xshift=5.0mm] at (anchor) {#6};
            \node[xshift=1.5mm] at (anchor) {$\uparrow$};
      }{}
}

\newcommand{\redmsg}[6]{
      \ifthenelse{\isin{#1}{left} \AND \isin{#2}{down}}{
            \coordinate (anchor) at ($({#3})!{#5}!({#4})$);
            \node[redmsgcircle, xshift=-5mm] at (anchor) {#6};
            \node[xshift=-1.5mm] at (anchor) {$\downarrow$};
      }{}
      \ifthenelse{\isin{#1}{right} \AND \isin{#2}{down}}{
            \coordinate (anchor) at ($({#3})!{#5}!({#4})$);
            \node[redmsgcircle, xshift=5mm] at (anchor) {#6};
            \node[xshift=1.5mm] at (anchor) {$\downarrow$};
      }{}

      \ifthenelse{\isin{#1}{down} \AND \isin{#2}{right}}{
            \coordinate (anchor) at ($({#3})!{#5}!({#4})$);
            \node[redmsgcircle, yshift=-5.0mm] at (anchor) {#6};
            \node[yshift=-2.0mm] at (anchor) {$\rightarrow$};
      }{}
      \ifthenelse{\isin{#1}{up} \AND \isin{#2}{right}}{
            \coordinate (anchor) at ($({#3})!{#5}!({#4})$);
            \node[redmsgcircle, yshift=5.0mm] at (anchor) {#6};
            \node[yshift=2.0mm] at (anchor) {$\rightarrow$};
      }{}

      \ifthenelse{\isin{#1}{down} \AND \isin{#2}{left}}{
            \coordinate (anchor) at ($({#3})!{#5}!({#4})$);
            \node[redmsgcircle, yshift=-5.0mm] at (anchor) {#6};
            \node[yshift=-2.0mm] at (anchor) {$\leftarrow$};
      }{}
      \ifthenelse{\isin{#1}{up} \AND \isin{#2}{left}}{
            \coordinate (anchor) at ($({#3})!{#5}!({#4})$);
            \node[redmsgcircle, yshift=5.0mm] at (anchor) {#6};
            \node[yshift=2.0mm] at (anchor) {$\leftarrow$};
      }{}

      \ifthenelse{\isin{#1}{left} \AND \isin{#2}{up}}{
            \coordinate (anchor) at ($({#3})!{#5}!({#4})$);
            \node[redmsgcircle, xshift=-5.0mm] at (anchor) {#6};
            \node[xshift=-1.5mm] at (anchor) {$\uparrow$};
      }{}
      \ifthenelse{\isin{#1}{right} \AND \isin{#2}{up}}{
            \coordinate (anchor) at ($({#3})!{#5}!({#4})$);
            \node[redmsgcircle, xshift=5.0mm] at (anchor) {#6};
            \node[xshift=1.5mm] at (anchor) {$\uparrow$};
      }{}
}

\newcommand{\msgcircle}[6]{
      \ifthenelse{\isin{#1}{left} \AND \isin{#2}{down}}{
            \coordinate (anchor) at ($({#3})!{#5}!({#4})$);
            \node[msgcircle,xshift=-5.0mm] at (anchor) {#6};
            \node[xshift=-1.5mm] at (anchor) {$\downarrow$};
      }{}
      \ifthenelse{\isin{#1}{right} \AND \isin{#2}{down}}{
            \coordinate (anchor) at ($({#3})!{#5}!({#4})$);
            \node[msgcircle,xshift=5.0mm] at (anchor) {#6};
            \node[xshift=1.5mm] at (anchor) {$\downarrow$};
      }{}

      \ifthenelse{\isin{#1}{down} \AND \isin{#2}{right}}{
            \coordinate (anchor) at ($({#3})!{#5}!({#4})$);
            \node[msgcircle, yshift=-5.0mm] at (anchor) {#6};
            \node[yshift=-2.0mm] at (anchor) {$\rightarrow$};
      }{}
      \ifthenelse{\isin{#1}{up} \AND \isin{#2}{right}}{
            \coordinate (anchor) at ($({#3})!{#5}!({#4})$);
            \node[msgcircle, yshift=5.0mm] at (anchor) {#6};
            \node[yshift=2.0mm] at (anchor) {$\rightarrow$};
      }{}

      \ifthenelse{\isin{#1}{down} \AND \isin{#2}{left}}{
            \coordinate (anchor) at ($({#3})!{#5}!({#4})$);
            \node[msgcircle, yshift=-5.0mm] at (anchor) {#6};
            \node[yshift=-2.0mm] at (anchor) {$\leftarrow$};
      }{}
      \ifthenelse{\isin{#1}{up} \AND \isin{#2}{left}}{
            \coordinate (anchor) at ($({#3})!{#5}!({#4})$);
            \node[msgcircle, yshift=5.0mm] at (anchor) {#6};
            \node[yshift=2.0mm] at (anchor) {$\leftarrow$};
      }{}

      \ifthenelse{\isin{#1}{left} \AND \isin{#2}{up}}{
            \coordinate (anchor) at ($({#3})!{#5}!({#4})$);
            \node[msgcircle, xshift=-5.0mm] at (anchor) {#6};
            \node[xshift=-1.5mm] at (anchor) {$\uparrow$};
      }{}
      \ifthenelse{\isin{#1}{right} \AND \isin{#2}{up}}{
            \coordinate (anchor) at ($({#3})!{#5}!({#4})$);
            \node[msgcircle, xshift=5.0mm] at (anchor) {#6};
            \node[xshift=1.5mm] at (anchor) {$\uparrow$};
      }{}
}

\tikzset{mainstyle/.style={fill=white, draw=black, shape=rectangle, align=center}}
\tikzset{dstyle/.style={mainstyle, minimum size=4mm, inner sep=0pt, text width=4mm}}
\tikzset{sstyle/.style={mainstyle, minimum size=5mm, inner sep=0pt, text width=5mm}}
\tikzset{ostyle/.style={fill=darkgrey, draw=black, shape=rectangle, minimum size=0.2cm, inner sep=0pt, text width=2mm}}

\tikzstyle{observation}=[ostyle]
\tikzstyle{deterministic}=[dstyle]
\tikzstyle{stochastic}=[sstyle]

\tikzstyle{filter}=[mainstyle, minimum width=1cm, minimum height=0.5cm]
\tikzstyle{selector}=[fill=white, draw=black, shape=trapezium, rotate=180, minimum width=1cm, minimum height=0.5cm]

\makeatletter
\DeclareRobustCommand{\cev}[1]{%
  {\mathpalette\do@cev{#1}}%
}
\newcommand{\do@cev}[2]{%
  \vbox{\offinterlineskip
    \sbox\z@{$\m@th#1 x$}%
    \ialign{##\cr
      \hidewidth\reflectbox{$\m@th#1\vec{}\mkern4mu$}\hidewidth\cr
      \noalign{\kern-\ht\z@}
      $\m@th#1#2$\cr
    }%
  }%
}

\tikzstyle{dash} = [dashed, -latex,>=latex]
\tikzstyle{line} = [draw, -latex,>=latex]
\tikzstyle{smallbox} = [draw, minimum size=5.0mm]
\tikzstyle{box} = [draw, minimum size=7.0mm]
\tikzstyle{bigbox} = [draw, minimum size=10.0mm]
\tikzstyle{hugebox} = [draw, minimum size=13.0mm]
\tikzstyle{rectangle} = [draw, minimum width=10.0mm, minimum height=20.0mm]
\tikzstyle{switch} = [trapezium, trapezium angle=120, draw, rotate=90,  inner ysep=5pt, outer sep=5pt,
minimum height=7mm, minimum width=7mm]
\tikzstyle{roundbox} = [draw, circle, inner sep=0pt, minimum size=3mm]
\usetikzlibrary{arrows.meta}
\usetikzlibrary{backgrounds}
\usepgfplotslibrary{patchplots}
\usepgfplotslibrary{fillbetween}
\pgfplotsset{%
    layers/standard/.define layer set={%
        background,axis background,axis grid,axis ticks,axis lines,axis tick labels,pre main,main,axis descriptions,axis foreground%
    }{
        grid style={/pgfplots/on layer=axis grid},%
        tick style={/pgfplots/on layer=axis ticks},%
        axis line style={/pgfplots/on layer=axis lines},%
        label style={/pgfplots/on layer=axis descriptions},%
        legend style={/pgfplots/on layer=axis descriptions},%
        title style={/pgfplots/on layer=axis descriptions},%
        colorbar style={/pgfplots/on layer=axis descriptions},%
        ticklabel style={/pgfplots/on layer=axis tick labels},%
        axis background@ style={/pgfplots/on layer=axis background},%
        3d box foreground style={/pgfplots/on layer=axis foreground},%
    },
}

\definecolor{darkgrey}{gray}{0.3}

%% file: sections/new_introduction.tex
In deterministic computation, building hierarchical systems is straightforward. Functions call functions, expressions nest within expressions, and complexity grows without any fundamental barrier. Languages and compilers execute such computational graphs step by step.


Probabilistic models need the same kind of hierarchy to be expressive enough. Consider ensemble forecasting: a simple model assigns a fixed precision to each expert, but what if an expert's reliability varies with the input? Then precision itself becomes a learned function of the input features, and that function should carry its own uncertainty. What if different experts are relevant in different regions of the input space? Then a routing layer must select among experts, adding another level of probabilistic structure on top. Each layer adds modeling power, and each layer should propagate uncertainty.
Yet, stacking probabilistic building blocks into deeper architectures typically leads to intractable operations. 
The field has responded by developing increasingly powerful black-box inference methods: Markov chain Monte Carlo \citep{neal_mcmc_2011}, black-box variational inference \citep{ranganath_black_2014}, amortized variational autoencoders \citep{kingma_autoencoding_2013}, normalizing flows \citep{rezende_variational_2015}, and diffusion models \citep{ho_denoising_2020}. These approaches handle non-conjugacy by learning or sampling the inference procedure but sacrifice closed-form tractability in the process.

We take a different path. Rather than improving black-box inference to handle with compositional models, we ask: \emph{can probabilistic models compose hierarchically while retaining closed-form inference?} Structural constraints on the approximate posterior---mean-field factorization and exponential-family form constraints---are known to reduce variational inference to closed-form updates in favorable settings. Message passing algorithms that minimize the Bethe free energy \citep{yedidia_constructing_2005, senoz_variational_2021} provide the execution mechanism. We show how to choose and compose fundamental probabilistic building blocks such that closed-form inference is preserved at every level of nesting.

The paper is organized around an analogy with programming languages: an \emph{alphabet} of building blocks (\cref{sec:alphabet}), a \emph{grammar} for composing them into models of increasing depth (\cref{sec:grammar}), and a \emph{runtime} in which form constraints and the Bethe free energy yield closed-form message passing derived from the model, not prescribed by the modeler (\cref{sec:runtime}).

We make the following contributions:
\begin{enumerate}
    \item We identify a five-letter alphabet of building blocks. The set is sufficient for universal approximation and closed under Q-conjugate variational message passing, so any composed model admits closed-form updates.
    \item We demonstrate composition at increasing depth: a static ensemble, input-dependent gating, and split-branch routing with piecewise-linear boundaries. Stacking routing layers encodes arbitrary decision trees, establishing universal function approximation.
    \item We apply the framework to ensemble forecasting, yielding a Bayesian mixture of experts in which gating functions are inferred rather than learned, providing principled uncertainty over expert selection.
\end{enumerate}

%% file: sections/alphabet.tex

We seek a set of probabilistic factors from which arbitrarily expressive models can be composed while retaining closed-form inference. We identify five such factors (\cref{fig:alphabet}). Sub-labels (a)--(e) match between figure and equation; edge styles encode variable roles: solid for weights and observations ($\vw, \vphi, y, \mu$), dashed for precisions ($\tau, \gamma, \beta$), and dash-dotted for the latent $z$.

This set is \emph{sufficient} for universal approximation (\cref{cor:universal}) and closed under Q-conjugate variational message passing (\cref{thm:messages,thm:marginals}); we do not claim minimum cardinality.

\input{figures/softdot}

\begin{subequations}\label{eq:alphabet}
\renewcommand{\theequation}{\theparentequation\alph{equation}}
\begin{align}
    \label{eq:softdot}
    f_*(z \mid \vw, \vphi, \tau) &= \gauss(z \mid \vw^\top \vphi, \tau^{-1}), \\
    \label{eq:exp-link}
    f_{\exp}(\gamma \mid z) &= \delta(\gamma - \exp(z)), \\
    \label{eq:gamma-factor}
    f_{\gammad}(\gamma \mid \alpha, \beta) &= \gammad(\gamma \mid \alpha, \beta), \\
    \label{eq:normal-factor}
    f_{\mathcal{N}}(y \mid \mu, \tau) &= \gauss(y \mid \mu, \tau^{-1}), \\
    \label{eq:equality}
    f_{=}(x_1, x_2, x_3) &= \delta(x_1 - x_2)\,\delta(x_1 - x_3).
\end{align}
\end{subequations}
When $y$ and $\mu$ are vectors, the normal factor \eqref{eq:normal-factor} denotes $\gauss(\vy \mid \vmu, \tau^{-1}\mI)$, i.e.\ a multivariate Gaussian whose covariance is a scalar precision $\tau$ shared across all dimensions.\wouterk{This is a sufficient alphabet for continuous outcomes not for discrete. Unless you intend to constrain a Gaussian to be discrete? If not, this should be mentioned somewhere.}
\ml{I don't think you are right. However, the graphs would need to be very large that it's probably will be not practical to compute. Theoretically you are wrong, I will explain in the discussion, thanks}

In this paper, we compose these factors into Forney-style factor graphs (FFG) \citep{forney_codes_2001, loeliger_introduction_2004}. In an FFG, square nodes represent factors and edges represent variables; each expression above becomes one node. For instance, the square labeled~$*$ in \cref{fig:alphabet} a is the factor $f_*$ from \eqref{eq:softdot}, and its four edges are the variables $\vw$, $\vphi$, $\tau$, and~$z$. Two factors that share a variable are connected by the same edge; composing models amounts to connecting nodes via shared edges.

%% file: figures/softdot.tex
\begin{figure*}[t]
\centering
\begin{tikzpicture}[every node/.style={font=\small}]
    \node[box] (sd) {$*$};
    \fill (sd.west) ++(0, 0.12) arc (90:270:0.12) -- cycle;
    \draw[fill=white] (sd.north) ++(0.12, 0) arc (0:180:0.12) -- cycle;
    \node[roundbox] at (0, 1.1) (sd_phi) {};
    \node[roundbox] at (-1.1, 0) (sd_w) {};
    \node[roundbox] at (0, -1.1) (sd_tau) {};
    \node[roundbox] at (1.1, 0) (sd_z) {};
    \draw[-] (sd) -- node[right, font=\scriptsize] {$\vphi$} (sd_phi);
    \draw[-] (sd) -- node[above, font=\scriptsize] {$\vw$} (sd_w);
    \draw[dashed] (sd) -- node[right, font=\scriptsize] {$\tau$} (sd_tau);
    \draw[dash dot] (sd) -- node[above, font=\scriptsize] {$z$} (sd_z);
    \node at (0, -1.7) {\scriptsize (a) softdot};

    \node[box] at (3.5, 0) (el) {$\exp$};
    \node[roundbox] at (3.5, 1.1) (el_z) {};
    \node[roundbox] at (3.5, -1.1) (el_gamma) {};
    \draw[dash dot] (el) -- node[right, font=\scriptsize] {$z$} (el_z);
    \draw[dashed] (el) -- node[right, font=\scriptsize] {$\gamma$} (el_gamma);
    \node at (3.5, -1.7) {\scriptsize (b) exp link};

    \node[box] at (7, 0) (gf) {$\gammad$};
    \draw[fill=white] (gf.north) ++(0.12, 0) arc (0:180:0.12) -- cycle;
    \node[clamped] at (8.1, 0) (gf_alpha) {};
    \node[roundbox] at (7, 1.1) (gf_beta) {};
    \node[roundbox] at (7, -1.1) (gf_gamma) {};
    \draw[-] (gf) -- node[above, font=\scriptsize] {$\alpha$} (gf_alpha);
    \draw[dashed] (gf) -- node[right, font=\scriptsize] {$\beta$} (gf_beta);
    \draw[dashed] (gf) -- node[right, font=\scriptsize] {$\gamma$} (gf_gamma);
    \node at (7, -1.7) {\scriptsize (c) gamma};

    \node[box] at (10.5, 0) (nf) {$\mathcal{N}$};
    \draw[fill=white] (nf.west) ++(0, 0.12) arc (90:270:0.12) -- cycle;
    \node[roundbox] at (10.5, 1.1) (nf_y) {};
    \node[roundbox] at (9.4, 0) (nf_mu) {};
    \node[roundbox] at (10.5, -1.1) (nf_tau) {};
    \draw[-] (nf) -- node[right, font=\scriptsize] {$y$} (nf_y);
    \draw[-] (nf) -- node[above, font=\scriptsize] {$\mu$} (nf_mu);
    \draw[dashed] (nf) -- node[right, font=\scriptsize] {$\tau$} (nf_tau);
    \node at (10.5, -1.7) {\scriptsize (d) normal};

    \node[smallbox] at (13.5, 0) (eq) {$=$};
    \node[roundbox] at (12.4, 0) (eq_in) {};
    \node[roundbox] at (14.6, 0.7) (eq_out1) {};
    \node[roundbox] at (14.6, -0.7) (eq_out2) {};
    \draw[-] (eq) -- (eq_in);
    \draw[-] (eq) -- (eq_out1);
    \draw[-] (eq) -- (eq_out2);
    \node at (13.5, -1.7) {\scriptsize (e) equality};
\end{tikzpicture}
\caption{The building blocks as factor graph nodes. Square nodes are factors; round nodes represent neighboring nodes to which these factors connect (e.g., priors, likelihoods, or other factors sharing the same variable). Filled black squares denote \emph{clamped} variables: fixed numerical values that cannot be connected to other factors (e.g., the shape $\alpha$ of the gamma factor). This clamped notation is used throughout for any factor edge that carries a fixed constant rather than a latent. Orientation markers on the softdot factor (filled semi-circle on the $\vw$ side, open on the $\vphi$ side) allow edge roles to be identified even when the node is rotated in composed graphs. Edge styles distinguish variable roles: solid for weights, features, and observations ($\vw$, $\vphi$, $y$, $\mu$), dashed for precisions ($\tau, \gamma, \beta$), and dash-dotted for the latent $z$. The equality node enforces that all connected edges carry the same variable. Composing models amounts to connecting factors via shared variable edges of the same line type.}
\label{fig:alphabet}
\end{figure*}
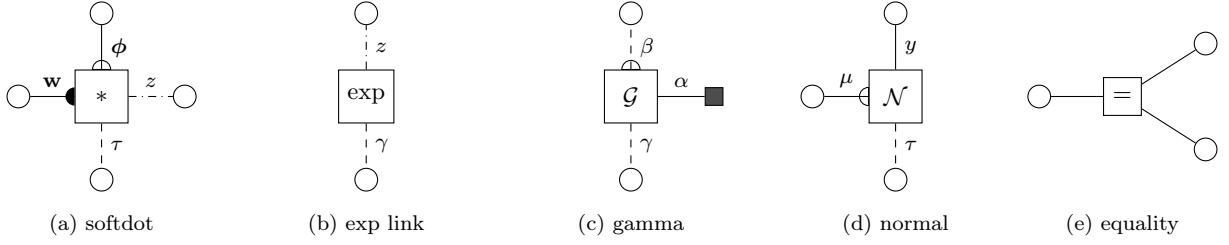

%% file: sections/grammar.tex

This section shows how to write models in the factor-graph language.
We start with one simple translation exercise: write a familiar probabilistic
model as a graph (Depth 0). Then we extend it by
composition: first by introducing a reusable word (Depth 1), followed by composing
larger words for richer sentences (Depth 2). Runtime and update rules are deferred to \cref{sec:runtime}. Readers who prefer an executable description may consult \cref{app:graphppl}, which provides the GraphPPL.jl \citep{nuijten_graphppljl_2024} code corresponding to each factor graph in this section.

\subsection{From letters to a model}\label{sec:depth0}

\input{figures/depth0_graph}

Given $n$ forecasters with predictions $\hat{y}_{i,j}$ for observation $y_j$,
the Depth-0 smoothing problem is to infer the forecaster precisions from paired
observations and predictions:
\begin{equation}\label{eq:static-ensemble-new}
    p(\vgamma \mid \vy, \bold{\hat{y}})
    \propto
    \prod_{i=1}^{n} f_{\gammad}(\gamma_i \mid \alpha_i, \beta_i)
    \prod_{\substack{i=1\\ j=1}}^{n,m}
    f_{\mathcal{N}}(\hat{y}_{i,j} \mid y_j, \gamma_i).
\end{equation}
The conditional notation in \eqref{eq:static-ensemble-new} names the inference
task, not a directed sampling order. Since the preceding section introduced
local factor functions rather than a complete joint distribution, for now
\eqref{eq:static-ensemble-new} should be read simply as ``infer $\vgamma$ given
observed targets and forecaster predictions.'' The runtime later explains how
the same product of local factors induces different message-passing problems
depending on which variables are observed and which are left latent.
The factor graph is a direct letter-by-letter transcription of this product:
one gamma prior letter $f_{\gammad}$ per expert, one normal letter
$f_{\mathcal{N}}$ per expert-observation pair, and equality letters that share
$\gamma_i$ across observations and $y_j$ across experts. \Cref{fig:depth0} shows the fragment for expert $i{=}1$ and observation $j{=}1$; dotted edges
indicate repetition over the remaining indices. The executable GraphPPL.jl code for this model is given in \cref{app:graphppl-depth0}.
\wouterk{I'm missing an explanation of \emph{why} the Depth-0 graph looks like this. I think you should mention that it's a (composite) likelihood with precision sharing. And the letter analogy becomes a bit confusing when the experts are introduced; are the letters still just the 5 factor nodes in Eq 2? Is a likelihood with a single expert (Fig 2) a letter or the likelihood with $n$ experts for a single observation? Or even the likelihood with $n$ experts and $m$ observations (Eq. 3)?}

\subsection{Defining a new word}\label{sec:depth1}

To make the precision variables $\gamma_i$ input-dependent, we define one reusable word. This is composed of the softdot \eqref{eq:softdot} node with an exponential link \eqref{eq:exp-link}, which gives the \emph{precision word} $\pi$ (\cref{fig:precision-word}), mapping $(\vw, \vphi, \tau)$ to $\gamma$.
\wouterk{So a word is a composite node that "closes the box" around the letter nodes? A word is not a connected collection of letters, such as the factor graph in Fig 2)? This may be useful to indicate explicitly.}
\ml{Yeah, I think you are right, I need to clarify it further. The word is a non-terminated graph constructed from the letters.}

\input{figures/precision_word_depth1}

Writing a Depth-1 sentence means replacing each Depth-0 precision letter
$f_{\gammad}(\gamma_i \mid \alpha_i,\beta_i)$ with $\pi$, and adding shared
parameter nodes $(\vw_i,\tau_i)$. The right-hand likelihood part is unchanged;
only the precision-generation part is replaced. \Cref{fig:depth1} shows this
pattern. We call this sentence family \emph{Precision-Gated Experts} (PGE).
\wouterk{Is a "sentence" a composite node containing at least one word? Or factor graph containing words, i.e., a connected collection of words and letters? If it's the second option (as Fig. 4 seems to imply, then it's inconsistent with the use of "word".}
\ml{Sentence probably is a wrong word. It's just a terminated graph. Maybe the better word will be a program.}



\subsection{Writing richer sentences}\label{sec:depth2}

\input{figures/depth2_graph}

The Depth-1 precision word $\pi$ maps $\vw^\top \vphi(\vx)$ through an exponential link, so $\gamma$ grows monotonically with the linear score: wherever $\vw^\top \vphi$ is large, $\gamma$ is large too. A single $\pi$ can therefore only cut the input space with one hyperplane, producing two regions. To express patterns like an exclusive-or (XOR) relation, which requires at least four regions, we need to compose a larger word from the same alphabet letters (\cref{fig:depth2}).

The key idea is to create two branches that compete. A router softdot produces a score $h$ from routing weights $\vv_i$ and features $\vphi(\vx_j)$. $h$ is then taken by two switch softdots with opposing signs, followed by exponential links: when $h > 0$ the left switch produces a large activation $\kappa^L$ and the right switch a small $\kappa^R$, and vice versa. Each branch has its own precision word $\pi$ (with separate weights $\vw_i$ and $\vu_i$), whose output feeds into a normal factor gated by the switch activation $\kappa^b$. A large $\kappa^b$ makes that normal factor tight, forcing the branch to pass its sub-expert score through; a small $\kappa^b$ leaves the normal diffuse and the branch negligible. After a final exponential link and likelihood, both branches share the observation $y_j$ through an equality node.
The result is a piecewise-linear precision function: the router selects which sub-expert controls $\gamma$ in each region of input space. Two experts with split-branch routing partition the input into four regions, which is sufficient to capture XOR. Stacking split-branch layers yields deeper sentences encoding arbitrary binary decision trees; the formal expressiveness argument is deferred to \cref{sec:expressiveness}.

%% file: figures/depth0_graph.tex
\begin{figure}[t]
\centering
\begin{tikzpicture}[every node/.style={font=\small}]

    \node[box] at (0, 0) (g) {$\gammad$};
    \draw[fill=white] (g.north) ++(0.12, 0) arc (0:180:0.12) -- cycle;
    \node[clamped] at (0, 1.2) (be) {};
    \node[above=0.3mm of be, font=\scriptsize] {$\beta_1$};
    \node[clamped] at (-1.2, 0) (al) {};
    \node[left=1mm of al, font=\scriptsize] {$\alpha_1$};
    \draw[dashed] (g) -- (be);
    \draw[-] (g) -- (al);

    \node[smallbox] at (1.5, 0) (eq_g) {$=$};
    \draw[dashed] (g) -- (eq_g) node[midway, below, font=\scriptsize] {$\gamma_1$};
    \draw[dashed] (eq_g) -- (1.5, 0.9);
    \node at (1.5, 1.2) {$\vdots$};

    \node[box] at (3, 0) (n) {$\mathcal{N}$};
    \draw[fill=white] (n.north) ++(0.12, 0) arc (0:180:0.12) -- cycle;
    \draw[dashed] (eq_g) -- (n);

    \node[clamped] at (3, 1.2) (yh) {};
    \node[above=0.3mm of yh, font=\scriptsize] {$\hat{y}_{1,1}$};
    \draw[-] (yh) -- (n);

    \node[smallbox] at (4.5, 0) (eqy) {$=$};
    \draw[-] (n) -- (eqy) node[midway, below, font=\scriptsize] {$y_1$};
    \node[clamped] at (4.5, 1.2) (obs) {};
    \node[above=0.3mm of obs, font=\scriptsize] {$y_1$};
    \draw[-] (obs) -- (eqy);
    \draw[-] (eqy) -- (5.4, 0);
    \node at (5.7, 0) {$\cdots$};

\end{tikzpicture}
\caption{Depth~0 factor graph (static ensemble), shown starting from expert $i{=}1$ and observation $j{=}1$. The vertical dots $\scalebox{0.5}{\vdots}$ on the $\gamma_1$ equality node extend to the other observations $j = 2, \ldots, m$; the horizontal dots $\cdots$ on the $y_1$ equality node extends to the other $n - 1$ experts.}
\label{fig:depth0}
\end{figure}
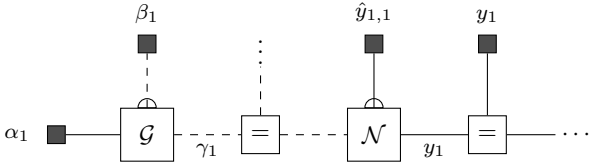

%% file: figures/precision_word_depth1.tex
\begin{figure*}[t]
\centering
\begin{subfigure}[b]{0.34\textwidth}
\centering
\begin{tikzpicture}[every node/.style={font=\small}, scale=0.85, transform shape]
    \draw[dashed, rounded corners=4pt] (-1.4, -1.2) rectangle (3.4, 1.2);

    \node[box] at (0, 0) (sd) {$*$};
    \fill (sd.west) ++(0, 0.12) arc (90:270:0.12) -- cycle;
    \draw[fill=white] (sd.north) ++(0.12, 0) arc (0:180:0.12) -- cycle;

    \node[box] at (2, 0) (el) {$\exp$};

    \draw[dash dot] (sd) -- node[above, font=\scriptsize] {$z$} (el);

    \node[roundbox] at (-2.2, 0) (iw) {};
    \node[roundbox] at (0, 2) (iphi) {};
    \node[roundbox] at (0, -2) (itau) {};
    \node[roundbox] at (4.2, 0) (igamma) {};

    \draw[-] (sd) -- node[above, font=\scriptsize] {$\vw$} (iw);
    \draw[-] (sd) -- node[right, font=\scriptsize] {$\vphi$} (iphi);
    \draw[dashed] (sd) -- node[right, font=\scriptsize] {$\tau$} (itau);
    \draw[dashed] (el) -- node[above, font=\scriptsize] {$\gamma$} (igamma);
\end{tikzpicture}
\caption{Expanded}
\end{subfigure}%
\hfill
\begin{subfigure}[b]{0.20\textwidth}
\centering
\begin{tikzpicture}[every node/.style={font=\small}, scale=0.85, transform shape]
    \node[draw, double, double distance=1.5pt, minimum size=9mm, rounded corners=1pt]
        at (0, 0) (pi) {$\pi$};
    \fill (pi.south) ++(-0.12, 0) arc (180:360:0.12) -- cycle;

    \node[roundbox] at (-1.7, 0) (cw) {};
    \node[roundbox] at (0, 2) (cphi) {};
    \node[roundbox] at (0, -2) (ctau) {};
    \node[roundbox] at (1.7, 0) (cgamma) {};

    \draw[-] (pi) -- node[above, font=\scriptsize] {$\vw$} (cw);
    \draw[-] (pi) -- node[right, font=\scriptsize] {$\vphi$} (cphi);
    \draw[dashed] (pi) -- node[right, font=\scriptsize] {$\tau$} (ctau);
    \draw[dashed] (pi) -- node[above, font=\scriptsize] {$\gamma$} (cgamma);
\end{tikzpicture}
\caption{Compact}
\end{subfigure}%
\hfill
\begin{subfigure}[b]{0.43\textwidth}
\centering
\begin{tikzpicture}[every node/.style={font=\small}, scale=0.85, transform shape,
    pibox/.style={draw, double, double distance=1.5pt, minimum size=7mm, rounded corners=1pt}]

    \node[box] at (0, 1.5) (pw) {$\mathcal{N}$};
    \node[smallbox] at (1.5, 1.5) (eq_w) {$=$};
    \draw[-] (pw) -- (eq_w) node[midway, above, font=\scriptsize] {$\vw_1$};
    \draw[-] (eq_w) -- (2.4, 1.5);
    \node at (2.7, 1.5) {$\cdots$};

    \node[box] at (0, -1.5) (ptau) {$\gammad$};
    \node[smallbox] at (1.5, -1.5) (eq_t) {$=$};
    \draw[dashed] (ptau) -- (eq_t) node[midway, below, font=\scriptsize] {$\tau_1$};
    \draw[dashed] (eq_t) -- (2.4, -1.5);
    \node at (2.7, -1.5) {$\cdots$};

    \node[pibox] at (1.5, 0) (pi) {$\pi$};
    \fill (pi.south) ++(-0.12, 0) arc (180:360:0.12) -- cycle;

    \draw[-] (eq_w) -- (pi);
    \draw[dashed] (eq_t) -- (pi);

    \node[left=4.5mm of pi, clamped] (phi) {};
    \node[left=1.3mm of phi, font=\scriptsize] {$\vphi(\vx_1)$};
    \draw[-] (phi) -- (pi);

    \node[box] at (3.5, 0.0) (n) {$\mathcal{N}$};
    \draw[fill=white] (n.north) ++(0.12, 0) arc (0:180:0.12) -- cycle;
    \draw[dashed] (pi) -- (n) node[midway, below, font=\scriptsize] {$\gamma_{1,1}$};

    \node[clamped] at (3.5, 0.9) (yh) {};
    \node[above=0.3mm of yh, font=\scriptsize] {$\hat{y}_{1,1}$};
    \draw[-] (yh) -- (n);

    \node[smallbox] at (5, 0) (eqy) {$=$};
    \draw[-] (n) -- (eqy) node[midway, below, font=\scriptsize] {$y_1$};
    \node[clamped] at (5, 0.9) (obs) {};
    \node[above=0.3mm of obs, font=\scriptsize] {$y_1$};
    \draw[-] (obs) -- (eqy);
    \draw[-] (eqy) -- (5.9, 0);
    \node at (6.2, 0) {$\cdots$};

\end{tikzpicture}
\caption{Depth~1 (PGE)}
\label{fig:depth1}
\end{subfigure}
\caption{The \emph{precision word} $\pi$ and Depth-1 model. (a)~Internal structure: a softdot and exponential link connected by latent $z$, computing input-dependent precision $\gamma$ from $\vw$, $\vphi$, and $\tau$. (b)~Compact notation; double border indicates a composite word, filled semi-circle marks the $\tau$ (input precision) side. (c)~Depth~1 factor graph (Precision-Gated Experts) for expert $i{=}1$, observation $j{=}1$: compared to Depth~0 (\cref{fig:depth0}), the gamma prior is replaced by $\pi$, producing input-dependent $\gamma_{1,1}$. Weights $\vw_1$ and $\tau_1$ are shared across observations ($\cdots$).}
\label{fig:precision-word}
\end{figure*}
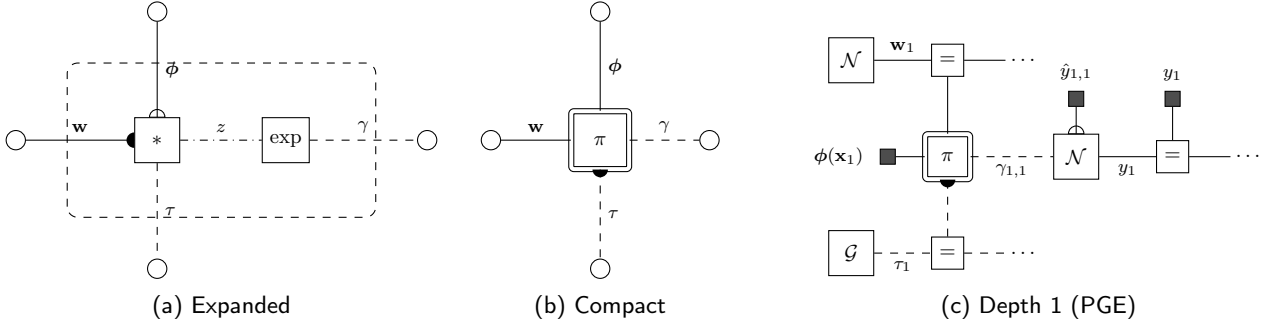

%% file: figures/depth2_graph.tex
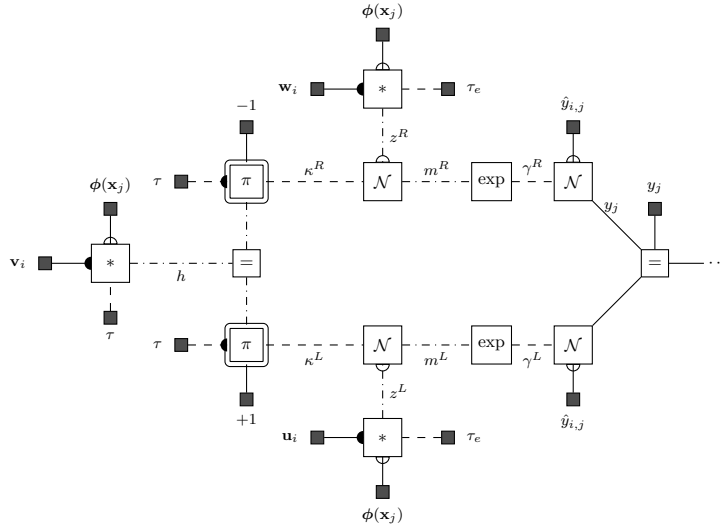
\begin{figure*}[t]
\centering
\begin{tikzpicture}[every node/.style={font=\small}, scale=0.72, transform shape,
    pibox/.style={draw, double, double distance=1.5pt, minimum size=7mm, rounded corners=1pt}]

    \node[box] at (0, 0) (router) {$*$};
    \fill (router.west) ++(0, 0.12) arc (90:270:0.12) -- cycle;
    \draw[fill=white] (router.north) ++(0.12, 0) arc (0:180:0.12) -- cycle;

    \node[clamped] at (0, 1) (phi_r) {};
    \node[above=0.3mm of phi_r, font=\scriptsize] {$\vphi(\vx_j)$};
    \draw[-] (phi_r) -- (router);

    \node[clamped] at (-1.2, 0) (vv) {};
    \node[left=1mm of vv, font=\scriptsize] {$\vv_i$};
    \draw[-] (vv) -- (router);

    \node[clamped] at (0, -1) (tau_r) {};
    \node[below=0.3mm of tau_r, font=\scriptsize] {$\tau$};
    \draw[dashed] (tau_r) -- (router);

    \node[smallbox] at (2.5, 0) (eq_h) {$=$};
    \draw[dash dot] (router) -- (eq_h) node[midway, below, font=\scriptsize] {$h$};

    \node[pibox] at (2.5, 1.5) (pi_sw_r) {$\pi$};
    \fill (pi_sw_r.west) ++(0, 0.12) arc (90:270:0.12) -- cycle;

    \draw[dash dot] (eq_h) -- (pi_sw_r);

    \node[clamped] at (2.5, 2.5) (neg1) {};
    \node[above=0.3mm of neg1, font=\scriptsize] {$-1$};
    \draw[-] (neg1) -- (pi_sw_r);

    \node[clamped] at (1.3, 1.5) (tau_sw_r) {};
    \node[left=1mm of tau_sw_r, font=\scriptsize] {$\tau$};
    \draw[dashed] (tau_sw_r) -- (pi_sw_r);

    \node[pibox] at (2.5, -1.5) (pi_sw_l) {$\pi$};
    \fill (pi_sw_l.west) ++(0, 0.12) arc (90:270:0.12) -- cycle;

    \draw[dash dot] (eq_h) -- (pi_sw_l);

    \node[clamped] at (2.5, -2.5) (pos1) {};
    \node[below=0.3mm of pos1, font=\scriptsize] {$+1$};
    \draw[-] (pos1) -- (pi_sw_l);

    \node[clamped] at (1.3, -1.5) (tau_sw_l) {};
    \node[left=1mm of tau_sw_l, font=\scriptsize] {$\tau$};
    \draw[dashed] (tau_sw_l) -- (pi_sw_l);

    \node[box] at (5, 1.5) (n_gate_r) {$\mathcal{N}$};
    \draw[fill=white] (n_gate_r.north) ++(0.12, 0) arc (0:180:0.12) -- cycle;

    \draw[dashed] (pi_sw_r) -- (n_gate_r) node[midway, above, font=\scriptsize] {$\kappa^R$};

    \node[box] at (5, 3.2) (sub_r) {$*$};
    \fill (sub_r.west) ++(0, 0.12) arc (90:270:0.12) -- cycle;
    \draw[fill=white] (sub_r.north) ++(0.12, 0) arc (0:180:0.12) -- cycle;

    \node[clamped] at (3.8, 3.2) (wr) {};
    \node[left=1mm of wr, font=\scriptsize] {$\vw_i$};
    \draw[-] (wr) -- (sub_r);

    \node[clamped] at (5, 4.2) (phi_sub_r) {};
    \node[above=0.3mm of phi_sub_r, font=\scriptsize] {$\vphi(\vx_j)$};
    \draw[-] (phi_sub_r) -- (sub_r);

    \node[clamped] at (6.2, 3.2) (tau_sub_r) {};
    \node[right=1mm of tau_sub_r, font=\scriptsize] {$\tau_e$};
    \draw[dashed] (tau_sub_r) -- (sub_r);

    \draw[dash dot] (sub_r) -- (n_gate_r) node[midway, right, font=\scriptsize] {$z^R$};

    \node[box] at (7, 1.5) (exp_r) {$\exp$};
    \draw[dash dot] (n_gate_r) -- (exp_r) node[midway, above, font=\scriptsize] {$m^R$};

    \node[box] at (8.5, 1.5) (lik_r) {$\mathcal{N}$};
    \draw[fill=white] (lik_r.north) ++(0.12, 0) arc (0:180:0.12) -- cycle;
    \draw[dashed] (exp_r) -- (lik_r) node[midway, above, font=\scriptsize] {$\gamma^R$};

    \node[clamped] at (8.5, 2.5) (yh_r) {};
    \node[above=0.3mm of yh_r, font=\scriptsize] {$\hat{y}_{i,j}$};
    \draw[-] (yh_r) -- (lik_r);

    \node[box] at (5, -1.5) (n_gate_l) {$\mathcal{N}$};
    \draw[fill=white] (n_gate_l.south) ++(-0.12, 0) arc (180:360:0.12) -- cycle;

    \draw[dashed] (pi_sw_l) -- (n_gate_l) node[midway, below, font=\scriptsize] {$\kappa^L$};

    \node[box] at (5, -3.2) (sub_l) {$*$};
    \fill (sub_l.west) ++(0, 0.12) arc (90:270:0.12) -- cycle;
    \draw[fill=white] (sub_l.south) ++(-0.12, 0) arc (180:360:0.12) -- cycle;

    \node[clamped] at (3.8, -3.2) (wl) {};
    \node[left=1mm of wl, font=\scriptsize] {$\vu_i$};
    \draw[-] (wl) -- (sub_l);

    \node[clamped] at (5, -4.2) (phi_sub_l) {};
    \node[below=0.3mm of phi_sub_l, font=\scriptsize] {$\vphi(\vx_j)$};
    \draw[-] (phi_sub_l) -- (sub_l);

    \node[clamped] at (6.2, -3.2) (tau_sub_l) {};
    \node[right=1mm of tau_sub_l, font=\scriptsize] {$\tau_e$};
    \draw[dashed] (tau_sub_l) -- (sub_l);

    \draw[dash dot] (sub_l) -- (n_gate_l) node[midway, right, font=\scriptsize] {$z^L$};

    \node[box] at (7, -1.5) (exp_l) {$\exp$};
    \draw[dash dot] (n_gate_l) -- (exp_l) node[midway, below, font=\scriptsize] {$m^L$};

    \node[box] at (8.5, -1.5) (lik_l) {$\mathcal{N}$};
    \draw[fill=white] (lik_l.south) ++(-0.12, 0) arc (180:360:0.12) -- cycle;
    \draw[dashed] (exp_l) -- (lik_l) node[midway, below, font=\scriptsize] {$\gamma^L$};

    \node[clamped] at (8.5, -2.5) (yh_l) {};
    \node[below=0.3mm of yh_l, font=\scriptsize] {$\hat{y}_{i,j}$};
    \draw[-] (yh_l) -- (lik_l);

    \node[smallbox] at (10, 0) (eqy) {$=$};
    \draw[-] (lik_r) -- (eqy) node[pos=0.4, above, font=\scriptsize] {$y_j$};
    \draw[-] (lik_l) -- (eqy);

    \node[clamped] at (10, 1) (obs) {};
    \node[above=0.3mm of obs, font=\scriptsize] {$y_j$};
    \draw[-] (obs) -- (eqy);

    \draw[-] (eqy) -- (10.9, 0);
    \node at (11.2, 0) {$\cdots$};

\end{tikzpicture}
\caption{Depth~2 factor graph (split-branch routing) for one expert, one observation. The router softdot produces $h$; two switch $\pi$ words with clamped weights $-1$ and $+1$ convert $h$ into opposing activations $\kappa^R, \kappa^L$. Each branch has a sub-expert softdot producing $z^b$, gated by a normal factor whose precision is $\kappa^b$: the active branch ($\kappa$ large) passes $z^b$ tightly, while the suppressed branch ($\kappa$ small) is diffuse. After exponential links and likelihoods, both branches share $y_j$ via the equality node.}
\label{fig:depth2}
\end{figure*}

%% file: sections/runtime.tex
The inference algorithm for any model composed of the alphabet of \cref{sec:alphabet} is not prescribed by the modeler but \emph{emerges} from the model specification \citep{senoz_variational_2021, yedidia_constructing_2005}. Executing inference requires two alternating operations: computing a message from a factor toward an edge, and computing the marginal on an edge from incoming messages. This section shows that both operations have closed-form solutions of known parametric form for any factor graph composed from the alphabet.

\subsection{Declarative inference}\label{sec:declarative}

The modeler specifies four things:
\begin{enumerate}
    \item A \textbf{graph}: a Forney-style factor graph assembled from the five alphabet letters \eqref{eq:alphabet}.
    \item \textbf{Terminations}: observation or unity factors attached to half-edges (edges connected to only one factor), determining which variables are observed or latent.
    \item \textbf{Factorization constraints}: a mean-field decomposition of the approximate posterior, e.g., $$q(\vw, \vtau, \vz) = \prod_i q(\vw_i)\, q(\tau_i) \prod_j q(z_{i,j}) \, .$$ \vspace{-20pt}
    \item \textbf{Form constraints}: the exponential-family form for each marginal (Gaussian for solid and dash-dot edges, Gamma for dashed edges).
\end{enumerate}
A terminated graph defines a conditional distribution after normalization; \cref{app:terminated-ffg} illustrates this on the Depth-0 model.
Given these declarations, the Bethe free energy (BFE) provides the variational objective (\cref{app:bfe}). Two message passing algorithms emerge from the BFE stationarity conditions (\cref{app:messages-from-stationarity}): \emph{variational message passing} (VMP) emerges when mean-field factorization constraints are imposed, and \emph{belief propagation} (BP) emerges in the unconstrained case. Around deterministic factors such as the exponential link \eqref{eq:exp-link}, mean-field factorization cannot be imposed without degeneracy, so BP is the only option there. Both algorithms produce messages sent along edges; the difference is what information they require as input.

\subsection{Computing messages}\label{sec:computing-messages}

A VMP message from factor~$a$ toward edge~$j$ (\cref{eq:vmp-message}) requires only the \emph{marginals} $q_i(s_i)$ on the remaining edges $i \in \mathcal{E}(a) \setminus j$. Suppose that each marginal has the parametric form indicated by its line style: Gaussian for solid and dash-dot edges, Gamma for dashed edges. Then the message computation depends only on the factor~$a$ and the \emph{line types} of its other edges---not on the rest of the graph. The line styles introduced in \cref{sec:alphabet} act as a type system: given marginals of the matching types, every VMP message is determined locally.

A BP message from factor~$a$ toward edge~$j$ (\cref{eq:bp-message}) requires the incoming \emph{messages} $\mu_{ia}(s_i)$ on the remaining edges, not marginals. The form of these messages depends on what is connected on the other side of those edges, so the line-type argument does not apply directly.

In our alphabet, two factors require BP: the equality node \eqref{eq:equality} and the exponential link \eqref{eq:exp-link}. Both are deterministic (delta functions) and cannot be mean-field factorized without degeneracy. The equality node is trivial---its BP message toward any edge is the product of incoming messages from the other edges, which preserves the message type. The exponential link is the non-trivial case: as a nonlinear deterministic factor, it cannot be mean-field factorized, so imposing $q(z)\,q(\gamma)$ with the constraint $\gamma = e^z$ would force both marginals to be degenerate.

\input{figures/runtime_messages}

However, every factor adjacent to the exp link is mean-field factorized, so the messages arriving at the exp link have known forms. On the $z$~edge (\cref{fig:runtime-messages}b), the only possible neighbors are the softdot \eqref{eq:softdot}, the normal factor \eqref{eq:normal-factor}, and the equality node \eqref{eq:equality}---and each sends a Gaussian message under the mean-field assumption.
In every case, the incoming message on~$z$ is Gaussian. The BP rule through the delta function then gives
\begin{equation}\label{eq:exp-bp-gamma}
    \mu_{\gamma \leftarrow \exp}(\gamma) = \int \delta(\gamma - e^z)\, \mu_{z \to \exp}(z)\, \mathrm{d}z,
\end{equation}
which is always a log-normal distribution on~$\gamma$. The same enumeration on the $\gamma$~edge (normal \eqref{eq:normal-factor}, softdot \eqref{eq:softdot}, gamma \eqref{eq:gamma-factor}, or equality \eqref{eq:equality} as neighbors) shows that the incoming message is always Gamma-typed. The BP rule in the $\gamma \to z$ direction then produces a log-gamma distribution on~$z$ \citep[Section~2.2]{lukashchuk_qconjugate_2024}.

The VMP line-type argument together with the BP enumeration above covers every factor--edge pair in the alphabet. We summarize this with the following definition and theorem.

\begin{defn}[Proper factor graph]\label{def:proper-ffg}
A factor graph is \emph{proper} if it is assembled from the alphabet \eqref{eq:alphabet} and every shared edge respects the line types: solid and dash-dot edges connect only to Gaussian-typed ports, and dashed edges connect only to Gamma-typed ports.
\end{defn}

\begin{thm}[Closed-form messages of known type]\label{thm:messages}
Let $\mathcal{G}$ be a proper factor graph with full mean-field factorization on all non-deterministic factors, and suppose the marginals on all edges have the parametric form indicated by their line type. Then for every factor--edge pair in~$\mathcal{G}$, the outgoing message exists in closed form and has a known parametric type:
\begin{enumerate}
    \item At VMP factors \emph{(}softdot \eqref{eq:softdot}, normal \eqref{eq:normal-factor}, gamma \eqref{eq:gamma-factor}\emph{)}: the message type matches the line type of the target edge---Gaussian on solid and dash-dot edges, Gamma on dashed edges.
    \item At the equality node \eqref{eq:equality} \emph{(}BP, trivial\emph{)}: the message is the product of incoming messages from the other edges and preserves their type.
    \item At the exp link \eqref{eq:exp-link} \emph{(}BP, non-trivial\emph{)}: the message toward~$\gamma$ is log-normal; the message toward~$z$ is log-gamma \citep[Section~2.2]{lukashchuk_qconjugate_2024}.
\end{enumerate}
\end{thm}
\begin{proof}
For VMP factors, the message depends only on the factor and the marginal types on its other edges (\cref{eq:vmp-message}). The line types fix these marginal types, so the message is determined locally, regardless of the rest of the graph. The equality node uses BP, but its message is the product of incoming messages (\cref{eq:bp-message}), which preserves the exponential family. For the exp link, mean-field factorization on all adjacent factors guaranties that the incoming message on~$z$ is always Gaussian and the incoming message on~$\gamma$ is always Gamma, as established by the enumeration above. The BP rule \eqref{eq:exp-bp-gamma} then produces a log-normal message on~$\gamma$ and a log-gamma message on~$z$. The full catalog of update rules is given in \cref{app:message-rules}.
\end{proof}

\subsection{Computing marginals}\label{sec:computing-marginals}

For all edges \emph{not adjacent to an exp-link factor}, the BFE stationarity condition (\cref{eq:edge-belief-bp}) gives the marginal as the normalized product of the two incoming messages:
\begin{equation}\label{eq:marginal-product}
    q_j^*(s_j) \propto \mu_{jb}(s_j)\, \mu_{jc}(s_j).
\end{equation}
By \cref{thm:messages}, both messages at such edges belong to the same exponential family (Gaussian or Gamma). Their product stays in-family, and the marginal is obtained by combining natural parameters. No approximation is involved.

At the two edges of each exp-link factor, the situation is different. \Cref{thm:messages} states that on the $z$~edge, one message is Gaussian (from the adjacent VMP factor) and the other is log-gamma (from the exp link); on the $\gamma$~edge, one is Gamma and the other is log-normal. In neither case does the product \eqref{eq:marginal-product} belong to a standard parametric family, so direct multiplication does not yield a tractable marginal.

To handle these edges, we impose a form constraint: $q(z) \in \mathcal{Q}_z = \{\gauss(z \mid m, v)\}$ with natural parameters $\veta = (\eta_1, \eta_2)^\top = (m/v,\, -1/(2v))^\top$. Restricting the BFE to this parametric family yields a constrained objective $F[\veta]$ that upper-bounds the unconstrained BFE, so the resulting inference is a variational bound.

We derive the stationarity condition for $q(z)$ in the concrete case where $z$ connects the softdot~\eqref{eq:softdot} on the left to the exp link~\eqref{eq:exp-link} on the right, with a normal factor~\eqref{eq:normal-factor} on the far side of~$\gamma$ (so that $e^z$ serves as the precision of the normal).

The exp link is a deterministic factor $f_{\exp}(z, \gamma) = \delta(\gamma - e^z)$. By \citep[Theorem~8]{senoz_variational_2021}, the node-local belief must respect the deterministic constraint, giving $q_{\exp}(z, \gamma) = q(z)\,\delta(\gamma - e^z)$ and reducing the node-local free energy to $-\Ent{q(z)}$ (for a general overview of the treatment of deterministic nodes in FFGs, see \citep[Section~5.2]{senoz_variational_2021}). The delta absorbs $\gamma$ and couples the far-side normal factor to~$z$: the normal factor contributes $\log f_{\mathcal{N}}(y \mid \mu, e^{-z})$ evaluated at $\gamma = e^z$. Under full mean-field on the normal factor, the BFE terms that depend on $q(z)$ reduce to
\begin{align}\label{eq:local-bfe-z}
    F_z[\veta] &= -\Ent{q(z)} 
    + \underbrace{\Exp{q_z q_w q_\tau}{
        -\log f_*
    }}_{\text{softdot}} \notag\\
    &+ \underbrace{\Exp{q_z q_y q_\mu}{
        -\log f_{\mathcal{N}}(y \mid \mu, e^{-z}) - z
    }}_{\text{exp + normal}},
\end{align}
where the Jacobian term $-z$ arises from the change of variables $\gamma = e^z$. The softdot factor is conjugate in~$z$: its contribution involves only Gaussian sufficient statistics of $q(z)$. The second term, originating from the exp link and the far-side normal factor, expands to
\begin{multline}\label{eq:exp-contribution}
    \Exp{q(z)}{-\bar{\ell}_{\exp}(z)}
    = \tfrac{1}{2}\Exp{q(z)}{z} \\
    - \tfrac{1}{2}\Exp{q(y),q(\mu)}{(y - \mu)^2}\,\Exp{q(z)}{e^z}
    + \text{const},
\end{multline}
where the coefficient of $\Exp{q(z)}{e^z}$ depends on the marginals $q(y)$ and $q(\mu)$ of the far-side normal factor. The two contributions in \eqref{eq:local-bfe-z} can be recognized as the two messages from \cref{thm:messages}: the softdot contribution is the Gaussian message $\mu_{z \leftarrow \text{softdot}}$, and the exp-link contribution \eqref{eq:exp-contribution} is the log-gamma message $\mu_{z \leftarrow \exp}$ (with its coefficients determined by the far-side normal factor). In a terminated FFG, edge~$z$ connects exactly two factors, so \cref{thm:messages} provides exactly two messages. Their log-sum defines the stationarity target:
\begin{equation}\label{eq:ell-from-messages}
    \bar{\ell}(z) = \log \mu_{z \leftarrow \text{softdot}}(z) + \log \mu_{z \leftarrow \exp}(z),
\end{equation}
where $\mu_{z \leftarrow \text{softdot}}$ is Gaussian and $\mu_{z \leftarrow \exp}$ is log-gamma (\cref{thm:messages}). For exponential-family $q(z)$, the gradient of the entropy satisfies $\nabla_{\veta} \Ent{q(z)} = -\vF(\veta)\,\veta$, where $\vF(\veta)$ is the Fisher information matrix \citep{khan_bayesian_2023}. Setting $\nabla_{\veta} F_z = 0$ and rearranging gives the stationarity condition
\begin{equation}\label{eq:stationary-z-generic}
    \veta^* = \vF(\veta^*)^{-1}\, \nabla_{\veta^*} \Exp{q^*(z)}{-\bar{\ell}(z)}.
\end{equation}

The non-conjugate term in \eqref{eq:exp-contribution} requires $\Exp{q(z)}{e^z}$ and its derivatives with respect to~$\veta$. Under the Gaussian form constraint, the moment-generating function provides this in closed form:
\begin{equation}\label{eq:mgf}
    \Exp{q(z)}{e^z} = \exp\!\left(m + \tfrac{v}{2}\right),
\end{equation}
and its derivatives with respect to $\veta$ are analytic expressions of $(\eta_1, \eta_2)$. This is precisely the Q-conjugacy condition \citep{lukashchuk_qconjugate_2024}: the expectations under $q(z)$ of all log-factors adjacent to~$z$ are closed-form functions of $\veta$. Combined with the closed-form Fisher information for the Gaussian in natural parameters,
\begin{equation}\label{eq:fisher-gaussian}
    \vF(\veta) = \begin{pmatrix} -\frac{1}{2\eta_2} & \frac{\eta_1}{2\eta_2^2} \\[4pt] \frac{\eta_1}{2\eta_2^2} & \frac{1}{2\eta_2^2} - \frac{\eta_1^2}{2\eta_2^3} \end{pmatrix},
\end{equation}
the stationarity condition \eqref{eq:stationary-z-generic} yields closed-form fixed-point equations for $\veta^*$.\footnote{By ``closed-form'' we mean that every term in \eqref{eq:stationary-z-generic}, including the expectations, the Fisher information, and their gradients, is an analytic function of $\veta$, with no intractable integrals. The resulting equation is nevertheless nonlinear in $\veta^*$ (through the moment-generating function \eqref{eq:mgf}), so it does not admit a closed-form \emph{solution}. What it does admit is a closed-form \emph{update scheme}: each evaluation of the right-hand side of \eqref{eq:stationary-z-generic} is available in closed form, and iterating this map converges to the fixed point. In practice, we solve \eqref{eq:stationary-z-generic} by natural gradient descent on the exponential family manifold \citep{lukashchuk_exponentialfamilymanifoldsjl_2025} using the Riemannian optimization framework Manopt.jl \citep{bergmann_manoptjl_2022}; see \cref{app:experimental} for details.}

The derivation for $q(\gamma)$ on the $\gamma$~edge is analogous: the Gamma form constraint provides closed-form sufficient statistics $\Exp{q(\gamma)}{\log \gamma} = \psi(\alpha) - \log \beta$ and
$\Exp{q(\gamma)}{\gamma} = \alpha/\beta$, where $\psi$ denotes the digamma function and $\alpha$, $\beta$ are the shape and rate of the Gamma marginal $q(\gamma)$, ensuring Q-conjugacy on that edge as well. The full derivation is deferred to \cref{app:gamma-q-conjugacy}.

The specific conjugate partner on the far side of the exp link---whether softdot, normal, gamma, or equality---does not affect the argument: \cref{thm:messages} guarantees it always produces a message of the matching line type, so the coefficients in \eqref{eq:exp-contribution} change but the closed-form structure is preserved. We summarize the result.

\begin{thm}[Closed-form marginals]\label{thm:marginals}
Let $\mathcal{G}$ be a proper factor graph (\cref{def:proper-ffg}) with full mean-field factorization on all non-deterministic factors and form constraints matching the line types. Then for every edge in~$\mathcal{G}$, the BFE-stationary marginal exists in closed form: at edges not adjacent to an exp link, as the normalized product of two same-family messages; at edges adjacent to an exp link, as the solution of a closed-form fixed-point equation.
\end{thm}

Together, \cref{thm:messages,thm:marginals} establish that both operations of the runtime---computing messages and computing marginals---are closed-form for any proper factor graph. Full mean-field is the worst case; structured factorization at conjugate boundaries can only improve the approximation.


%% file: figures/runtime_messages.tex
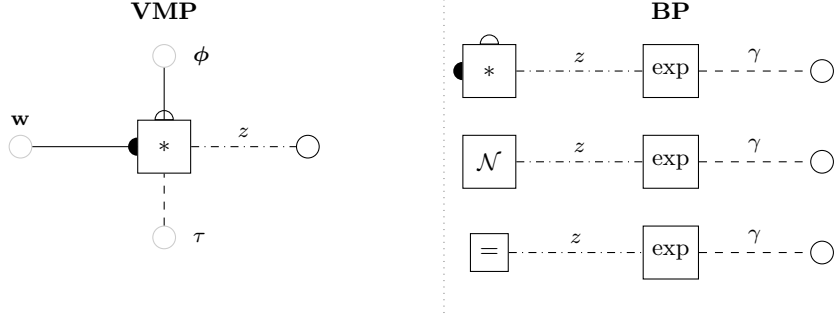
\begin{figure*}[t]
\centering
\begin{tikzpicture}[every node/.style={font=\small}]

    \node[font=\small\bfseries] at (1.5, 2.8) {VMP};

    \node[box] at (1.5, 1.0) (sd) {$*$};
    \fill (sd.west) ++(0, 0.12) arc (90:270:0.12) -- cycle;
    \draw[fill=white] (sd.north) ++(0.12, 0) arc (0:180:0.12) -- cycle;

    \node[roundbox, draw=gray!40] at (-0.4, 1.0) (sd_w) {};
    \draw (sd_w) -- (sd);
    \node[above=0mm of sd_w, font=\scriptsize] {$\vw$};

    \node[roundbox, draw=gray!40] at (1.5, 2.2) (sd_phi) {};
    \draw (sd_phi) -- (sd);
    \node[right=1mm of sd_phi, font=\scriptsize] {$\vphi$};

    \node[roundbox, draw=gray!40] at (1.5, -0.2) (sd_tau) {};
    \draw[dashed] (sd_tau) -- (sd);
    \node[right=1mm of sd_tau, font=\scriptsize] {$\tau$};

    \node[roundbox] at (3.4, 1.0) (sd_z) {};
    \draw[dash dot] (sd) -- (sd_z) node[midway, above, font=\scriptsize] {$z$};

    \draw[dotted, gray] (5.2, -1.2) -- (5.2, 3.0);

    \begin{scope}[xshift=7.0cm]
    \node[font=\small\bfseries] at (1.2, 2.8) {BP};

    \node[box] at (-1.2, 2.0) (r1_sd) {$*$};
    \fill (r1_sd.west) ++(0, 0.12) arc (90:270:0.12) -- cycle;
    \draw[fill=white] (r1_sd.north) ++(0.12, 0) arc (0:180:0.12) -- cycle;
    \node[box] at (1.2, 2.0) (r1_exp) {$\exp$};
    \draw[dash dot] (r1_sd) -- (r1_exp) node[midway, above, font=\scriptsize] {$z$};
    \node[roundbox] at (3.2, 2.0) (r1_g) {};
    \draw[dashed] (r1_exp) -- (r1_g) node[midway, above, font=\scriptsize] {$\gamma$};

    \node[box] at (-1.2, 0.8) (r2_n) {$\gauss$};
    \node[box] at (1.2, 0.8) (r2_exp) {$\exp$};
    \draw[dash dot] (r2_n) -- (r2_exp) node[midway, above, font=\scriptsize] {$z$};
    \node[roundbox] at (3.2, 0.8) (r2_g) {};
    \draw[dashed] (r2_exp) -- (r2_g) node[midway, above, font=\scriptsize] {$\gamma$};

    \node[smallbox] at (-1.2, -0.4) (r3_eq) {$=$};
    \node[box] at (1.2, -0.4) (r3_exp) {$\exp$};
    \draw[dash dot] (r3_eq) -- (r3_exp) node[midway, above, font=\scriptsize] {$z$};
    \node[roundbox] at (3.2, -0.4) (r3_g) {};
    \draw[dashed] (r3_exp) -- (r3_g) node[midway, above, font=\scriptsize] {$\gamma$};

    \end{scope}

\end{tikzpicture}
\caption{Two modes of message computation. \textbf{(a)}~Under mean-field factorization constraints, the message from the softdot toward~$z$ depends only on the marginal \emph{types} of its other edges: $q(\vw) \in \gauss$ and $q(\vphi) \in \gauss$ (solid lines), $q(\tau) \in \gammad$ (dashed line). Which factors are connected on the other side of these edges is irrelevant; the line styles act as a type system. \textbf{(b)}~The exp link uses belief propagation (BP); as a deterministic node it cannot be mean-field factorized. The message toward~$\gamma$ depends on the actual message arriving on~$z$: $\mu_{\gamma \leftarrow \exp}(\gamma) = \int \delta(\gamma {-} e^z)\, \mu_{z \to \exp}(z)\, \mathrm{d}z$. We enumerate all possible connections on the~$z$ edge: softdot, normal, and equality each send a Gaussian message under mean-field, so the outgoing message on~$\gamma$ is always log-normal. The equality node \eqref{eq:equality} also uses BP, but its message is simply the product of incoming messages, which preserves the type; the exp link is the only BP node where the enumeration is non-trivial.}
\label{fig:runtime-messages}
\end{figure*}

%% file: sections/expressiveness.tex

By \cref{thm:messages,thm:marginals}, any proper factor graph (\cref{def:proper-ffg}) has closed-form inference. The question is therefore: how expressive are proper factor graphs?

\begin{proposition}[Decision tree encoding]\label{prop:decision-tree}
For any binary decision tree $T$ of depth $d$ with axis-aligned splits and linear leaf functions, there exists a proper factor graph that implements $T$ in the limit $\tau \to \infty$.
\end{proposition}

\begin{proof}
By induction on $d$. The base case ($d{=}0$) is a single softdot factor. For $d > 0$, the root split is implemented by a depth-2 router with opposing switches (\cref{sec:depth2}); each child subtree of depth $d{-}1$ is implemented by the inductive hypothesis.
\end{proof}

Consider the output variable $y_j$ in the depth-2 construction (\cref{fig:depth2}): it is shared across all branches via equality nodes, and its posterior mean $\mathbb{E}_q[y_j](\vx)$ is determined by the precision-weighted combination of the expert predictions $\hat{y}_i$. In the sharp-routing limit ($\tau \to \infty$), each routing decision becomes deterministic and the dominant expert's prediction is selected in each region. Stacking $d$ layers of splits produces $2^d$ regions, and each region selects one expert's constant prediction~$\hat{y}_i$. The posterior mean is then a so-called simple function (a finite sum of constants times indicator functions of measurable sets), which can approximate any measurable function to arbitrary accuracy \citep[Ch.~1, Theorem~1.17]{rudin_real_2013}. Universal approximation follows.

\begin{corollary}[Universal approximation]\label{cor:universal}
For any continuous $f\!: \mathcal{X} \to \R$ on compact $\mathcal{X} \subset \R^p$ and any $\epsilon > 0$, there exists a proper factor graph built by stacking split-branch routing layers such that $\| \mathbb{E}_{q}[y_j](\vx) - f(\vx) \|_\infty < \epsilon$ in the sharp-routing limit, with all inference updates closed-form.
\end{corollary}

For finite~$\tau$ the posterior over $y_j$ carries calibrated uncertainty about the routing decisions (\cref{fig:xor-uncertainty}). Small $\tau$ yields wide posteriors over the router score~$h$, so both branches contribute and the model hedges across experts; large $\tau$ concentrates the posterior on a single branch. This uncertainty is a direct output of inference, not an add-on; it can be propagated downstream to any quantity that depends on~$y_j$. The connection to classical neural universal approximation theorems is discussed in \cref{app:neural}.

\input{figures/xor_uncertainty}

%% file: figures/xor_uncertainty.tex
\begin{figure}[t]
\centering
\includegraphics[width=\columnwidth]{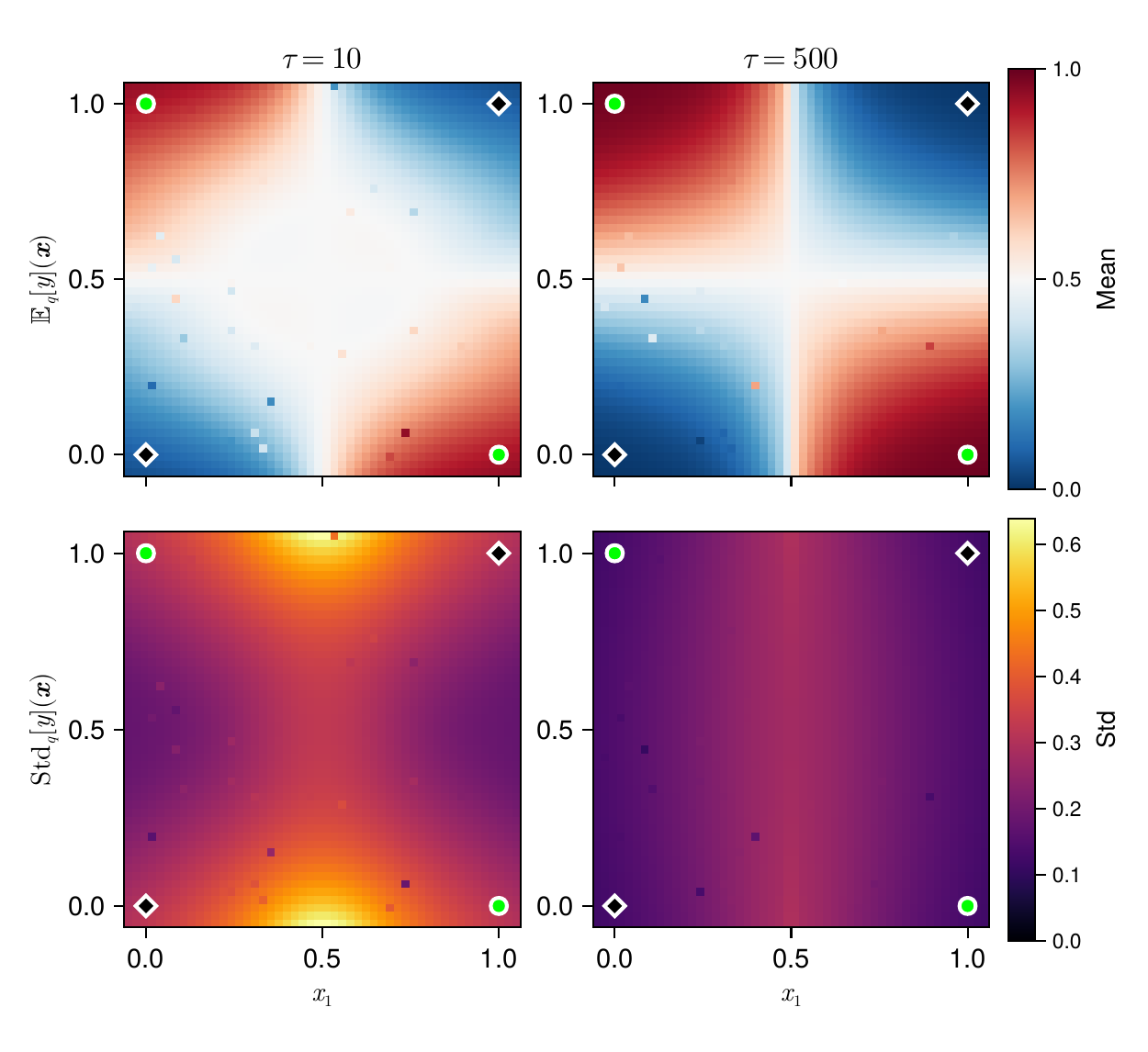}
\caption{Posterior prediction for the XOR encoding with two depth-2 experts at different routing precisions~$\tau$. \textbf{Top:} the posterior mean $\mathbb{E}_q[y_j](\vx)$ transitions from a smooth blend ($\tau{=}10$) to sharp classification regions ($\tau{=}500$). \textbf{Bottom:} the posterior standard deviation reveals routing uncertainty---large everywhere under soft routing, concentrated at the decision boundary under sharp routing. Explicit parameter values are given in \cref{app:xor}.}
\label{fig:xor-uncertainty}
\end{figure}

%% file: sections/application.tex

Time series forecasting is a natural testbed for the framework. Multiple forecasters with different inductive biases (linear, convolutional, recurrent) are trained independently on the same data, and their relative reliability varies with the temporal context: one forecaster may excel during trend regimes while another is more accurate around seasonal peaks. The task of predicting the next time window is a standard regression problem, but the interesting question is \emph{which expert to trust given the current context}. Our framework answers this question probabilistically: the precision-gated mechanism learns input-dependent trust from the recent history.
\subsection{Models and inference variants}

We evaluate three model structures at increasing depth. The Static ensemble (Depth~0, \cref{fig:depth0}) learns a single precision~$\gamma_i$ per forecaster, recovering Cochran's inverse-variance weighting \citep{cochran_problems_1937}. The Dynamic model (Depth~1, \cref{fig:depth1}) replaces each gamma prior with a precision word~$\pi$, making trust input-dependent. The Noisy model (Depth~1) augments the Dynamic model with an additional normal factor with precision~$\kappa$ between the expert-combined variable and the observation, separating forecaster disagreement from irreducible predictive noise; the factor graph and GraphPPL.jl code are given in \cref{app:graphppl-noisy}.

For the Dynamic and Noisy models, inference maintains a posterior $q(\vw_i)$ over the weight vectors. We consider two inference variants that share the same model but differ in the covariance constraint: full covariance $q(\vw_i) = \gauss(\vw_i \mid \vm_i, \vV_i)$, or diagonal covariance $q(\vw_i) = \gauss(\vw_i \mid \vm_i, \text{diag}(\vv_i))$, reducing parameters from $O(d^2)$ to $O(d)$ per expert. This gives five configurations: Static, Dynamic, Dynamic Diagonal, Noisy, and Noisy Diagonal. All use the same closed-form message passing of \cref{sec:runtime}.

\subsection{Setup}

We combine $n = 7$ forecasters (five neural: CNN, DLinear, NLinear \citep{zeng_are_2023}, LSTM \citep{hochreiter_long_1997}, NConv; two constant: 10th and 90th quantiles) on five benchmarks (ETTh1, ETTh2 \citep{zhou_informer_2021}, Exchange Rate, Electricity \citep{trindade_electricityloaddiagrams20112014_2015}, Traffic) with horizons $H \in \{96, 192, 336, 720\}$. Feature vectors $\vphi(\vx_j)$ are obtained from a VAE with a latent dimension of 64. For comparison, we include softmax-based Mixture-of-Experts (MoE) \citep{jacobs_adaptive_1991} baselines with single-layer and two-layer gating networks. Full experimental details, dataset statistics, and parameter counts are in the appendix (\cref{app:experimental,tab:datasets,tab:methods_params}). We report MSE and negative log-likelihood (NLL; see \cref{app:baseline_nll} for MoE baseline NLL computation). Our framework provides calibrated predictive distributions via marginal posteriors, not just point forecasts.

\begin{figure*}[!t]
    \centering
    \begin{subfigure}[t]{0.49\textwidth}
        \centering
        \includegraphics[width=\linewidth]{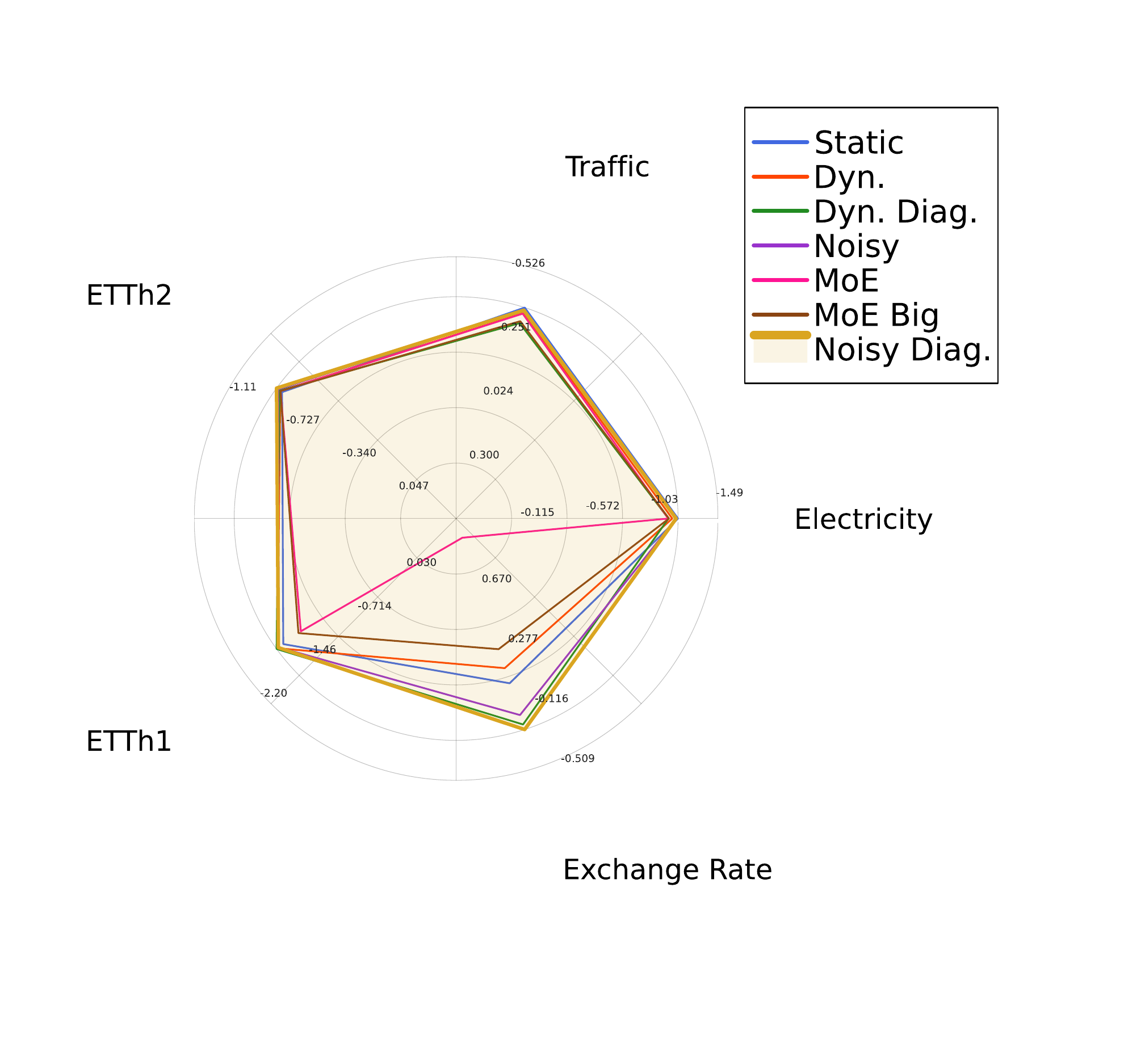}
        \caption{All methods: Log MSE}
        \label{fig:radar_mse_all}
    \end{subfigure}
    \hfill
    \begin{subfigure}[t]{0.49\textwidth}
        \centering
        \includegraphics[width=\linewidth]{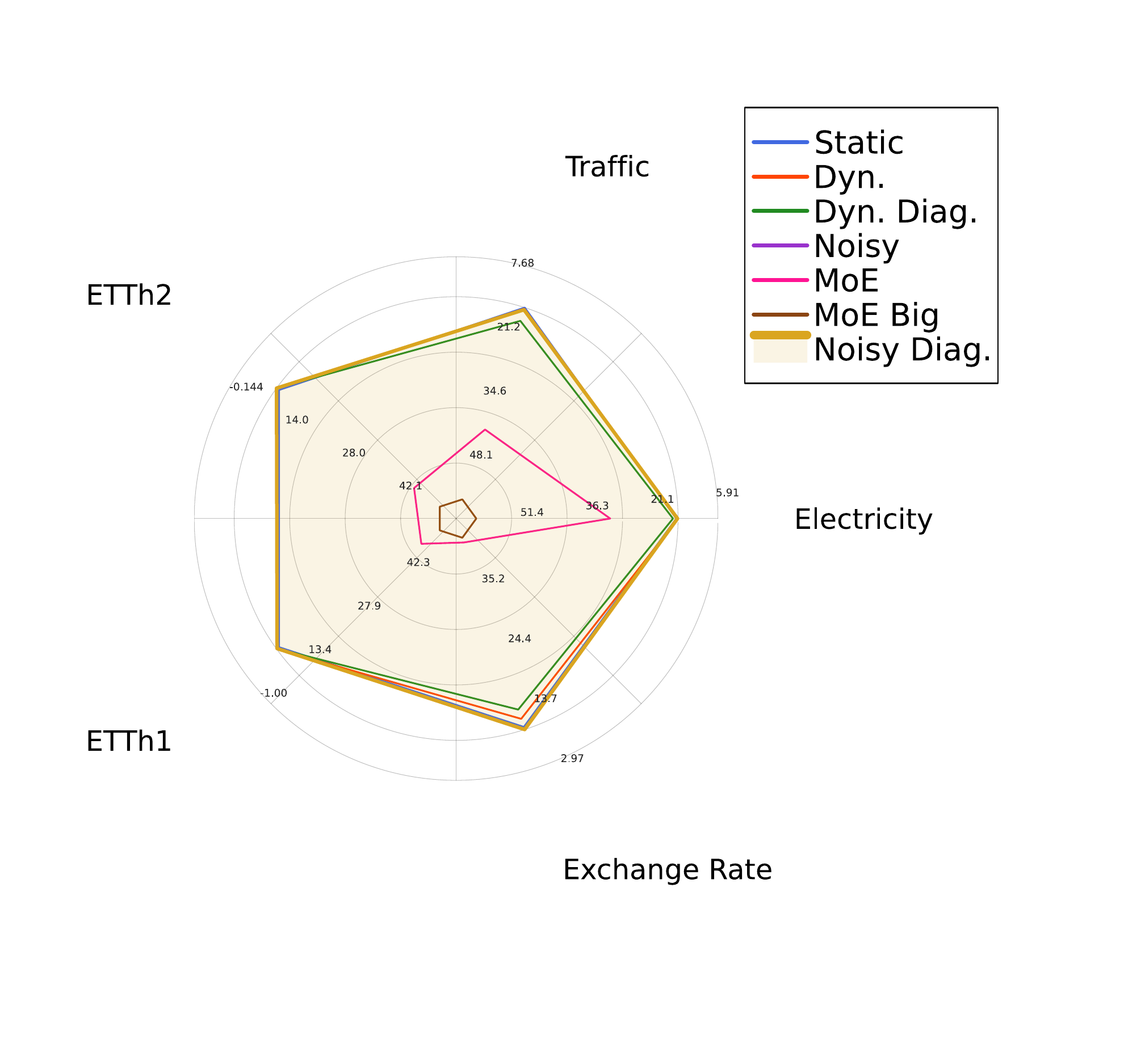}
        \caption{All methods: Log NLL}
        \label{fig:radar_nll_all}
    \end{subfigure}

    \vspace{0.5em}

    \begin{subfigure}[t]{0.49\textwidth}
        \centering
        \includegraphics[width=\linewidth]{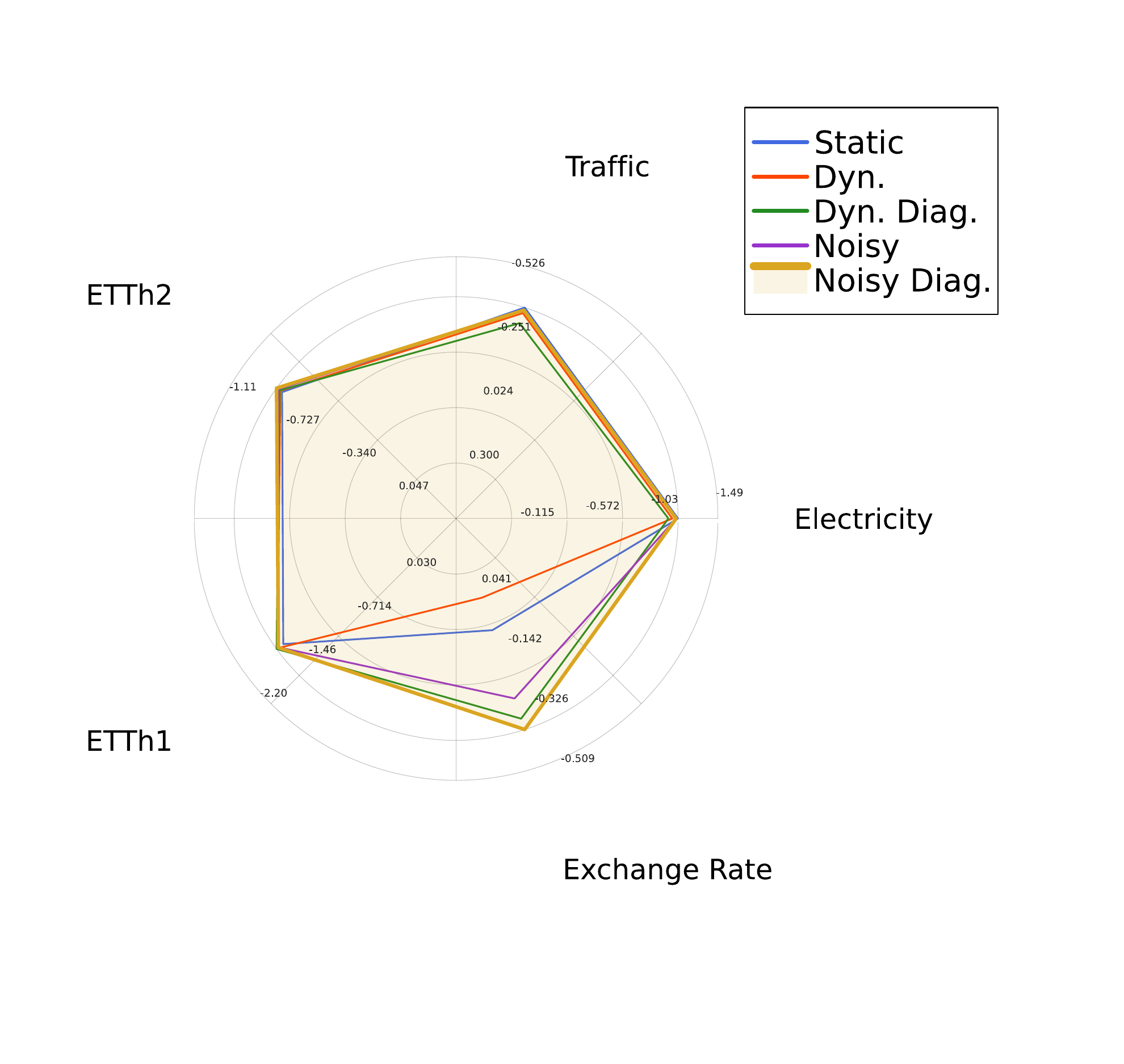}
        \caption{Our framework: Log MSE}
        \label{fig:radar_mse_bayes}
    \end{subfigure}
    \hfill
    \begin{subfigure}[t]{0.49\textwidth}
        \centering
        \includegraphics[width=\linewidth]{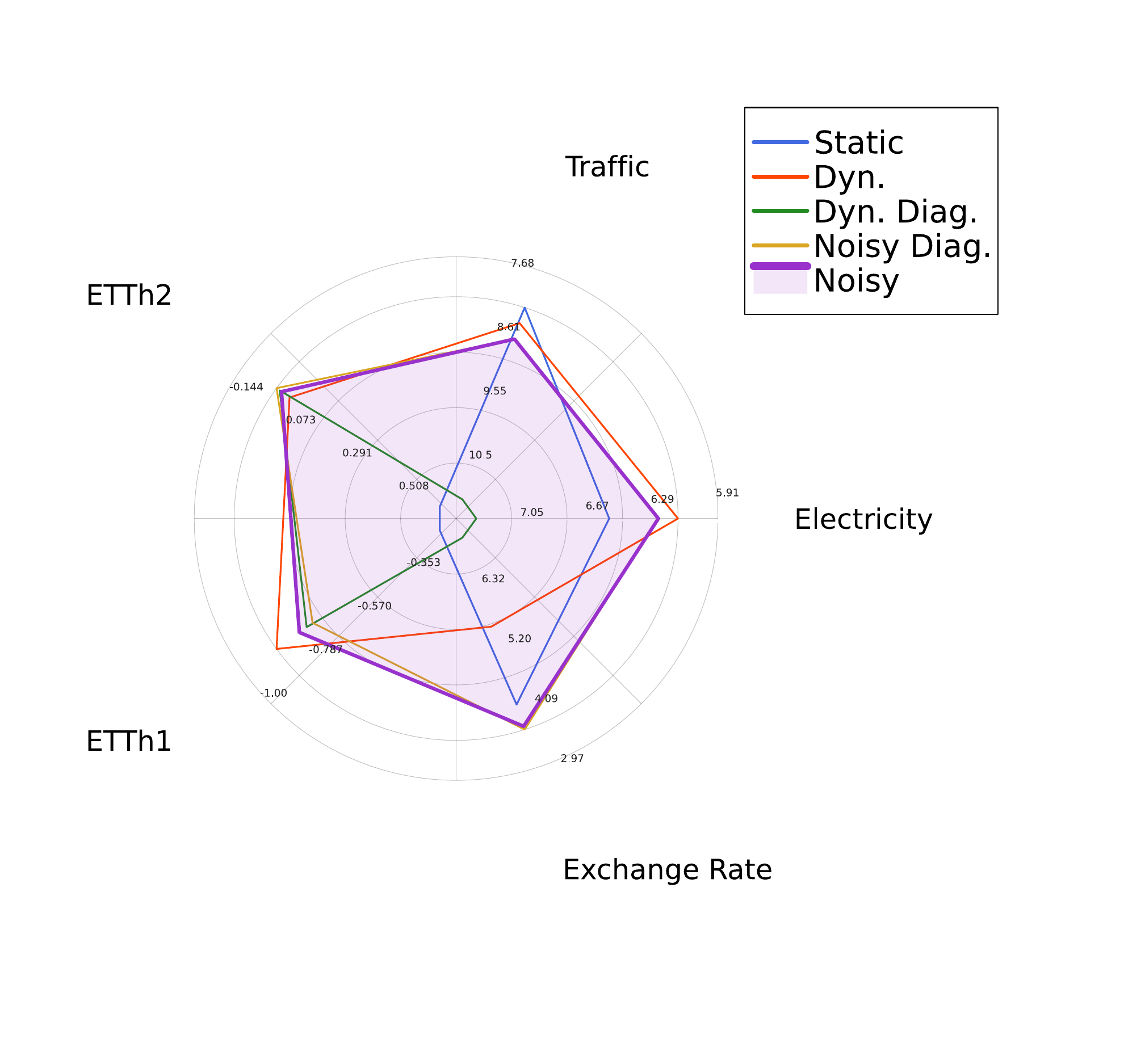}
        \caption{Our framework: Log NLL}
        \label{fig:radar_nll_bayes}
    \end{subfigure}
    \caption{Radar charts of log-transformed MSE and NLL averaged over all horizons. Each axis corresponds to a dataset; larger area indicates better performance. (a,\,b)~All methods: the Noisy Diagonal variant attains the largest overall area; larger neural MoE models improve MSE but become overconfident, resulting in poor NLL. (c,\,d)~Our framework only: Noisy Diagonal achieves the best MSE coverage; both noisy variants provide comparable NLL.}
    \label{fig:radar_charts}
\end{figure*}

\subsection{Results}

Full results are in \cref{tab:ensemble_comparison}; radar charts and Pareto frontiers in \cref{fig:radar_charts,fig:pareto_frontier}. Forecasting examples and comparison of models with distribution of experts in Figure \ref{fig:forecast_compare}.

The Static ensemble already provides a strong baseline, matching or surpassing the best individual forecaster on multiple dataset--horizon pairs. The Dynamic variants extend this by adapting expert trust to the input, achieving the best MSE across all ETTh1 horizons with substantially lower NLL. The Diagonal inference variants are particularly strong on multivariate datasets: on Exchange Rate, Noisy Diagonal achieves the best NLL by a large margin while reducing the parameter count from 15\,512 to 952 ($d{=}65$). The softmax-based MoE baselines are unstable: although occasionally competitive in MSE, their NLL values are often non-finite, indicating poor calibration. Overall, the Noisy Diagonal variant provides the best aggregate trade-off between accuracy, calibration, and parameter efficiency (\cref{fig:pareto_frontier}). Code implementation available in repository: \href{https://github.com/biaslab/PrecisionGatedExperts}{Repository}.

%% file: sections/new_discussion.tex

The depth-0 instance of our framework recovers the inverse-variance weighting of \citet{cochran_problems_1937}; our framework generalizes it by learning precisions from data, making them input-dependent and composing them hierarchically.

Our universality result (\cref{cor:universal}) parallels classical universal approximation theorems for neural networks \citep{cybenko_approximation_1989, hornik_approximation_1991}: the core mechanism, a sufficiently fine partition of the input space, is the same; however, our framework replaces gradient-based point estimation with Bethe free energy optimization, yielding marginal posterior distributions.

The current alphabet covers $\times$, $\exp$, and $+$ (the last is not needed for universality but is easy to include); these suffice so that a model shares the topology of the corresponding computational graph but carries uncertainty on every edge. Unlike a deterministic compiler, which executes instructions without an objective, the probabilistic runtime \emph{optimizes} the Bethe free energy at every iteration. This separates the framework from linearization-based approaches such as Laplace propagation \citep{smola_laplace_2003}, which approximate the nonlinearity and optimize no bound; here, the nonlinearity is exact, and only the posterior family is constrained. We do not have a formal lower bound on alphabet size, but we conjecture at least one non-conjugate factor is required: a purely linear--Gaussian alphabet remains linear--Gaussian end-to-end and cannot reach universal approximation; the exponential link plays that role here.

The per-edge stationarity condition \eqref{eq:stationary-z-generic} has the same form as the global optimality condition of \citet[Eq.~5]{khan_information_2025}; our contribution is that it arises \emph{locally} at each non-conjugate edge of the factor graph, so that marginal computation decomposes into independent per-edge problems solved by message passing.

Concurrent work on the generalized HGF \citep{weber_generalized_2026} shares this philosophy but relies on a Laplace approximation updates without expressiveness guarantees; our construction provides both.

%% file: sections/conclusion.tex

The practical success of deep learning rests not only on its expressive power \citep{lecun_deep_2015} but also on its \emph{accessibility}: layers, backpropagation, and an optimizer suffice to compose powerful models from modular building blocks. The factor graph approach follows a similar compositional philosophy \citep{loeliger_introduction_2004}, yet in practice demands considerably more: deriving message rules, selecting message types, and mixing algorithmic techniques for each new problem.

Our contribution is to narrow this gap. A finite alphabet of Q-conjugate factors composes freely, bringing layer-like modularity to probabilistic inference: softdot experts, exponential-link gates, and split-branch routers stack to arbitrary depth with guaranteed closed-form variational message passing. The alphabet is the language, the Bethe free energy is the compiler, and message passing is the runtime \citep{vandelaar_forneylab_2018, bagaev_reactive_2023}; no new derivations are needed. The tradeoff is explicit: the practitioner trades arbitrary differentiable operations for closed-form marginal inference. This limitation is mild—the alphabet is a universal approximator—and should diminish as the library of reusable composed subgraphs grows.

%% file: sections/appendix_ffg.tex

This appendix reviews Forney-style factor graphs and shows how message passing algorithms arise from minimizing the Bethe free energy. We follow the presentation of \citet{senoz_variational_2021}, to which we refer for proofs and further details.

\subsection{Forney-style factor graphs}\label{app:ffg-def}

A Forney-style factor graph (FFG) is an undirected graph $\mathcal{G} = (\mathcal{V}, \mathcal{E})$ with nodes $\mathcal{V}$ and edges $\mathcal{E} \subseteq \mathcal{V} \times \mathcal{V}$ \citep{forney_codes_2001, loeliger_introduction_2004}. We denote the neighboring edges of a node $a \in \mathcal{V}$ by $\mathcal{E}(a)$. Vice versa, for an edge $i \in \mathcal{E}$, the notation $\mathcal{V}(i)$ collects all neighboring nodes. As a notational convention, we index nodes by $a, b, c$ and edges by $i, j, k$, unless stated otherwise.

An FFG represents a factorized positive function
\begin{equation}\label{eq:ffg-factorization}
    f(\vs) = \prod_{a \in \mathcal{V}} f_a(\vs_a),
\end{equation}
where $\vs$ is the collection of all variables, $\vs_a$ collects the argument variables of factor $f_a$, and a node $a \in \mathcal{V}$ corresponds to a factor $f_a$. The neighboring edges $\mathcal{E}(a)$ correspond to the variables $\vs_a$ that are the arguments of $f_a$.

In the FFG representation, a node can be connected to an arbitrary number of edges, while an edge can only be connected to at most two nodes. An edge that connects to only one node (e.g., for an observed variable or a fixed hyperparameter) is called a \emph{half-edge}. When a variable must be shared among more than two factors, FFGs introduce explicit \emph{equality nodes} with factor function
\begin{equation}\label{eq:equality-node}
    f_a(s_i, s_j, s_k) = \delta(s_j - s_i)\,\delta(s_j - s_k),
\end{equation}
which constrain connected edges to carry identical values.

\subsection{Terminated factor graphs}\label{app:terminated-ffg}

If every edge in the FFG has exactly two connected nodes (including equality nodes), we call the graph a \emph{terminated} FFG (TFFG). A half-edge can be closed by an observation factor, e.g.\ $f_{\mathrm{obs}}(s_i) = \delta(s_i - s_i^{\mathrm{obs}})$, or by a unity factor $f_1(s_i) = 1$. The latter leaves the variable uninformative; ordinary priors, such as the Gamma factor on $\gamma_i$, are just regular factors in the graph. All FFGs in this paper are assumed to be terminated before inference is run.

For the Depth-0 model, the same unnormalized factor product
\begin{equation}
    f(\vy, \bold{\hat{y}}, \vgamma)
    =
    \prod_{i=1}^{n} f_{\gammad}(\gamma_i \mid \alpha_i, \beta_i)
    \prod_{\substack{i=1\\ j=1}}^{n,m}
    f_{\mathcal{N}}(\hat{y}_{i,j} \mid y_j, \gamma_i)
\end{equation}
supports different directed graphical-model readings after termination. In smoothing, $y_j$ and $\hat{y}_{i,j}$ are observed, so the terminated graph defines
\[
    p(\vgamma \mid \vy, \bold{\hat{y}}) \propto f(\vy, \bold{\hat{y}}, \vgamma).
\]
In prediction, only the forecaster predictions are observed; $y_*$ is left latent (closed by a unity factor), giving
\[
    p(y_*, \vgamma \mid \bold{\hat{y}}_*) \propto
    \prod_i f_{\gammad}(\gamma_i \mid \alpha_i, \beta_i)
    \prod_i f_{\mathcal{N}}(\hat{y}_{i,*} \mid y_*, \gamma_i).
\]
Thus the FFG fixes the local factorization, while termination specifies which conditional distribution is represented.

\subsection{Bethe free energy}\label{app:bfe}

Given a TFFG $\mathcal{G} = (\mathcal{V}, \mathcal{E})$ for a factorized function $f(\vs) = \prod_{a \in \mathcal{V}} f_a(\vs_a)$, the Bethe free energy (BFE) is defined as \citep{yedidia_constructing_2005}
\begin{align}\label{eq:bethe-free-energy}
    F[q, f] &\triangleq \sum_{a \in \mathcal{V}} \underbrace{\int q_a(\vs_a) \log \frac{q_a(\vs_a)}{f_a(\vs_a)} \, \mathrm{d}\vs_a}_{F[q_a, f_a]} \notag\\
    &\quad + \sum_{i \in \mathcal{E}} \underbrace{\int q_i(s_i) \log \frac{1}{q_i(s_i)} \, \mathrm{d}s_i}_{H[q_i]},
\end{align}
such that the factorized beliefs
\begin{equation}\label{eq:bethe-factorization}
    q(\vs) = \prod_{a \in \mathcal{V}} q_a(\vs_a) \prod_{i \in \mathcal{E}} q_i(s_i)^{-1}
\end{equation}
satisfy the following constraints:
\begin{subequations}\label{eq:bfe-constraints}
\begin{align}
    \int q_a(\vs_a) \, \mathrm{d}\vs_a &= 1, \quad \text{for all } a \in \mathcal{V}, \label{eq:normalization}\\
    \int q_a(\vs_a) \, \mathrm{d}\vs_{a \setminus i} &= q_i(s_i), \quad \text{for all } a \in \mathcal{V} \text{ and all } i \in \mathcal{E}(a). \label{eq:marginalization}
\end{align}
\end{subequations}
The normalization constraint \eqref{eq:normalization} and marginalization constraint \eqref{eq:marginalization} together imply that the edge marginals $q_i(s_i)$ are also normalized. The BFE includes a local free energy term $F[q_a, f_a]$ for each node $a \in \mathcal{V}$ and an entropy term $H[q_i]$ for each edge $i \in \mathcal{E}$. In a TFFG every edge connects exactly two nodes; the edge entropy corrects for the double-counting that would arise from the node-local free energies alone.

The Bethe factorization \eqref{eq:bethe-factorization} and constraints \eqref{eq:bfe-constraints} are summarized by the \emph{local polytope}
\begin{multline}\label{eq:local-polytope}
    \mathcal{L}(\mathcal{G}) = \bigl\{q_a \text{ for all } a \in \mathcal{V} \text{ s.t.~\eqref{eq:normalization}},\\
    \text{and } q_i \text{ for all } i \in \mathcal{E}(a) \text{ s.t.~\eqref{eq:marginalization}}\bigr\},
\end{multline}
which defines the constrained search space for the factorized variational distribution. The problem is then to find the beliefs in the local polytope that minimize the BFE:
\begin{equation}\label{eq:bfe-min}
    q^*(\vs) = \argmin_{q \in \mathcal{L}(\mathcal{G})} F[q, f].
\end{equation}

\subsection{Message passing from stationarity conditions}\label{app:messages-from-stationarity}

Instead of solving the global optimization problem \eqref{eq:bfe-min} directly, we employ the factorization of the BFE and the local polytope to subdivide it into interdependent local objectives. The resulting fixed-point equations can be interpreted as messages on the edges of the FFG. We state the two results needed in this paper; proofs can be found in \citep{kschischang_factor_2001, yedidia_constructing_2005, dauwels_variational_2007, senoz_variational_2021}.

\paragraph{Sum-product message passing (belief propagation).}
The stationary points of the Lagrangian associated with \eqref{eq:bfe-min} yield the sum-product message update. The message from node $b$ to edge $j$ is
\begin{equation}\label{eq:bp-message}
    \mu_{jb}^{(k+1)}(s_j) = \int f_b(\vs_b) \prod_{\substack{i \in \mathcal{E}(b) \\ i \neq j}} \mu_{ib}^{(k)}(s_i) \, \mathrm{d}\vs_{b \setminus j},
\end{equation}
where $k$ is an iteration index and $\mu_{ib}$ denotes the incoming message on edge $i$ towards node $b$. The edge belief is proportional to the product of the two incoming messages:
\begin{equation}\label{eq:edge-belief-bp}
    q_j^*(s_j) = \frac{\mu_{jb}^*(s_j)\, \mu_{jc}^*(s_j)}{\int \mu_{jb}^*(s_j)\, \mu_{jc}^*(s_j)\, \mathrm{d}s_j}.
\end{equation}
On tree-structured graphs the sum-product algorithm recovers the exact posterior; on graphs with cycles, convergence is not guaranteed in general \citep{yedidia_constructing_2005}.

\paragraph{Naive mean-field variational message passing.}
When a naive mean-field factorization constraint is imposed at node $b$, i.e., $q_b(\vs_b) = \prod_{i \in \mathcal{E}(b)} q_i(s_i)$, the local free energy for factor $b$ becomes
\begin{multline}\label{eq:local-fe-mf}
    F[q_b, f_b] = \sum_{i \in \mathcal{E}} \int q_i(s_i) \log q_i(s_i) \, \mathrm{d}s_i \\ - \int \Bigl\{\prod_{i \in \mathcal{E}(b)} q_i(s_i)\Bigr\} \log f_b(\vs_b) \, \mathrm{d}\vs_b.
\end{multline}
The variational message passing update \citep{winn_variational_2005} from node $b$ towards edge $j$ is then given by
\begin{equation}\label{eq:vmp-message}
    \mu_{jb}^{(k+1)}(s_j) = \exp\!\left(\int \Bigl\{\prod_{\substack{i \in \mathcal{E}(b) \\ i \neq j}} q_i^{(k)}(s_i) \Bigr\} \log f_b(\vs_b) \, \mathrm{d}\vs_{b \setminus j} \right),
\end{equation}
which is the exponentiated expected log-factor under the current beliefs of all other edges. We use naive mean-field factorization constraints at all non-conjugate factor boundaries in this paper.

\subsection{Message rules for the alphabet}\label{app:message-rules}


These rules, combined with a valid schedule, are sufficient to run inference on any factor graph composed from the alphabet.

\paragraph{Soft-dot} The soft-dot factor node has the form:
\begin{equation*}
    f_*(z \, | \, \vw, \vphi, \tau) = \mathcal{N}(z \, | \, \vw^{\top} \vphi, \tau^{-1}) \, .
\end{equation*}
We assume the node's incoming messages are marginals of the form:
\begin{align*}
    q(z) &= \mathcal{N}(z \, | \, m_z, s_z^2) \\
    q(\vw) &= \mathcal{N}(\vw \, | \, m_{\vw}, s^2_{\vw}) \\
    q(\vphi) &= \mathcal{N}(\vphi \, | \, m_{\vphi}, s^2_{\vphi}) \\
    q(\tau) &= \Gamma(\tau \, | \, \alpha_\tau, \beta_\tau) \, .
\end{align*}
The forward message towards $z$ is:
\begin{align*}
    \mu(z) \!
    &\propto \exp\big( \mathbb{E}_q \big[ \log \mathcal{N}(z \, | \, \vw^{\top} \vphi, \tau^{-1}) \big] \big) \\
    &\propto \exp\big( \mathbb{E}_q \big[ \! - \! \frac{\tau}{2}(z^2 \! - \! 2z \vw^{\top} \vphi \! + \! \vw^{\top} \vphi \vphi^{\top} \vw) \big] \big) \\
    &\propto \exp \! \big( \! - \! \frac{\tau}{2}(z^2 \! - \! 2z m_{\vw}^{\top} m_{\vphi} \! + \! m_{\vw}^{\top} m_{\vphi} m_{\vphi}^{\top} m_{\vw}) \big] \big) \\
    &\propto \mathcal{N}(z \, | \, m_{\vw}^{\top} m_{\vphi}, \, \tau^{-1})
\end{align*}
The backwards message towards $\vw$ is:
\begin{equation}\label{eq:softdot-vmp-w-message}
\begin{aligned}
    \mu(\vw) \! 
    &\propto \! \exp\big( \mathbb{E}_q \big[ \log \mathcal{N}(z \, | \, \vw^{\top} \vphi, \tau^{-1}) \big] \big) \\
    &\propto \! \exp \!\big( \! - \! \frac{\tau}{2}(m_z^2 \! - \! 2m_z \vw^{\top} \! m_{\vphi} \! + \! \vw^{\top} \! (m_{\vphi} m_{\vphi}^{\top} \! + \! S_{\vphi}) \vw) \big) \\
    &\propto \mathcal{N}(\vw \, | \, m_{\vw}, S_{\vw}) \, ,
\end{aligned}
\end{equation}
whose parameters are
\begin{align*}
    m_{\vw} &= \tau(m_{\vphi} m_{\vphi}^{\top} + S_{\vphi}) (m_z m_{\vphi}) \\ 
    S_{\vw} &= \tau^{-1}(m_{\vphi} m_{\vphi}^{\top} + S_{\vphi})^{-1} \, .
\end{align*}
The backwards message towards $\vphi$ is:
\begin{align*}
    \mu(\vphi) \!
    &\propto \! \exp\big( \mathbb{E}_q \big[ \log \mathcal{N}(z \, | \, \vw^{\top} \vphi, \tau^{-1}) \big] \big) \\
    &\propto \! \exp\big( \! - \! \frac{\tau}{2}(m_z^2 \! - \! 2m_z m_{\vw}^{\top} \vphi \! + \! \vphi^{\top}(m_{\vw} m_{\vw}^{\top} \! + \! S_{\vw}) \vphi) \big) \\
    &\propto \mathcal{N}(\vphi \, | \, m_{\vphi}, S_{\vphi}) \, ,
\end{align*}
whose parameters are
\begin{align*}
    m_{\vphi} &= \tau(m_{\vw} m_{\vw}^{\top} + S_{\vw}) (m_z m_{\vw}) \\ 
    S_{\vphi} &= \tau^{-1}(m_{\vw} m_{\vw}^{\top} + S_{\vw})^{-1} \, .
\end{align*}

The backwards message towards $\tau$ is:
\begin{equation}\label{eq:softdot-vmp-tau-message}
\begin{aligned}
    \mu(\tau) &\propto \exp\big(\mathbb{E}_q\big[\log \mathcal{N}(z \, | \, \vw^{\top} \vphi, \tau^{-1}) \big]\big) \\
    &\propto \exp\big( \mathbb{E}_q \big[ \frac{1}{2}\log \tau - \frac{\tau}{2}(z - \vw^{\top}\vphi)^2 \big]\big) \\
    &\propto \tau^{1/2} \exp\Big( - \frac{\tau}{2}\big((m_z - m_{\vw}^{\top}m_{\vphi})^2 + s^2_z \\
    &\qquad \qquad \ + \text{tr}[S_{\vw}(m_{\vphi} m_{\vphi}^{\top} + S_{\vphi})] \big) \Big) \\
    &\propto \mathcal{G}(\tau \, | \, \alpha_\tau, \beta_\tau) \, ,
\end{aligned}
\end{equation}
whose parameters are:
\begin{align*}
    \alpha_\tau &= 3/2\\
    \beta_\tau &= (m_z \! - \! m_{\vw}^{\top}m_{\vphi})^2 \! + \! s^2_z \! + \! \text{tr}[S_{\vw}(m_{\vphi} m_{\vphi}^{\top} \! + \! S_{\vphi})] .
\end{align*}

\paragraph{LogGamma} The LogGamma distribution is a probability distribution on the real line with density
\begin{equation*}
    \mathcal{LG}(x \mid a, b) = \frac{e^{b x} e^{-e^{x}/a}}{a^{b} \Gamma(b)}, \quad -\infty < x < \infty, \; a > 0, \; b > 0 \, .
\end{equation*}
It arises as the change of variable $x = \log \gamma$ applied to a Gamma-distributed $\gamma$.

\paragraph{Exponential} The exponential link factor node is:
\begin{equation*}
    f_{\exp}(\gamma \, | \, z) = \delta\big(\gamma - \exp(z) \big) \, .
\end{equation*}
The node's incoming messages are of the form:
\begin{align*}
    \mu(\gamma) = \mathcal{G}(\gamma \, | \, \alpha_\gamma, \beta_\gamma) \, , \quad
    \mu(z) = \mathcal{N}(z \, | \, m_z, s^2_z) \, .
\end{align*}
The forwards message to $\gamma$ is:
\begin{align*}
    \mu(\gamma) &= \int \delta\big(\gamma - \exp(z)\big) \mu(z) \mathrm{d}z \\
    &= | \frac{\partial \log( \gamma)}{\partial \gamma}| \int \delta\big(\log(\gamma) - z \big) \mathcal{N}(z \, | \, m_z, s^2_z) \mathrm{d}z \\
    &= \frac{1}{\gamma} \mathcal{N}(\log(\gamma) \, | \, m_z, s^2_z) \\
    &= \mathcal{LN}( \gamma \, | \, m_z, s_2^2) \, .
\end{align*}
The backwards message towards $z$ is obtained by an exact change of variable. Substituting $\gamma = \exp(z)$ with Jacobian $|\partial \exp(z) / \partial z| = \exp(z)$:
\begin{align*}
    \mu(z) &= \int \delta\big(\gamma - \exp(z)\big) \mu(\gamma) \,\mathrm{d}\gamma \\
    &= \Big|\frac{\partial \exp(z)}{\partial z}\Big| \int \delta\big(\exp(z) - \gamma\big) \, \mathcal{G}(\gamma \mid \alpha_\gamma, \beta_\gamma) \,\mathrm{d}\gamma \\
    &= \exp(z) \cdot \mathcal{G}\big(\exp(z) \mid \alpha_\gamma, \beta_\gamma\big) \\
    &= \exp(z) \cdot \frac{\beta_\gamma^{\alpha_\gamma}}{\Gamma(\alpha_\gamma)} \exp(z)^{\alpha_\gamma - 1} \exp\!\big(-\beta_\gamma \exp(z)\big) \\
    &= \frac{\beta_\gamma^{\alpha_\gamma}}{\Gamma(\alpha_\gamma)} \exp(\alpha_\gamma z) \exp\!\big(-\beta_\gamma \exp(z)\big) \\
    &= \mathcal{LG}\big(z \mid \beta_\gamma^{-1}, \alpha_\gamma\big) \, .
\end{align*}

\paragraph{Gamma} The Gamma distribution factor node is:
\begin{equation*}
     f_{\mathcal{G}}(\gamma \, | \, \alpha, \beta) = \mathcal{G}(\gamma \, | \, \alpha, \beta) = \frac{\beta^{\alpha}}{\Gamma(\alpha)} \gamma^{\alpha-1} \exp(-\beta \gamma) \, .
 \end{equation*}
We assume incoming messages of the form:
\begin{align*}
    q(\gamma) = \mathcal{G}(\gamma \, | \, \eta_\gamma, \zeta_\gamma) \, , \quad
    q(\beta) = \mathcal{G}(\beta \, | \, \eta_\beta, \zeta_\beta) \, ,
\end{align*}
where $\eta$ is a shape and $\zeta$ a rate parameter. The forward message towards $\gamma$ is:
\begin{align*}
    \mu(\gamma) &\propto \exp\big(\mathbb{E}_q \big[ \log \Gamma(\gamma \, | \, \alpha, \beta) \big] \big) \\
    &\propto \exp\big(\mathbb{E}_q \big[ (\alpha - 1)\log (\gamma) -\beta \gamma \big] \big) \\
    &\propto \gamma^{\alpha-1} \exp\big( - \frac{\eta_\beta}{\zeta_\beta} \gamma  \big) \\
    &\propto \mathcal{G}(\gamma \, | \, \alpha, \frac{\eta_\beta}{\zeta_\beta})
\end{align*}
The backwards message towards $\beta$ is:
\begin{align*}
    \mu(\beta) &\propto \exp\big(\mathbb{E}_q \big[ \log \Gamma(\gamma \, | \, \alpha, \beta) \big] \big) \\
    &\propto \exp\big( \mathbb{E}_q \big[ \alpha \log \beta  - \beta \gamma \big] \big) \\
    &\propto \beta^\alpha \exp\big( - \beta \frac{\eta_\gamma}{\zeta_\gamma} \big) \\
    &\propto \mathcal{G}(\beta \, | \, \alpha, \frac{\eta_\gamma}{\zeta_\gamma}) \, .
\end{align*}

\paragraph{Normal} The Normal factor node has the form:
\begin{equation*}
    f_{\mathcal{N}}(y \, | \, \mu, \tau) = \mathcal{N}(y \, | \, \mu, \tau^{-1}) \, .
\end{equation*}
We assume the incoming messages are:
\begin{align*}
    q(y,\mu) &= q(y \, | \, \mu, s^2_y) q(\mu \, | \, m_{\mu}, s^2_{\mu}) \\
    &= \mathcal{N}\left(\begin{bmatrix} y \\ \mu \end{bmatrix} \, | \, \begin{bmatrix} m_{\mu} \\ m_{\mu} \end{bmatrix}, \begin{bmatrix} s^2_y + s^2_{\mu} & s^2_{\mu} \\ s^2_{\mu} & s^2_{\mu} \end{bmatrix}\right)\\
    q(\tau) &= \mathcal{G}(\tau \, | \, \alpha_\tau, \beta_\tau)
\end{align*}
Then the forward message to $y$ is:
\begin{align*}
    \mu(y) &\propto \exp\big( \mathbb{E}_q \big[\log \mathcal{N}(y \, | \, \mu, \tau^{-1}) \big] \big) \\
    &\propto \exp\big( \mathbb{E}_q \big[\frac{1}{2}\log \tau -\frac{\tau}{2}(y^2 - 2y\mu + \mu^2)  \big] \big) \\
    &\propto \exp\big(  - \frac{\alpha_\tau}{2 \beta_\tau}(y^2 \! - \! 2 y m_{\mu} \! + \! m_{\mu}^2 )  \big) \\
    &\propto \mathcal{N}\big(y \, | \, m_{\mu}, \frac{\alpha_\tau}{\beta_{\tau}}\big) \, ,
\end{align*}
where $\psi(\cdot)$ is the digamma function. The backwards message towards $\mu$ is:
\begin{align*}
    \mu(\mu) 
    &\propto \exp\big( \mathbb{E}_q \big[\frac{1}{2}\log \tau -\frac{\tau}{2}(y^2 - 2y\mu + \mu^2)  \big] \big) \\ 
    &\propto \exp\big(  - \frac{\alpha_\tau}{2 \beta_\tau}(m_{\mu}^2 \! - \! 2 m_{\mu} \mu \! + \! \mu^2 )  \big) \\
    &\propto \mathcal{N}\big(\mu \, | \, m_{\mu}, \frac{\alpha_\tau}{\beta_\tau}\big)
\end{align*}
The backwards message towards $\tau$ is:
\begin{align*}
    \mu(\tau)
    &\propto \exp\big( \mathbb{E}_q \big[ \frac{1}{2}\log \tau - \frac{\tau}{2}(y^2 - 2y\mu + \mu^2) \big]\big) \\
    &\propto \tau^{1/2} \exp\big( - \frac{\tau}{2} s^2_y \big) \\
    &\propto \mathcal{G}(\tau \, | \, \frac{3}{2}, s^2_y) \, .
\end{align*}
 
\paragraph{Equality} The equality node has the following form:
\begin{equation*} 
    f_=(x_1, x_2, x_3) = \delta\big(x_1 - x_2\big) \delta\big(x_2 - x_3\big)
\end{equation*}
The forward message to $x_1$ is:
\begin{align*}
    \mu(x_1) &= \iint \delta\big(x_1 - x_2\big) \delta\big(x_1 - x_3\big) \mu_{x_2}(x_2) \mu_{x_3}(x_3) \mathrm{d}x_2 \mathrm{d}x_3 \\
    &= \int \delta\big(x_1 - x_3\big) \mu_{x_2}(x_1) \mu_{x_3}(x_3) \mathrm{d}x_3 \\
    &= \mu_{x_2}(x_1) \mu_{x_3}(x_1) \, .
\end{align*}
In words, the outgoing message is the product of incoming messages.
In fact, the equality node is symmetric; the outgoing message towards any edge is the product of incoming messages on other edges \citep{loeliger_factor_2007}.

\subsection{Q-conjugacy on the $\gamma$ edge}\label{app:gamma-q-conjugacy}

This section derives the closed-form stationarity condition for the marginal $q(\gamma)$ on an edge adjacent to the exponential link, under the Gamma form constraint $q(\gamma) = \mathcal{G}(\gamma \mid \alpha, \beta)$. The derivation parallels the Gaussian $z$-edge treatment in the main text but uses the Gamma exponential family structure.

\paragraph{Local BFE terms.}
On the $\gamma$ edge, two messages arrive from the two adjacent factors. From the exponential link, the forward message is log-Normal (\cref{app:message-rules}):
\begin{equation}\label{eq:msg-lognormal}
    \mu_{\gamma \leftarrow \exp}(\gamma) = \mathcal{LN}(\gamma \mid m_z, s_z^2) \, ,
\end{equation}
whose log-density is
\begin{equation*}
    \log \mathcal{LN}(\gamma \mid m_z, s_z^2) = -\frac{(\log \gamma - m_z)^2}{2 s_z^2} - \log \gamma - \tfrac{1}{2}\log(2\pi s_z^2) \, .
\end{equation*}
From the far-side conjugate factor (e.g., \ the normal or gamma factor), the message is Gamma:
\begin{equation}\label{eq:msg-gamma-far}
    \mu_{\gamma \leftarrow \mathrm{far}}(\gamma) = \mathcal{G}(\gamma \mid a, b) \, ,
\end{equation}
whose log-density is $(a-1)\log \gamma - b\gamma + a\log b - \log\Gamma(a)$. The stationarity target is
\begin{equation}\label{eq:ell-gamma}
    \bar{\ell}(\gamma) = \log \mu_{\gamma \leftarrow \exp}(\gamma) + \log \mu_{\gamma \leftarrow \mathrm{far}}(\gamma) \, .
\end{equation}

\paragraph{Gamma as exponential family.}
The Gamma distribution $\mathcal{G}(\gamma \mid \alpha, \beta)$ with shape $\alpha$ and rate $\beta$ is an exponential family with natural parameters
\begin{equation}\label{eq:gamma-natural}
    \veta = \begin{pmatrix} \eta_1 \\ \eta_2 \end{pmatrix} = \begin{pmatrix} \alpha - 1 \\ -\beta \end{pmatrix}, \quad \eta_1 > -1, \; \eta_2 < 0 \, ,
\end{equation}
sufficient statistics $\vT(\gamma) = (\log \gamma, \; \gamma)^\top$, and log-partition function
\begin{equation}\label{eq:gamma-logpartition}
    A(\veta) = \log \Gamma(\eta_1 + 1) - (\eta_1 + 1)\log(-\eta_2) \, .
\end{equation}

\paragraph{Closed-form expectations.}
The expected sufficient statistics under $q(\gamma) = \mathcal{G}(\gamma \mid \alpha, \beta)$ are given by the gradient of the log-partition:
\begin{equation}\label{eq:gamma-expectations}
    \nabla_{\veta} A(\veta) = \begin{pmatrix} \Exp{q(\gamma)}{\log \gamma} \\[4pt] \Exp{q(\gamma)}{\gamma} \end{pmatrix} = \begin{pmatrix} \psi(\eta_1 + 1) - \log(-\eta_2) \\[4pt] -(\eta_1 + 1)/\eta_2 \end{pmatrix},
\end{equation}
where $\psi$ denotes the digamma function. Both are analytic functions of $\veta$.

The log-Normal message \eqref{eq:msg-lognormal} additionally requires the second moment $\Exp{q}{(\log \gamma)^2}$. By the variance identity,
\begin{align}\label{eq:gamma-logsq}
    \Exp{q(\gamma)}{(\log \gamma)^2} &= \mathrm{Var}_{q}[\log \gamma] + \bigl(\Exp{q(\gamma)}{\log \gamma}\bigr)^2 \notag \\
    &= \psi'(\eta_1 + 1) + \bigl(\psi(\eta_1 + 1) - \log(-\eta_2)\bigr)^2,
\end{align}
where $\psi'$ is the trigamma function. This is again an analytic function of $\veta$.

Writing $\bar{g} \triangleq \Exp{q}{\log\gamma}$, $\bar{g}_2 \triangleq \Exp{q}{(\log\gamma)^2}$, and $\bar{\gamma} \triangleq \Exp{q}{\gamma}$, the expected energy under $q(\gamma)$ is
\begin{align}\label{eq:exp-ell-gamma}
    \Exp{q}{-\bar{\ell}(\gamma)}
    &= \tfrac{1}{2s_z^2}\,\bar{g}_2
       - \tfrac{m_z}{s_z^2}\,\bar{g}
       + \bar{g} \notag \\
    &\quad - (a\!-\!1)\,\bar{g}
       + b\,\bar{\gamma}
       + \mathrm{const} \, ,
\end{align}
where every expectation is a closed-form function of $\veta$ via \eqref{eq:gamma-expectations}--\eqref{eq:gamma-logsq}. This is the Q-conjugacy condition: $\Exp{q}{-\bar{\ell}(\gamma)}$ and its gradient with respect to $\veta$ are analytic functions of $\veta$.

\paragraph{Fisher information matrix.}
The Fisher information matrix of the Gamma distribution in natural parameters is the Hessian of the log-partition \eqref{eq:gamma-logpartition}:
\begin{equation}\label{eq:fisher-gamma}
    \vF(\veta) = \nabla_{\veta}^2 A(\veta) = \begin{pmatrix} \psi'(\eta_1 + 1) & -1/\eta_2 \\[4pt] -1/\eta_2 & (\eta_1 + 1)/\eta_2^2 \end{pmatrix}.
\end{equation}

\paragraph{Stationarity condition.}
Setting the gradient of the local BFE contribution at the $\gamma$ edge to zero and rearranging (as in the main text for the $z$ edge) yields the fixed-point equation
\begin{equation}\label{eq:stationary-gamma}
    \veta^* = \vF(\veta^*)^{-1} \, \nabla_{\veta^*} \Exp{q^*(\gamma)}{-\bar{\ell}(\gamma)} \, .
\end{equation}
Since $\vF(\veta)$ \eqref{eq:fisher-gamma} and $\nabla_{\veta}\Exp{q(\gamma)}{-\bar{\ell}(\gamma)}$ \eqref{eq:exp-ell-gamma} are both closed-form functions of $\veta$, the right-hand side of \eqref{eq:stationary-gamma} is a closed-form expression of $\veta^*$. This gives a closed-form fixed-point equation for the Gamma marginal $q^*(\gamma)$, analogous to the Gaussian fixed-point equation for $q^*(z)$ in the main text.

Differentiating \eqref{eq:exp-ell-gamma} with respect to $\veta$ and writing $\bar{g} = \psi(\eta_1\!+\!1) - \log(-\eta_2)$ and $c_1 = 1 - m_z/s_z^2 - (a\!-\!1)$, the gradient is
\begin{align}\label{eq:grad-ell-gamma}
    &\nabla_{\veta}\Exp{q}{-\bar{\ell}(\gamma)} = \notag \\
    &\begin{pmatrix}
        \frac{\psi'(\eta_1\!+\!1)}{s_z^2}
        \bigl[\bar{g} + \tfrac{\psi''(\eta_1\!+\!1)}{2\,\psi'(\eta_1\!+\!1)}\bigr]
        + c_1\,\psi'(\eta_1\!+\!1) \\[6pt]
        \frac{\bar{g}}{s_z^2(-\eta_2)}
        + \frac{c_1}{(-\eta_2)}
        - \frac{b}{\eta_2^{2}}
    \end{pmatrix},
\end{align}
where $\psi''$ is the polygamma function of order two. Every term is an analytic function of $\veta$. Substituting \eqref{eq:fisher-gamma} and \eqref{eq:grad-ell-gamma} into \eqref{eq:stationary-gamma} yields a closed-form fixed-point equation for $\veta^*$, which we solve by natural gradient descent on the Gamma manifold as described in \cref{app:experimental}.

%% file: sections/appendix_xor.tex

This appendix gives explicit parameter values that realize the XOR encoding of \cref{fig:xor-uncertainty}. We use features $\vphi(\vx) = (x_1, x_2, 1)^\top$ with a bias term and two experts ($i \in \{1,2\}$) with depth-2 routing, each producing a log-precision $m_i(\vx)$ that feeds into a likelihood $y \sim \mathcal{N}(\hat{y}_i, e^{-m_i})$. Expert~1 predicts $\hat{y}_1{=}0$ and expert~2 predicts $\hat{y}_2{=}1$.

\paragraph{Router weights.}
Each expert has a routing vector $\vv_i$ and shared sub-expert weight vectors $\vw^L$, $\vw^R$:
\begin{align*}
    \vv_1 &= (14,\; 0,\; {-}7)^\top, & \vv_2 &= ({-}14,\; 0,\; 7)^\top, \\
    \vw^L &= (0,\; 10,\; 0)^\top, & \vw^R &= (0,\; {-}10,\; 10)^\top.
\end{align*}
The softdot precision is $\tau = 2000$ (approximating the sharp-routing limit).

\paragraph{Routing mechanism.}
For expert~$i$ at input $\vx$, the router score is $h_i = \vv_i^\top \vphi(\vx)$. Two switches with clamped weights $\pm 1$ produce opposing activations:
$\kappa_i^R = \exp(h_i)$, $\kappa_i^L = \exp(-h_i)$.
The blended log-precision is
\[
    m_i(\vx) \approx \frac{\kappa_i^R \cdot (\vw^R)^\top \vphi(\vx) + \kappa_i^L \cdot (\vw^L)^\top \vphi(\vx)}{\kappa_i^R + \kappa_i^L},
\]
which selects the sub-expert from the active branch.

\paragraph{Trace through all four inputs.}
\Cref{tab:xor-trace} shows the routing scores and resulting log-precisions. Expert~1 achieves high $\gamma_1 = e^{m_1}$ (and thus high influence for its prediction $\hat{y}_1{=}0$) precisely at the XOR${=}0$ inputs; expert~2 achieves high $\gamma_2 = e^{m_2}$ (high influence for $\hat{y}_2{=}1$) at the XOR${=}1$ inputs.

\begin{table}[h]
\centering
\caption{XOR encoding trace. For each input, the active branch selects the sub-expert whose output $z$ yields a large log-precision $m$, giving that expert high influence $\gamma = e^m$.}
\label{tab:xor-trace}
\small
\begin{tabular}{@{}cccccccc@{}}
\toprule
$x_1$ & $x_2$ & XOR & $h_1$ & $m_1$ & $h_2$ & $m_2$ & Dominant \\
\midrule
0 & 0 & 0 & $-7$ & ${\approx}\,10$ & $+7$ & ${\approx}\,0$ & Exp.\ 1 ($\hat{y}{=}0$) \\
0 & 1 & 1 & $-7$ & ${\approx}\,0$ & $+7$ & ${\approx}\,10$ & Exp.\ 2 ($\hat{y}{=}1$) \\
1 & 0 & 1 & $+7$ & ${\approx}\,0$ & $-7$ & ${\approx}\,10$ & Exp.\ 2 ($\hat{y}{=}1$) \\
1 & 1 & 0 & $+7$ & ${\approx}\,10$ & $-7$ & ${\approx}\,0$ & Exp.\ 1 ($\hat{y}{=}0$) \\
\bottomrule
\end{tabular}
\end{table}

%% file: sections/appendix_neural.tex

The universal approximation result of \cref{cor:universal} parallels classical theorems for neural networks. \citet{cybenko_approximation_1989} proved that single-hidden-layer networks with sigmoidal activation functions are universal approximators. \citet{hornik_approximation_1991} extended this to arbitrary bounded nonconstant activation functions, showing that the multilayer feedforward architecture itself—not the specific choice of activation—is the source of universality (Theorems~1 and~2 therein).

The structural analogy with our framework is direct. In the sharp-routing limit ($\tau \to \infty$), the split-branch construction of \cref{sec:depth2} partitions the input space into regions separated by axis-aligned hyperplanes, selecting a different expert prediction in each region. This is the same piecewise-constant mechanism that underlies the neural proofs: a sufficiently fine partition of the input space, combined with a constant approximation in each cell, can approximate any continuous function on a compact domain to arbitrary accuracy.

The key difference lies in the inference mechanism:
\begin{itemize}
    \item Neural networks are trained via gradient descent on a loss function, producing point estimates of all parameters.
    \item Our framework optimizes the Bethe free energy with closed-form message passing, producing marginal posterior distributions over all parameters and latent variables.
\end{itemize}
As a consequence, epistemic uncertainty is available at every level of the hierarchy without additional cost—no ensembles-of-ensembles, no MC dropout, no Laplace approximation. \Cref{fig:xor-uncertainty} illustrates this concretely: the posterior standard deviation reveals where the model is uncertain about its routing decisions, information that is absent from any point-estimate approach.

%% file: sections/appendix_graphppl.tex

The models from \cref{sec:grammar} can be written directly in GraphPPL.jl \citep{nuijten_graphppljl_2024}, a probabilistic programming language that compiles model specifications into Forney-style factor graphs for inference in RxInfer.jl \citep{bagaev_reactive_2023}.

In the \texttt{@model} block, each tilde statement creates one factor node. The mapping to the alphabet \eqref{eq:alphabet} is as follows:
\begin{center}
\begin{tabular}{@{}ll@{}}
\toprule
\textsf{GraphPPL} & \textsf{Alphabet letter} \\
\midrule
\texttt{NormalMeanPrecision} & $f_{\mathcal{N}}$ \eqref{eq:normal-factor} \\
\texttt{GammaShapeRate} & $f_{\gammad}$ \eqref{eq:gamma-factor} \\
\texttt{softdot} & $f_*$ \eqref{eq:softdot} \\
\texttt{Log} & $f_{\exp}$ \eqref{eq:exp-link} \\
\bottomrule
\end{tabular}
\end{center}
Equality nodes $f_{=}$ \eqref{eq:equality} are inserted automatically whenever a variable appears in more than two factors.

\subsection{Depth 0: static ensemble}\label{app:graphppl-depth0}

The following listing implements the static ensemble model introduced in \cref{sec:depth0}. The outer loop over $i$ corresponds to the $\cdots$ in \cref{fig:depth0} (repetition over experts), and the inner loop over $j$ corresponds to the $\vdots$ (repetition over observations).

\begin{minted}{julia}
@model function static_ensemble(
    n_forecasters, X, y, priors
)
  for i = 1:n_forecasters
    gamma[i] ~ priors[:gamma][i]
  end
  for i = 1:n_forecasters
    for j = 1:length(y)
      y[j] ~ NormalMeanPrecision(
        X[i,j], gamma[i]
      )
    end
  end
end
\end{minted}

\subsection{Depth 1: Precision-Gated Experts}\label{app:graphppl-depth1}

The following listing implements the Precision-Gated Experts model introduced in \cref{sec:depth1}. The loops over $i$ and $j$ carry the same meaning as in \cref{app:graphppl-depth0} ($\cdots$ and $\vdots$ in \cref{fig:depth1}).

\begin{minted}{julia}
@model function pge(n_forecasters,
    n_obs, features, predictions,
    y, priors
)
  for i = 1:n_forecasters
    w[i]    ~ priors[:w][i]
    tau[i]  ~ priors[:tau][i]
    beta[i] ~ priors[:beta][i]
  end
  for j = 1:n_obs
    for i = 1:n_forecasters
      z[i,j] ~ softdot(
        features[j], w[i], tau[i]
      )
      gamma[i,j] ~ GammaShapeRate(
        1.0, beta[i]
      )
      z[i,j] ~ Log(gamma[i,j])
      y[j] ~ NormalMeanPrecision(
        predictions[i,j], gamma[i,j]
      )
    end
  end
end
\end{minted}

The constraints and initialization blocks declare the factorization and form constraints from \cref{sec:declarative}:

\begin{minted}{julia}
@constraints function pge_constraints()
  q(w, z, gamma, tau, beta) =
    q(w)q(z, gamma)q(tau)q(beta)
  q(z)     :: ProjectedTo(
    NormalMeanVariance)
  q(gamma) :: ProjectedTo(Gamma)
end

@initialization function pge_init(priors)
  q(w)     = priors[:w]
  q(z)     = NormalMeanVariance(0.0, 1.0)
  q(gamma) = GammaShapeScale(1.0, 1.0)
  q(tau)   = priors[:tau]
  q(beta)  = priors[:beta]
end
\end{minted}

Given these three blocks, RxInfer.jl constructs the factor graph, derives the message passing schedule, and runs inference automatically — no update equations are written by the user.

\subsection{Noisy Experts}\label{app:graphppl-noisy}

The noisy experts model (\cref{fig:noisy-model}) augments the Precision-Gated Experts (\cref{app:graphppl-depth1}) with an additional noise layer for each expert. Instead of treating each forecaster prediction $\hat{y}_{i,j}$ as an exact input, the model introduces a latent variable $\mathrm{pred}_{i,j}$ connected to the observed prediction through a Normal factor with learned precision $\kappa_i$.

\paragraph{Training model.} During training, $y$ is a data variable (observed). The factorization between $q(y, \mathrm{pred})$ is not required: since $y$ is observed, the structured dependence between $y$ and $\mathrm{pred}$ is absorbed by the data and does not affect the variational updates.

\begin{minted}{julia}
@model function noisy_experts(
    n_forecasters, n_obs,
    features, predictions, y, priors
)
  local w, z, gamma, tau, beta, kappa, pred
  for i = 1:n_forecasters
    w[i]     ~ priors[:w][i]
    tau[i]   ~ priors[:tau][i]
    beta[i]  ~ priors[:beta][i]
    kappa[i] ~ priors[:kappa][i]
  end
  for j = 1:n_obs
    for i = 1:n_forecasters
      z[i,j] ~ softdot(
        features[j], w[i], tau[i]
      ) where {meta = LowRankMeta()}
      gamma[i,j] ~ GammaShapeRate(
        1.0, beta[i]
      )
      z[i,j] ~ Log(gamma[i,j])
      pred[i,j] ~ NormalMeanPrecision(
        predictions[i,j], kappa[i]
      )
      y[j] ~ NormalMeanPrecision(
        pred[i,j], gamma[i,j]
      )
    end
  end
end
\end{minted}

\paragraph{Prediction model.} For prediction, $y$ is no longer observed but is instead a latent variable with an uninformative prior. In this setting, the choice between mean-field and belief propagation on the $\mathrm{pred} \to y$ edge becomes significant. Under the mean-field variational rule the outgoing message towards $y$ from $f_{\mathcal{N}}(\mathrm{pred}, \gamma)$ uses only $\mathbb{E}_{q}[\mathrm{pred}]$:
\begin{minted}{julia}
# Mean-field rule: ignores Var(pred)
@rule NormalMeanPrecision(:out, Marginalisation) (
    q_mu::Any, q_tau::Any
) = NormalMeanPrecision(
    mean(q_mu), mean(q_tau)
)
\end{minted}
This sets the predictive precision to $\mathbb{E}[\gamma_{i,j}]$ alone, discarding the forecaster noise $\kappa_i^{-1}$ entirely. By contrast, the belief propagation rule propagates the full uncertainty in $\mathrm{pred}$:
\begin{minted}{julia}
# BP rule: additive variance 1/kappa + 1/gamma
@rule NormalMeanPrecision(:out, Marginalisation) (
    m_mu::NormalDistributionsFamily,
    m_tau::PointMass
) = begin
    m_mu_mean, m_mu_cov = mean_cov(m_mu)
    NormalMeanPrecision(m_mu_mean,
      inv(m_mu_cov + inv(mean(m_tau))))
end
\end{minted}
Here $\mathrm{Var}(\mathrm{pred}_{i,j}) = \kappa_i^{-1}$, so the predictive variance becomes $\kappa_i^{-1} + \gamma_{i,j}^{-1}$. The prediction model activates this rule by keeping $q(y, \mathrm{pred})$ in a single structured factor and marking $y$ as uninformative:

\begin{minted}{julia}
@model function noisy_experts_prediction(
    n_forecasters, n_obs,
    features, predictions, priors
)
  local w, z, gamma, tau, beta, kappa, pred, y
  for i = 1:n_forecasters
    w[i]     ~ priors[:w][i]
    tau[i]   ~ priors[:tau][i]
    beta[i]  ~ priors[:beta][i]
    kappa[i] ~ priors[:kappa][i]
  end
  for j = 1:n_obs
    for i = 1:n_forecasters
      z[i,j] ~ softdot(
        features[j], w[i], tau[i]
      ) where {meta = LowRankMeta()}
      gamma[i,j] ~ GammaShapeRate(
        1.0, beta[i]
      )
      z[i,j] ~ Log(gamma[i,j])
      pred[i,j] ~ NormalMeanPrecision(
        predictions[i,j], kappa[i]
      )
      y[j] ~ NormalMeanPrecision(
        pred[i,j], gamma[i,j]
      )
    end
    y[j] ~ Uninformative()
  end
end
\end{minted}

The constraints and initialization blocks for both models:

\begin{minted}{julia}
@constraints function noisy_constraints(
    priors, prediction
)
  if prediction
    q(w, z, gamma, tau, beta,
      kappa, pred, y) =
      q(w)q(z, gamma)q(tau)
      q(beta)q(kappa)q(y, pred)
  else
    q(w, z, gamma, tau, beta,
      kappa, pred) =
      q(w)q(z, gamma)q(tau)
      q(beta)q(kappa)q(pred)
  end
  q(z) :: ProjectedTo(
    NormalMeanVariance)
  q(gamma) :: ProjectedTo(Gamma)
end

@initialization function noisy_init(priors)
  q(w)     = deepcopy(priors[:w])
  q(z)     = NormalMeanVariance(0.0, 1.0)
  q(gamma) = GammaShapeScale(1.0, 1.0)
  q(tau)   = priors[:tau]
  q(beta)  = priors[:beta]
  q(kappa) = priors[:kappa]
  q(pred)  = NormalMeanVariance(0.0, 1.0)
end
\end{minted}

\input{figures/noisy_graph.tex}

%% file: figures/noisy_graph.tex
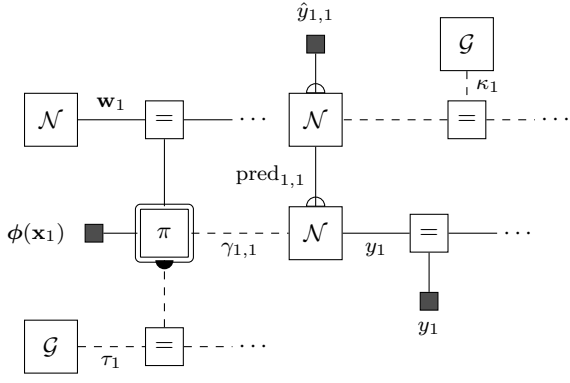
\begin{figure}[t]
\centering
\begin{tikzpicture}[every node/.style={font=\small},
    pibox/.style={draw, double, double distance=1.5pt, minimum size=7mm, rounded corners=1pt}]


    \node[box] at (0, 1.5) (pw) {$\mathcal{N}$};
    \node[smallbox] at (1.5, 1.5) (eq_w) {$=$};
    \draw[-] (pw) -- (eq_w) node[midway, above, font=\scriptsize] {$\vw_1$};
    \draw[-] (eq_w) -- (2.4, 1.5);
    \node at (2.7, 1.5) {$\cdots$};

    \node[box] at (0, -1.5) (ptau) {$\gammad$};
    \node[smallbox] at (1.5, -1.5) (eq_t) {$=$};
    \draw[dashed] (ptau) -- (eq_t) node[midway, below, font=\scriptsize] {$\tau_1$};
    \draw[dashed] (eq_t) -- (2.4, -1.5);
    \node at (2.7, -1.5) {$\cdots$};

    \node[pibox] at (1.5, 0) (pi) {$\pi$};
    \fill (pi.south) ++(-0.12, 0) arc (180:360:0.12) -- cycle;

    \draw[-] (eq_w) -- (pi);
    \draw[dashed] (eq_t) -- (pi);

    \node[left=4.5mm of pi, clamped] (phi) {};
    \node[left=1.3mm of phi, font=\scriptsize] {$\vphi(\vx_1)$};
    \draw[-] (phi) -- (pi);


    \node[box] at (3.5, 0) (n) {$\mathcal{N}$};
    \draw[fill=white] (n.north) ++(0.12, 0) arc (0:180:0.12) -- cycle;
    \draw[dashed] (pi) -- (n) node[midway, below, font=\scriptsize] {$\gamma_{1,1}$};


    \node[box] at (3.5, 1.5) (nnoise) {$\mathcal{N}$};
    \draw[fill=white] (nnoise.north) ++(0.12, 0) arc (0:180:0.12) -- cycle;
    \draw[-] (n) -- (nnoise) node[midway, left, font=\scriptsize] {$\mathrm{pred}_{1,1}$};

    \node[clamped] at (3.5, 2.5) (yh) {};
    \node[above=0.3mm of yh, font=\scriptsize] {$\hat{y}_{1,1}$};
    \draw[-] (yh) -- (nnoise);


    \node[box] at (5.5, 2.5) (pkappa) {$\gammad$};
    \node[smallbox] at (5.5, 1.5) (eq_k) {$=$};
    \draw[dashed] (pkappa) -- (eq_k) node[midway, right, font=\scriptsize] {$\kappa_1$};
    \draw[dashed] (eq_k) -- (nnoise);
    \draw[dashed] (eq_k) -- (6.4, 1.5);
    \node at (6.7, 1.5) {$\cdots$};


    \node[smallbox] at (5, 0) (eqy) {$=$};
    \draw[-] (n) -- (eqy) node[midway, below, font=\scriptsize] {$y_1$};
    \node[clamped] at (5, -0.9) (obs) {};
    \node[below=0.3mm of obs, font=\scriptsize] {$y_1$};
    \draw[-] (obs) -- (eqy);
    \draw[-] (eqy) -- (5.9, 0);
    \node at (6.2, 0) {$\cdots$};

\end{tikzpicture}
\caption{Noisy experts factor graph, shown for expert $i{=}1$ and observation $j{=}1$. Compared to Depth~1 (\cref{fig:depth1}), the clamped prediction $\hat{y}_{1,1}$ is replaced by a latent variable $\mathrm{pred}_{1,1}$, connected to the observed forecaster output through an additional Normal factor with learned precision $\kappa_1$. This separates forecaster disagreement from irreducible predictive uncertainty.}
\label{fig:noisy-model}
\end{figure}

%% file: sections/appendix_experimental.tex

\subsection{Base forecasters}

We train $n = 7$ forecasters independently using a fixed input sequence length of 96 time steps.
\begin{itemize}
    \item \textbf{DLinear} \citep{zeng_are_2023}: decomposes the input series into trend and seasonal components, then applies separate linear layers to forecast each part. Channel-independent.
    \item \textbf{NLinear} \citep{zeng_are_2023}: applies a linear forecasting layer after normalizing the input sequence. Channel-independent.
    \item \textbf{LSTM} \citep{hochreiter_long_1997}: recurrent neural network for temporal dependencies.
    \item \textbf{CNN}: cross-channel one-dimensional convolutional model applied across time-series channels.
    \item \textbf{NConv}: similar to NLinear, but replaces the linear layer with a one-dimensional convolution.
    \item \textbf{Quantile 10 / Quantile 90}: constant forecasters corresponding to the 10th and 90th quantiles.
\end{itemize}

All neural networks are trained with the AdamW optimizer \citep{loshchilov_decoupled_2019}, learning rate $0.001$, for $50$ epochs; the best checkpoint is selected by validation performance. For ETTh1 and ETTh2, forecasting is performed on a univariate target; for Exchange Rate, Electricity, and Traffic, all dimensions are predicted.

\subsection{Feature extraction}

To obtain a compact representation of multivariate time series with input length 96, we train a variational autoencoder (VAE) \citep{kingma_autoencoding_2013} that reconstructs the input sequence through a latent representation of dimension 64.

\subsection{Baseline gating methods}\label{app:baseline_nll}

For comparison we consider two gating variants based on the classical Mixture-of-Experts architecture \citep{jacobs_adaptive_1991}: \textbf{MoE}, a single-layer neural network with softmax output, and \textbf{MoE Big}, a two-layer fully connected neural network with softmax output. Both are trained with early stopping: training halts when the loss decrease falls below $\delta = 10^{-5}$, with 100 epochs and patience 50.

\paragraph{NLL computation for the softmax baselines.}
The softmax gating weights $w_i = e^{z_i}/\sum_j e^{z_j}$ are positive and sum to one, so the predictive mean is $\hat{y} = \sum_i w_i \hat{y}_i$. To obtain a predictive variance we interpret the unnormalized scores $e^{z_i}$ as precisions: setting $\gamma_i = e^{z_i}$ recovers the same weighted average $\hat{y} = \sum_i \gamma_i \hat{y}_i / \sum_i \gamma_i$ and yields a predictive precision $\sum_j e^{z_j}$, giving a Gaussian predictive distribution $\gauss(\hat{y},\, (\sum_j e^{z_j})^{-1})$ from which NLL is computed.

This interpretation is inherently limited: the softmax is shift-invariant, so replacing $z_i \to z_i + c$ leaves the weights (and hence the predictive mean) unchanged but scales the implied precision by $e^{c}$. No additional term in the training objective can resolve this ambiguity, since the gradient of the softmax output with respect to a global shift is zero. Consequently, the NLL values reported for the MoE baselines depend on an unidentified degree of freedom. In PGE, the precisions are explicit random variables equipped with priors, so no such ambiguity arises and the predictive uncertainty is fully determined by the model.

\subsection{Our framework}

All configurations of our framework are trained for 5 variational message passing iterations and use 3 iterations during prediction.

\paragraph{Solving the non-conjugate fixed point.}
The stationarity condition \eqref{eq:stationary-z-generic} is a nonlinear fixed-point equation whose right-hand side is available in closed form but whose solution is not. We solve it by casting it as the minimization of the local free energy $F_z[\veta]$ \eqref{eq:local-bfe-z} over the natural-parameter space of the Gaussian, which carries the Fisher information metric and forms a Riemannian manifold. Optimization on this manifold is performed with Manopt.jl \citep{bergmann_manoptjl_2022}, using the exponential family manifold interface of \citet{lukashchuk_exponentialfamilymanifoldsjl_2025}. We use an Armijo line search with initial step size $10^{-2}$ and a bounded-norm update rule (maximum step norm $0.5$), running up to 50 iterations with a tolerance of $10^{-6}$.

%% file: sections/full_results.tex
\subsection{Dataset info}

The statistical characteristics of the datasets include the dataset names, the number of input dimensions, corresponding to the number of individual time series in each dataset, the length of each time series, the training, validation, and test splits adopted for expert and ensemble training, and the sampling frequency of the temporal index.

\begin{table}[t]
\centering
\caption{Characteristics of the benchmark datasets, including dimensionality, length, split ratio (train:val:test), and sampling frequency.}
\label{tab:datasets}
\resizebox{\columnwidth}{!}{%
\begin{tabular}{l r r c l}
\hline
Dataset & Dims & Length & Split ratio & Frequency \\
\hline
ETTh1       & 7   & 14307 & 6:2:2 & 15 min \\
ETTh2       & 7   & 14307 & 6:2:2 & 15 min \\
Electricity & 321 & 26211 & 7:1:2 & Hourly \\
Traffic     & 862 & 17451 & 7:1:2 & Hourly \\
Exchange    & 8   & 7207  & 7:1:2 & Daily \\
\hline
\end{tabular}%
}
\end{table}

\subsection{Methods parameter count }

Table \ref{tab:methods_params} reports the trainable parameter counts for all considered ensemble methods in the VAE-based feature setting with input dimensionality $d=65$ and number of experts $K=7$. Here, $d=65$ consists of a 64-dimensional latent representation produced by the VAE together with one additional constant feature. The table presents both the general formulas and the corresponding numerical values for the experimental setup used in this work.

As expected, the Static model has the smallest number of parameters, since it only learns global coefficients for each expert. The MoE model is also relatively compact due to its single-layer softmax gating mechanism, whereas MoE Big is substantially larger because it employs a deeper gating network. The Dynamic and Noisy variants require the largest number of parameters because they model full covariance structures, whose complexity grows quadratically with $d$. By contrast, the diagonal variants remain much more parameter-efficient, since they restrict uncertainty modeling to diagonal terms only.

These parameter counts are later used to analyze the trade-off between forecasting performance and model complexity.

  \begin{table}[t]
  \centering
  \caption{Parameter-count formulas and counts for the VAE-feature setting ($d=65$, $K=7$).}
  \label{tab:methods_params}
  \setlength{\tabcolsep}{1.5pt}
  \begin{tabular}{lcc}
  \toprule
  Method & Formula for $N_{\text{total}}$ & Count ($d=65$) \\
  \midrule
  Static & $2K$ & $14$ \\
  MoE & $K(d+1)$ & $462$ \\
  MoE Big & $H(d+1)+K(H+1)$ & $14023$ \\
  Dyn.\ Diag. & $2K(d+2)$ & $938$ \\
  Noisy Diag. & $K(2d+6)$ & $952$ \\
  Dynamic & $K\!\left(d+\frac{d(d+1)}{2}+4\right)$ & $15498$ \\
  Noisy & $K\!\left(d+\frac{d(d+1)}{2}+6\right)$ & $15512$ \\
  \bottomrule
  \end{tabular}
  \end{table}

\subsection{Pareto Frontier}

The Pareto frontier was constructed for the Static, Dynamic, Dynamic Diagonal, Noisy, and Noisy Diagonal models. It compares the radar-plot area percentage shown in Figures \ref{fig:radar_mse_all} and \ref{fig:radar_nll_all} against the number of trainable parameters used by each model, with parameter counts reported in Table \ref{tab:methods_params}.

\begin{figure*}[t]
    \centering
    \includegraphics[width=\linewidth]{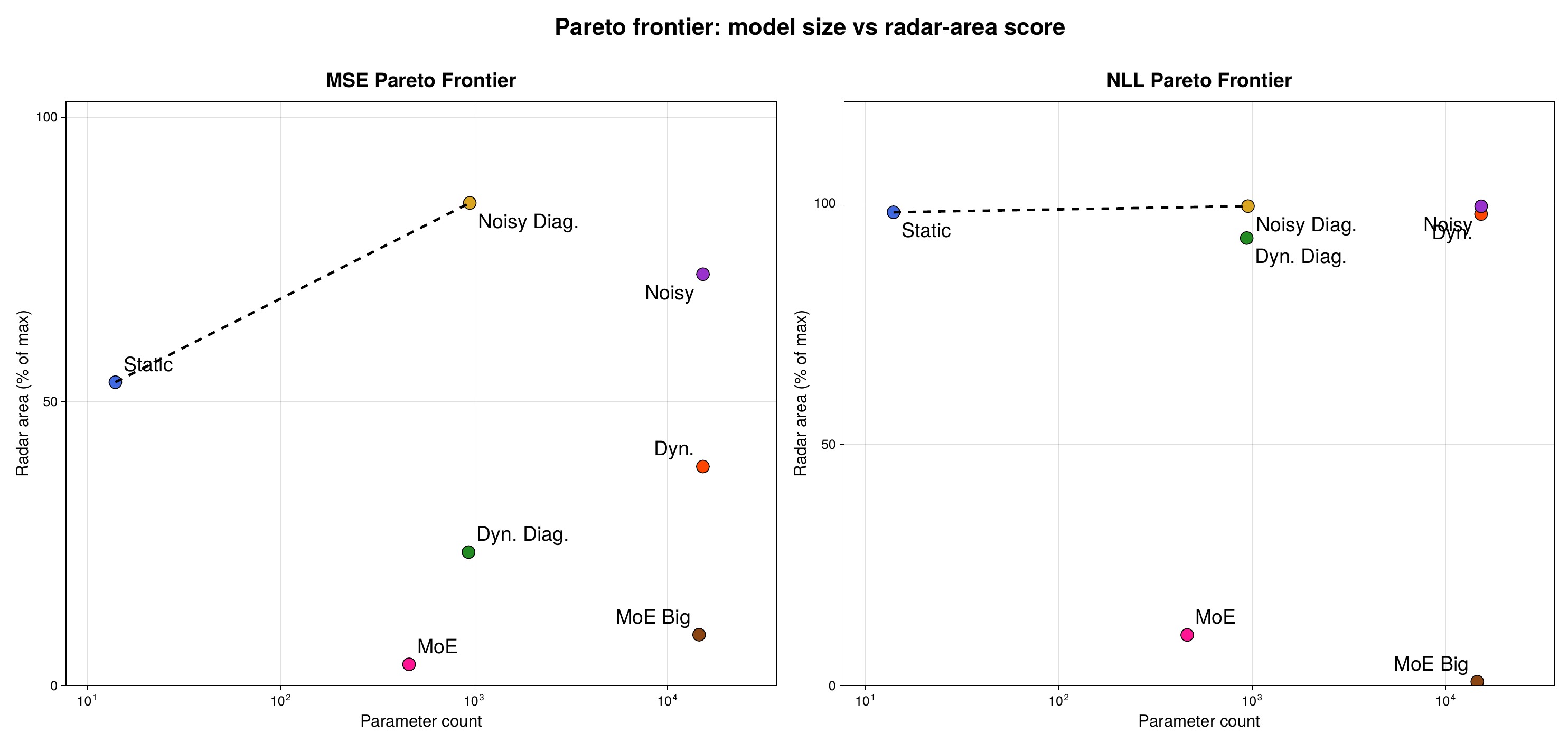}
    \caption{Pareto frontier relating model size to radar-chart area. Static and Noisy Diagonal offer the strongest trade-off between complexity and predictive performance.}
    \label{fig:pareto_frontier}
\end{figure*}

\subsection{Experts Performance}

This table presents the MAE and MSE results for each expert model on all datasets described in Table \ref{tab:datasets}, evaluated over the forecasting horizons ${96,192,336,720}$. All experts were trained with an input sequence length of 96. The considered expert models are CNN, DLinear, LSTM, NLinear, NConv, as well as the quantile predictors q10 and q90.

\begin{table*}[t]
\centering
\scriptsize
\setlength{\tabcolsep}{2.0pt}
\renewcommand{\arraystretch}{1.05}
\caption{Comparison of MSE and MAE for baseline models and quantile-based predictors with input sequence length $96$.}
\label{tab:baselines_seq96}
\resizebox{\textwidth}{!}{%
\begin{tabular}{llcccccccccccccc}
\toprule
\multirow{2}{*}{Dataset} & \multirow{2}{*}{H} &
\multicolumn{2}{c}{CNN} &
\multicolumn{2}{c}{DLinear} &
\multicolumn{2}{c}{LSTM} &
\multicolumn{2}{c}{NLinear} &
\multicolumn{2}{c}{NConv} &
\multicolumn{2}{c}{q10} &
\multicolumn{2}{c}{q90} \\
\cmidrule(lr){3-4} \cmidrule(lr){5-6} \cmidrule(lr){7-8} \cmidrule(lr){9-10}
\cmidrule(lr){11-12} \cmidrule(lr){13-14} \cmidrule(lr){15-16}
& & MSE & MAE & MSE & MAE & MSE & MAE & MSE & MAE & MSE & MAE & MSE & MAE & MSE & MAE \\
\midrule
\multirow{5}{*}{Exchange\_rate}
& 96  & 2.455 & 1.187 & 0.157 & 0.306 & 1.150 & 0.881 & 0.183 & 0.317 & 0.180 & 0.314 & 1.439 & 0.958 & 3.423 & 1.411 \\
& 192 & 3.559 & 1.448 & 0.325 & 0.446 & 2.589 & 1.337 & 0.400 & 0.478 & 0.410 & 0.479 & 1.379 & 0.962 & 3.584 & 1.447 \\
& 336 & 2.295 & 1.184 & 0.558 & 0.606 & 1.536 & 1.023 & 0.804 & 0.680 & 0.820 & 0.680 & 1.407 & 0.974 & 3.751 & 1.491 \\
& 720 & 4.902 & 1.706 & 1.495 & 0.984 & 1.982 & 1.147 & 2.187 & 1.166 & 2.426 & 1.225 & 2.269 & 1.262 & 2.313 & 1.125 \\
& Avg & 3.303 & 1.381 & 0.634 & 0.585 & 1.814 & 1.097 & 0.894 & 0.660 & 0.959 & 0.674 & 1.623 & 1.039 & 3.268 & 1.368 \\
\midrule
\multirow{5}{*}{ETTh1}
& 96  & 0.456 & 0.586 & 0.156 & 0.312 & 0.412 & 0.536 & 0.147 & 0.301 & 0.181 & 0.333 & 0.611 & 0.680 & 0.601 & 0.675 \\
& 192 & 0.248 & 0.398 & 0.162 & 0.318 & 0.278 & 0.419 & 0.137 & 0.295 & 0.178 & 0.323 & 0.538 & 0.630 & 0.647 & 0.708 \\
& 336 & 0.298 & 0.449 & 0.137 & 0.309 & 0.231 & 0.377 & 0.135 & 0.291 & 0.174 & 0.328 & 0.501 & 0.605 & 0.708 & 0.750 \\
& 720 & 1.006 & 0.905 & 0.290 & 0.462 & 0.275 & 0.409 & 0.216 & 0.361 & 0.310 & 0.440 & 0.465 & 0.583 & 0.799 & 0.809 \\
& Avg & 0.502 & 0.585 & 0.187 & 0.350 & 0.299 & 0.435 & 0.159 & 0.312 & 0.211 & 0.356 & 0.529 & 0.624 & 0.689 & 0.736 \\
\midrule
\multirow{5}{*}{ETTh2}
& 96  & 1.796 & 1.114 & 0.313 & 0.448 & 0.851 & 0.752 & 0.339 & 0.465 & 0.427 & 0.505 & 3.160 & 1.561 & 1.086 & 0.825 \\
& 192 & 1.017 & 0.772 & 0.270 & 0.414 & 0.855 & 0.751 & 0.320 & 0.455 & 0.414 & 0.504 & 2.943 & 1.496 & 1.081 & 0.823 \\
& 336 & 1.109 & 0.853 & 0.269 & 0.413 & 0.848 & 0.715 & 0.345 & 0.470 & 0.423 & 0.508 & 2.787 & 1.449 & 1.081 & 0.824 \\
& 720 & 1.957 & 1.228 & 0.385 & 0.516 & 1.232 & 0.966 & 0.681 & 0.661 & 0.685 & 0.655 & 2.683 & 1.420 & 1.198 & 0.875 \\
& Avg & 1.470 & 0.992 & 0.309 & 0.448 & 0.947 & 0.796 & 0.421 & 0.513 & 0.487 & 0.543 & 2.893 & 1.482 & 1.111 & 0.837 \\
\midrule
\multirow{5}{*}{Electricity}
& 96  & 0.350 & 0.432 & 0.176 & 0.267 & 0.396 & 0.443 & 0.178 & 0.266 & 0.348 & 0.363 & 2.778 & 1.335 & 1.838 & 1.095 \\
& 192 & 0.388 & 0.462 & 0.251 & 0.326 & 0.335 & 0.403 & 0.260 & 0.324 & 0.307 & 0.345 & 2.786 & 1.337 & 1.872 & 1.105 \\
& 336 & 0.364 & 0.434 & 0.223 & 0.314 & 0.337 & 0.407 & 0.229 & 0.310 & 0.248 & 0.319 & 2.800 & 1.340 & 1.922 & 1.119 \\
& 720 & 0.418 & 0.462 & 0.344 & 0.404 & 0.425 & 0.462 & 0.370 & 0.407 & 0.462 & 0.449 & 2.832 & 1.347 & 2.049 & 1.155 \\
& Avg & 0.380 & 0.448 & 0.249 & 0.336 & 0.373 & 0.429 & 0.259 & 0.341 & 0.341 & 0.369 & 2.799 & 1.340 & 1.920 & 1.119 \\
\midrule
\multirow{5}{*}{Traffic}
& 96  & 0.743 & 0.422 & 0.475 & 0.308 & 0.845 & 0.469 & 0.478 & 0.305 & 1.205 & 0.596 & 2.738 & 1.139 & 3.308 & 1.442 \\
& 192 & 0.737 & 0.423 & 0.710 & 0.431 & 0.817 & 0.443 & 0.716 & 0.430 & 0.831 & 0.453 & 2.746 & 1.140 & 3.161 & 1.425 \\
& 336 & 0.730 & 0.403 & 0.547 & 0.345 & 0.740 & 0.407 & 0.556 & 0.337 & 0.622 & 0.344 & 2.752 & 1.141 & 3.109 & 1.437 \\
& 720 & 0.800 & 0.447 & 0.814 & 0.473 & 1.055 & 0.576 & 0.820 & 0.466 & 1.191 & 0.588 & 2.772 & 1.144 & 3.073 & 1.426 \\
& Avg & 0.753 & 0.424 & 0.637 & 0.389 & 0.864 & 0.474 & 0.643 & 0.385 & 0.962 & 0.496 & 2.752 & 1.141 & 3.163 & 1.433 \\
\bottomrule
\end{tabular}%
}
\end{table*}

\subsection{Performance of ensembles}

The MSE and NLL values are presented for the Static, Dynamic, Dynamic Diagonal, Noisy, Noisy Diagonal, MoE, and MoE Big ensemble methods. The MoE model is implemented as a single-layer neural network with softmax-based weighting, whereas MoE Big is a two-layer softmax-based neural network with a higher number of trainable parameters, as detailed in Table \ref{tab:methods_params}. All ensemble approaches were evaluated on the same five datasets listed in Table \ref{tab:datasets}. In all cases, the ensemble input consisted of sequences of length 96, which were compressed into a 64-dimensional latent state using a variational autoencoder trained to reconstruct all dataset features over the same input length.

\begin{table*}[t]
\centering
\tiny
\setlength{\tabcolsep}{2.5pt}
\caption{MSE / NLL for Static, Dynamic (Dyn.), Dynamic Diagonal (Dyn. Diag.), Noisy Experts (Noisy), Noisy Diagonal (Noisy Diag.), Mixture of Experts (MoE), and MoE Big ensembles, compared against the best baseline model selected by lowest baseline MSE for each dataset and horizon. Rows are organized by dataset, horizon, and metric; models remain as columns. An average block is included for each dataset. \textbf{Bold} indicates the best result; {\underline{\textcolor{blue}{blue underlined}}} indicates the second best. Non-finite values are treated as missing. The \(\pm\) values denote 95\% confidence intervals for each metric.}
\label{tab:ensemble_comparison}
\begin{tabular}{lll c c c c c c c c}
\toprule
Dataset & H & Metric & Static & Dyn. & Dyn. Diag. & Noisy & Noisy Diag. & MoE & MoE Big & Best \\
\midrule
\multirow{10}{*}{{\scriptsize ETTh1}} & \multirow{2}{*}{96} & MSE & 0.132 \(\pm\) 0.01 & \underline{\textcolor{blue}{0.128 \(\pm\) 0.01}} & \textbf{0.123 \(\pm\) 0.01} & 0.131 \(\pm\) 0.01 & 0.129 \(\pm\) 0.01 & 0.147 \(\pm\) 0.01 & 0.156 \(\pm\) 0.01 & 0.147 \\
 &  & NLL & 1.138 \(\pm\) 0.10 & \textbf{0.412 \(\pm\) 0.02} & 0.451 \(\pm\) 0.01 & \underline{\textcolor{blue}{0.447 \(\pm\) 0.02}} & 0.468 \(\pm\) 0.01 & $(2.1 \pm 1.7)\!\times\!10^{25}$ & $(2.3 \pm 1.3)\!\times\!10^{24}$ & $\times$ \\
\cmidrule(lr){2-11}
 & \multirow{2}{*}{192} & MSE & \underline{\textcolor{blue}{0.114 \(\pm\) 0.01}} & 0.115 \(\pm\) 0.01 & \textbf{0.112 \(\pm\) 0.01} & 0.116 \(\pm\) 0.01 & 0.115 \(\pm\) 0.01 & 0.137 \(\pm\) 0.01 & 0.137 \(\pm\) 0.01 & 0.137 \\
 &  & NLL & 0.759 \(\pm\) 0.08 & \textbf{0.370 \(\pm\) 0.02} & 0.425 \(\pm\) 0.01 & \underline{\textcolor{blue}{0.408 \(\pm\) 0.02}} & 0.433 \(\pm\) 0.01 & $(1.1 \pm 1.1)\!\times\!10^{22}$ & $(1.2 \pm 0.9)\!\times\!10^{19}$ & $\times$ \\
\cmidrule(lr){2-11}
 & \multirow{2}{*}{336} & MSE & 0.105 \(\pm\) 0.01 & \underline{\textcolor{blue}{0.096 \(\pm\) 0.00}} & \textbf{0.095 \(\pm\) 0.00} & 0.098 \(\pm\) 0.01 & 0.097 \(\pm\) 0.00 & 0.172 \(\pm\) 0.01 & 0.231 \(\pm\) 0.01 & 0.135 \\
 &  & NLL & 0.610 \(\pm\) 0.07 & \textbf{0.314 \(\pm\) 0.01} & 0.384 \(\pm\) 0.01 & \underline{\textcolor{blue}{0.362 \(\pm\) 0.01}} & 0.393 \(\pm\) 0.01 & $(1.1 \pm 0.7)\!\times\!10^{22}$ & $(2.5 \pm 1.9)\!\times\!10^{14}$ & $\times$ \\
\cmidrule(lr){2-11}
 & \multirow{2}{*}{720} & MSE & 0.145 \(\pm\) 0.01 & \textbf{0.112 \(\pm\) 0.01} & \underline{\textcolor{blue}{0.113 \(\pm\) 0.01}} & 0.115 \(\pm\) 0.01 & 0.116 \(\pm\) 0.01 & 0.215 \(\pm\) 0.01 & 0.216 \(\pm\) 0.01 & 0.216 \\
 &  & NLL & 0.801 \(\pm\) 0.07 & \textbf{0.376 \(\pm\) 0.02} & 0.439 \(\pm\) 0.01 & \underline{\textcolor{blue}{0.424 \(\pm\) 0.01}} & 0.454 \(\pm\) 0.01 & $(9.1 \pm 6.9)\!\times\!10^{11}$ & $(1.5 \pm 1.1)\!\times\!10^{17}$ & $\times$ \\
\cmidrule(lr){2-11}
 & \multirow{2}{*}{Avg} & MSE & 0.124 \(\pm\) 0.01 & \underline{\textcolor{blue}{0.113 \(\pm\) 0.01}} & \textbf{0.111 \(\pm\) 0.01} & 0.115 \(\pm\) 0.01 & 0.114 \(\pm\) 0.01 & 0.168 \(\pm\) 0.01 & 0.185 \(\pm\) 0.01 & 0.159 \\
 &  & NLL & 0.827 \(\pm\) 0.08 & \textbf{0.368 \(\pm\) 0.02} & 0.424 \(\pm\) 0.01 & \underline{\textcolor{blue}{0.410 \(\pm\) 0.01}} & 0.437 \(\pm\) 0.01 & $(5.1 \pm 4.3)\!\times\!10^{24}$ & $(5.7 \pm 3.3)\!\times\!10^{23}$ & $\times$ \\
\midrule
\multirow{10}{*}{{\scriptsize ETTh2}} & \multirow{2}{*}{96} & MSE & 0.325 \(\pm\) 0.02 & 0.346 \(\pm\) 0.02 & 0.344 \(\pm\) 0.02 & 0.340 \(\pm\) 0.02 & 0.337 \(\pm\) 0.02 & 0.339 \(\pm\) 0.02 & \textbf{0.313 \(\pm\) 0.02} & \underline{\textcolor{blue}{0.313}} \\
 &  & NLL & 2.333 \(\pm\) 0.14 & 0.934 \(\pm\) 0.03 & \underline{\textcolor{blue}{0.885 \(\pm\) 0.02}} & 0.897 \(\pm\) 0.03 & \textbf{0.874 \(\pm\) 0.02} & $(3.1 \pm 2.5)\!\times\!10^{23}$ & $(1.1 \pm 1.1)\!\times\!10^{31}$ & $\times$ \\
\cmidrule(lr){2-11}
 & \multirow{2}{*}{192} & MSE & \underline{\textcolor{blue}{0.296 \(\pm\) 0.02}} & 0.336 \(\pm\) 0.02 & 0.332 \(\pm\) 0.02 & 0.327 \(\pm\) 0.02 & 0.321 \(\pm\) 0.02 & 0.320 \(\pm\) 0.02 & 0.320 \(\pm\) 0.02 & \textbf{0.270} \\
 &  & NLL & 1.736 \(\pm\) 0.12 & 0.924 \(\pm\) 0.04 & \underline{\textcolor{blue}{0.872 \(\pm\) 0.02}} & 0.882 \(\pm\) 0.03 & \textbf{0.856 \(\pm\) 0.02} & $(2.9 \pm 2.7)\!\times\!10^{23}$ & $(7.1 \pm 6.6)\!\times\!10^{24}$ & $\times$ \\
\cmidrule(lr){2-11}
 & \multirow{2}{*}{336} & MSE & 0.304 \(\pm\) 0.02 & 0.353 \(\pm\) 0.02 & 0.350 \(\pm\) 0.02 & 0.337 \(\pm\) 0.02 & 0.333 \(\pm\) 0.02 & \underline{\textcolor{blue}{0.269 \(\pm\) 0.01}} & 0.269 \(\pm\) 0.01 & \textbf{0.269} \\
 &  & NLL & 1.661 \(\pm\) 0.11 & 0.961 \(\pm\) 0.04 & \underline{\textcolor{blue}{0.902 \(\pm\) 0.02}} & 0.904 \(\pm\) 0.03 & \textbf{0.878 \(\pm\) 0.02} & $(5.8 \pm 4.1)\!\times\!10^{19}$ & $(1.4 \pm 1.3)\!\times\!10^{23}$ & $\times$ \\
\cmidrule(lr){2-11}
 & \multirow{2}{*}{720} & MSE & 0.480 \(\pm\) 0.02 & \textbf{0.321 \(\pm\) 0.02} & 0.323 \(\pm\) 0.02 & 0.326 \(\pm\) 0.02 & \underline{\textcolor{blue}{0.322 \(\pm\) 0.02}} & 0.385 \(\pm\) 0.02 & 0.385 \(\pm\) 0.02 & 0.385 \\
 &  & NLL & 1.976 \(\pm\) 0.11 & 0.870 \(\pm\) 0.03 & \underline{\textcolor{blue}{0.857 \(\pm\) 0.02}} & 0.863 \(\pm\) 0.03 & \textbf{0.854 \(\pm\) 0.02} & $(2.1 \pm 1.5)\!\times\!10^{19}$ & $\times$ & $\times$ \\
\cmidrule(lr){2-11}
 & \multirow{2}{*}{Avg} & MSE & 0.351 \(\pm\) 0.02 & 0.339 \(\pm\) 0.02 & 0.337 \(\pm\) 0.02 & 0.333 \(\pm\) 0.02 & 0.328 \(\pm\) 0.02 & 0.328 \(\pm\) 0.02 & \underline{\textcolor{blue}{0.322 \(\pm\) 0.02}} & \textbf{0.309} \\
 &  & NLL & 1.926 \(\pm\) 0.12 & 0.922 \(\pm\) 0.03 & \underline{\textcolor{blue}{0.879 \(\pm\) 0.02}} & 0.886 \(\pm\) 0.03 & \textbf{0.866 \(\pm\) 0.02} & $(1.5 \pm 1.3)\!\times\!10^{23}$ & $(3.8 \pm 3.6)\!\times\!10^{30}$ & $\times$ \\
\midrule
\multirow{10}{*}{{\scriptsize \shortstack{Exchange\\rate}}} & \multirow{2}{*}{96} & MSE & 0.326 \(\pm\) 0.01 & 0.398 \(\pm\) 0.01 & 0.239 \(\pm\) 0.01 & 0.306 \(\pm\) 0.01 & 0.262 \(\pm\) 0.01 & 2.455 \(\pm\) 0.03 & \underline{\textcolor{blue}{0.183 \(\pm\) 0.01}} & \textbf{0.157} \\
 &  & NLL & 18.757 \(\pm\) 0.68 & $(1.7 \pm 0.1)\!\times\!10^{2}$ & $(6.2 \pm 0.5)\!\times\!10^{2}$ & \underline{\textcolor{blue}{10.912 \(\pm\) 0.33}} & \textbf{9.705 \(\pm\) 0.32} & $(2.5 \pm 0.2)\!\times\!10^{19}$ & $(1.5 \pm 0.2)\!\times\!10^{13}$ & $\times$ \\
\cmidrule(lr){2-11}
 & \multirow{2}{*}{192} & MSE & 0.820 \(\pm\) 0.03 & 0.821 \(\pm\) 0.03 & 0.521 \(\pm\) 0.02 & 0.509 \(\pm\) 0.02 & \underline{\textcolor{blue}{0.453 \(\pm\) 0.02}} & 3.559 \(\pm\) 0.11 & 3.559 \(\pm\) 0.11 & \textbf{0.325} \\
 &  & NLL & 40.131 \(\pm\) 1.35 & $(1.7 \pm 0.1)\!\times\!10^{2}$ & $(1.6 \pm 0.1)\!\times\!10^{3}$ & \underline{\textcolor{blue}{17.519 \(\pm\) 0.73}} & \textbf{16.568 \(\pm\) 0.63} & $(2.3 \pm 0.2)\!\times\!10^{14}$ & $\times$ & $\times$ \\
\cmidrule(lr){2-11}
 & \multirow{2}{*}{336} & MSE & 1.161 \(\pm\) 0.03 & 1.420 \(\pm\) 0.03 & 1.071 \(\pm\) 0.02 & 0.854 \(\pm\) 0.02 & \underline{\textcolor{blue}{0.749 \(\pm\) 0.02}} & 2.295 \(\pm\) 0.07 & 2.295 \(\pm\) 0.07 & \textbf{0.558} \\
 &  & NLL & 43.942 \(\pm\) 1.04 & $(2.4 \pm 0.1)\!\times\!10^{2}$ & $(1.9 \pm 0.1)\!\times\!10^{3}$ & \underline{\textcolor{blue}{27.064 \(\pm\) 0.65}} & \textbf{25.050 \(\pm\) 0.56} & $(5.9 \pm 1.3)\!\times\!10^{20}$ & $(1.1 \pm 0.2)\!\times\!10^{21}$ & $\times$ \\
\cmidrule(lr){2-11}
 & \multirow{2}{*}{720} & MSE & 1.675 \(\pm\) 0.04 & 1.754 \(\pm\) 0.03 & \textbf{1.140 \(\pm\) 0.01} & 1.515 \(\pm\) 0.03 & \underline{\textcolor{blue}{1.472 \(\pm\) 0.03}} & 1.982 \(\pm\) 0.05 & 1.982 \(\pm\) 0.05 & 1.495 \\
 &  & NLL & \underline{\textcolor{blue}{36.109 \(\pm\) 0.78}} & $(1.3 \pm 0.0)\!\times\!10^{2}$ & $(8.6 \pm 0.5)\!\times\!10^{2}$ & 36.573 \(\pm\) 0.61 & \textbf{35.816 \(\pm\) 0.72} & $(4.3 \pm 1.1)\!\times\!10^{21}$ & $(2.8 \pm 0.8)\!\times\!10^{32}$ & $\times$ \\
\cmidrule(lr){2-11}
 & \multirow{2}{*}{Avg} & MSE & 0.995 \(\pm\) 0.02 & 1.098 \(\pm\) 0.03 & 0.743 \(\pm\) 0.01 & 0.796 \(\pm\) 0.02 & \underline{\textcolor{blue}{0.734 \(\pm\) 0.02}} & 2.573 \(\pm\) 0.07 & 2.005 \(\pm\) 0.06 & \textbf{0.634} \\
 &  & NLL & 34.735 \(\pm\) 0.96 & $(1.8 \pm 0.1)\!\times\!10^{2}$ & $(1.2 \pm 0.1)\!\times\!10^{3}$ & \underline{\textcolor{blue}{23.017 \(\pm\) 0.58}} & \textbf{21.785 \(\pm\) 0.56} & $(1.2 \pm 0.3)\!\times\!10^{21}$ & $(9.4 \pm 2.8)\!\times\!10^{31}$ & $\times$ \\
\midrule
\multirow{10}{*}{{\scriptsize Electricity}} & \multirow{2}{*}{96} & MSE & 0.182 \(\pm\) 0.01 & 0.197 \(\pm\) 0.01 & 0.198 \(\pm\) 0.01 & 0.188 \(\pm\) 0.01 & 0.188 \(\pm\) 0.01 & 0.176 \(\pm\) 0.01 & \textbf{0.176 \(\pm\) 0.01} & \underline{\textcolor{blue}{0.176}} \\
 &  & NLL & $(5.2 \pm 0.2)\!\times\!10^{2}$ & \textbf{$(3.1 \pm 0.1)\!\times\!10^{2}$} & $(1.4 \pm 0.1)\!\times\!10^{3}$ & \underline{\textcolor{blue}{$(3.3 \pm 0.1)\!\times\!10^{2}$}} & $(3.3 \pm 0.1)\!\times\!10^{2}$ & $(6.5 \pm 5.4)\!\times\!10^{14}$ & $(5.7 \pm 4.3)\!\times\!10^{19}$ & $\times$ \\
\cmidrule(lr){2-11}
 & \multirow{2}{*}{192} & MSE & \textbf{0.225 \(\pm\) 0.01} & 0.232 \(\pm\) 0.01 & 0.248 \(\pm\) 0.01 & \underline{\textcolor{blue}{0.228 \(\pm\) 0.01}} & 0.228 \(\pm\) 0.01 & 0.251 \(\pm\) 0.01 & 0.251 \(\pm\) 0.01 & 0.251 \\
 &  & NLL & $(5.6 \pm 0.2)\!\times\!10^{2}$ & \textbf{$(3.6 \pm 0.1)\!\times\!10^{2}$} & $(1.3 \pm 0.0)\!\times\!10^{3}$ & \underline{\textcolor{blue}{$(4.2 \pm 0.1)\!\times\!10^{2}$}} & $(4.2 \pm 0.1)\!\times\!10^{2}$ & $(1.1 \pm 0.3)\!\times\!10^{13}$ & $(1.4 \pm 1.3)\!\times\!10^{22}$ & $\times$ \\
\cmidrule(lr){2-11}
 & \multirow{2}{*}{336} & MSE & \textbf{0.207 \(\pm\) 0.01} & 0.216 \(\pm\) 0.01 & 0.229 \(\pm\) 0.01 & 0.210 \(\pm\) 0.01 & \underline{\textcolor{blue}{0.209 \(\pm\) 0.01}} & 0.224 \(\pm\) 0.01 & 0.230 \(\pm\) 0.01 & 0.223 \\
 &  & NLL & $(5.8 \pm 0.2)\!\times\!10^{2}$ & \textbf{$(3.6 \pm 0.1)\!\times\!10^{2}$} & $(1.3 \pm 0.0)\!\times\!10^{3}$ & \underline{\textcolor{blue}{$(4.0 \pm 0.2)\!\times\!10^{2}$}} & $(4.0 \pm 0.2)\!\times\!10^{2}$ & $(4.3 \pm 1.5)\!\times\!10^{10}$ & $(4.7 \pm 1.4)\!\times\!10^{12}$ & $\times$ \\
\cmidrule(lr){2-11}
 & \multirow{2}{*}{720} & MSE & \textbf{0.308 \(\pm\) 0.01} & 0.321 \(\pm\) 0.01 & 0.320 \(\pm\) 0.01 & \underline{\textcolor{blue}{0.311 \(\pm\) 0.01}} & 0.311 \(\pm\) 0.01 & 0.340 \(\pm\) 0.01 & 0.370 \(\pm\) 0.01 & 0.344 \\
 &  & NLL & $(7.0 \pm 0.2)\!\times\!10^{2}$ & \textbf{$(4.6 \pm 0.1)\!\times\!10^{2}$} & $(1.9 \pm 0.0)\!\times\!10^{3}$ & \underline{\textcolor{blue}{$(5.7 \pm 0.1)\!\times\!10^{2}$}} & $(5.7 \pm 0.1)\!\times\!10^{2}$ & $(1.6 \pm 0.2)\!\times\!10^{5}$ & $(3.5 \pm 2.2)\!\times\!10^{23}$ & $\times$ \\
\cmidrule(lr){2-11}
 & \multirow{2}{*}{Avg} & MSE & \textbf{0.231 \(\pm\) 0.01} & 0.242 \(\pm\) 0.01 & 0.249 \(\pm\) 0.01 & 0.234 \(\pm\) 0.01 & \underline{\textcolor{blue}{0.234 \(\pm\) 0.01}} & 0.248 \(\pm\) 0.01 & 0.257 \(\pm\) 0.01 & 0.248 \\
 &  & NLL & $(5.9 \pm 0.2)\!\times\!10^{2}$ & \textbf{$(3.7 \pm 0.1)\!\times\!10^{2}$} & $(1.5 \pm 0.0)\!\times\!10^{3}$ & \underline{\textcolor{blue}{$(4.3 \pm 0.1)\!\times\!10^{2}$}} & $(4.3 \pm 0.1)\!\times\!10^{2}$ & $(1.7 \pm 1.4)\!\times\!10^{14}$ & $(9.2 \pm 5.7)\!\times\!10^{22}$ & $\times$ \\
\midrule
\multirow{10}{*}{{\scriptsize Traffic}} & \multirow{2}{*}{96} & MSE & 0.503 \(\pm\) 0.02 & 0.528 \(\pm\) 0.02 & 0.556 \(\pm\) 0.03 & 0.540 \(\pm\) 0.03 & 0.540 \(\pm\) 0.03 & \textbf{0.472 \(\pm\) 0.02} & 0.477 \(\pm\) 0.02 & \underline{\textcolor{blue}{0.475}} \\
 &  & NLL & \textbf{$(1.9 \pm 0.1)\!\times\!10^{3}$} & \underline{\textcolor{blue}{$(2.5 \pm 0.1)\!\times\!10^{3}$}} & $(6.3 \pm 0.6)\!\times\!10^{4}$ & $(2.9 \pm 0.1)\!\times\!10^{3}$ & $(2.9 \pm 0.1)\!\times\!10^{3}$ & $(9.7 \pm 7.1)\!\times\!10^{13}$ & $\times$ & $\times$ \\
\cmidrule(lr){2-11}
 & \multirow{2}{*}{192} & MSE & 0.629 \(\pm\) 0.03 & 0.648 \(\pm\) 0.03 & 0.684 \(\pm\) 0.03 & \underline{\textcolor{blue}{0.622 \(\pm\) 0.03}} & \textbf{0.622 \(\pm\) 0.03} & 0.677 \(\pm\) 0.03 & 0.690 \(\pm\) 0.03 & 0.710 \\
 &  & NLL & \textbf{$(2.2 \pm 0.1)\!\times\!10^{3}$} & \underline{\textcolor{blue}{$(3.0 \pm 0.2)\!\times\!10^{3}$}} & $(6.8 \pm 0.8)\!\times\!10^{4}$ & $(4.3 \pm 0.2)\!\times\!10^{3}$ & $(4.3 \pm 0.2)\!\times\!10^{3}$ & $(1.4 \pm 1.1)\!\times\!10^{21}$ & $\times$ & $\times$ \\
\cmidrule(lr){2-11}
 & \multirow{2}{*}{336} & MSE & \textbf{0.515 \(\pm\) 0.02} & 0.528 \(\pm\) 0.03 & 0.567 \(\pm\) 0.03 & 0.517 \(\pm\) 0.02 & \underline{\textcolor{blue}{0.517 \(\pm\) 0.02}} & 0.552 \(\pm\) 0.02 & 0.556 \(\pm\) 0.03 & 0.547 \\
 &  & NLL & \textbf{$(2.1 \pm 0.1)\!\times\!10^{3}$} & \underline{\textcolor{blue}{$(2.7 \pm 0.2)\!\times\!10^{3}$}} & $(5.1 \pm 0.5)\!\times\!10^{4}$ & $(3.4 \pm 0.2)\!\times\!10^{3}$ & $(3.4 \pm 0.2)\!\times\!10^{3}$ & $(5.5 \pm 2.2)\!\times\!10^{14}$ & $(9.9 \pm 4.2)\!\times\!10^{21}$ & $\times$ \\
\cmidrule(lr){2-11}
 & \multirow{2}{*}{720} & MSE & 0.748 \(\pm\) 0.03 & 0.757 \(\pm\) 0.04 & 0.782 \(\pm\) 0.04 & \underline{\textcolor{blue}{0.741 \(\pm\) 0.03}} & \textbf{0.741 \(\pm\) 0.03} & 0.766 \(\pm\) 0.03 & 1.036 \(\pm\) 0.04 & 0.800 \\
 &  & NLL & \textbf{$(2.3 \pm 0.1)\!\times\!10^{3}$} & \underline{\textcolor{blue}{$(3.2 \pm 0.2)\!\times\!10^{3}$}} & $(8.0 \pm 0.9)\!\times\!10^{4}$ & $(4.8 \pm 0.2)\!\times\!10^{3}$ & $(4.8 \pm 0.2)\!\times\!10^{3}$ & $(2.9 \pm 0.9)\!\times\!10^{12}$ & $\times$ & $\times$ \\
\cmidrule(lr){2-11}
 & \multirow{2}{*}{Avg} & MSE & \textbf{0.599 \(\pm\) 0.03} & 0.615 \(\pm\) 0.03 & 0.647 \(\pm\) 0.03 & 0.605 \(\pm\) 0.03 & \underline{\textcolor{blue}{0.605 \(\pm\) 0.03}} & 0.617 \(\pm\) 0.03 & 0.690 \(\pm\) 0.03 & 0.633 \\
 &  & NLL & \textbf{$(2.2 \pm 0.1)\!\times\!10^{3}$} & \underline{\textcolor{blue}{$(2.9 \pm 0.2)\!\times\!10^{3}$}} & $(6.6 \pm 0.7)\!\times\!10^{4}$ & $(3.8 \pm 0.2)\!\times\!10^{3}$ & $(3.8 \pm 0.2)\!\times\!10^{3}$ & $(3.5 \pm 2.6)\!\times\!10^{20}$ & $(9.9 \pm 4.2)\!\times\!10^{21}$ & $\times$ \\
\bottomrule
\end{tabular}
\end{table*}

\subsection{Forecasting Comparison}

Figure \ref{fig:forecast_compare} presents a comparison of the forecasting performance of the Static, Noisy Diagonal, and MoE ensemble models with confidence intervals on the Electricity dataset for the forecasting horizon of 720. For the Static ensemble, the figure shows the contribution values assigned to each expert. For the Noisy Diagonal model, the figure presents the time-varying values associated with each expert together with the Top Share plot, which illustrates the extent to which the model concentrates the forecasting contribution on a single expert over time. For the MoE model, the figure shows the gating-layer weights after the softmax activation. Figure \ref{fig:forecast_compare_2} also shows, that on some datasets MoE can be overconfident and collapse into one expert, which can be not the best among experts for this data. 

\begin{figure*}[t]
    \centering
    \includegraphics[width=\linewidth]{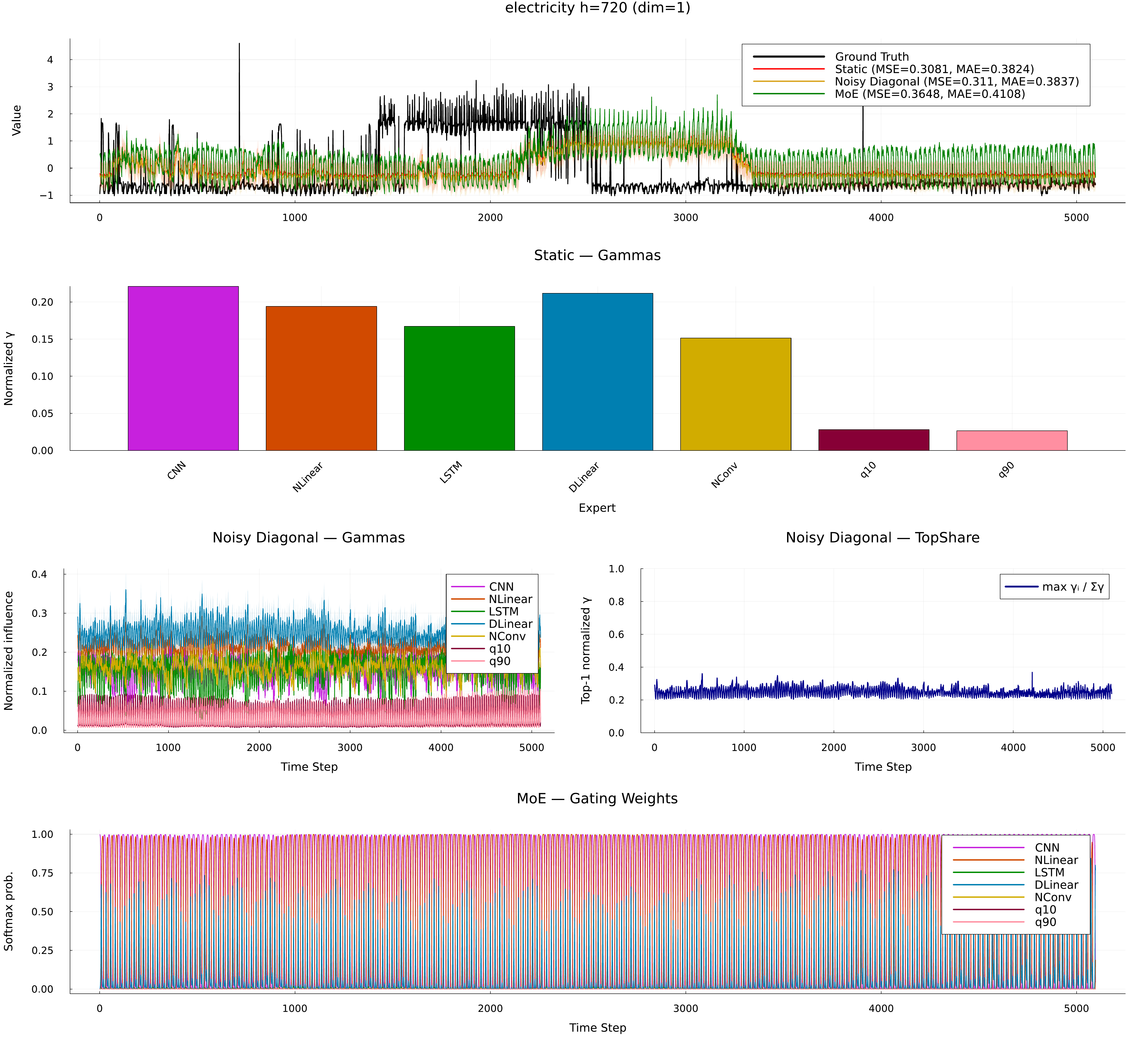}
    \caption{Comparison of forecasting with confidence interval of Static, Noisy Diagonal and MoE ensembles on electricity dataset and horizon 720. With values of each ensemble method for it's experts. }
    \label{fig:forecast_compare}
\end{figure*}

\begin{figure*}[t]
    \centering
    \includegraphics[width=\linewidth]{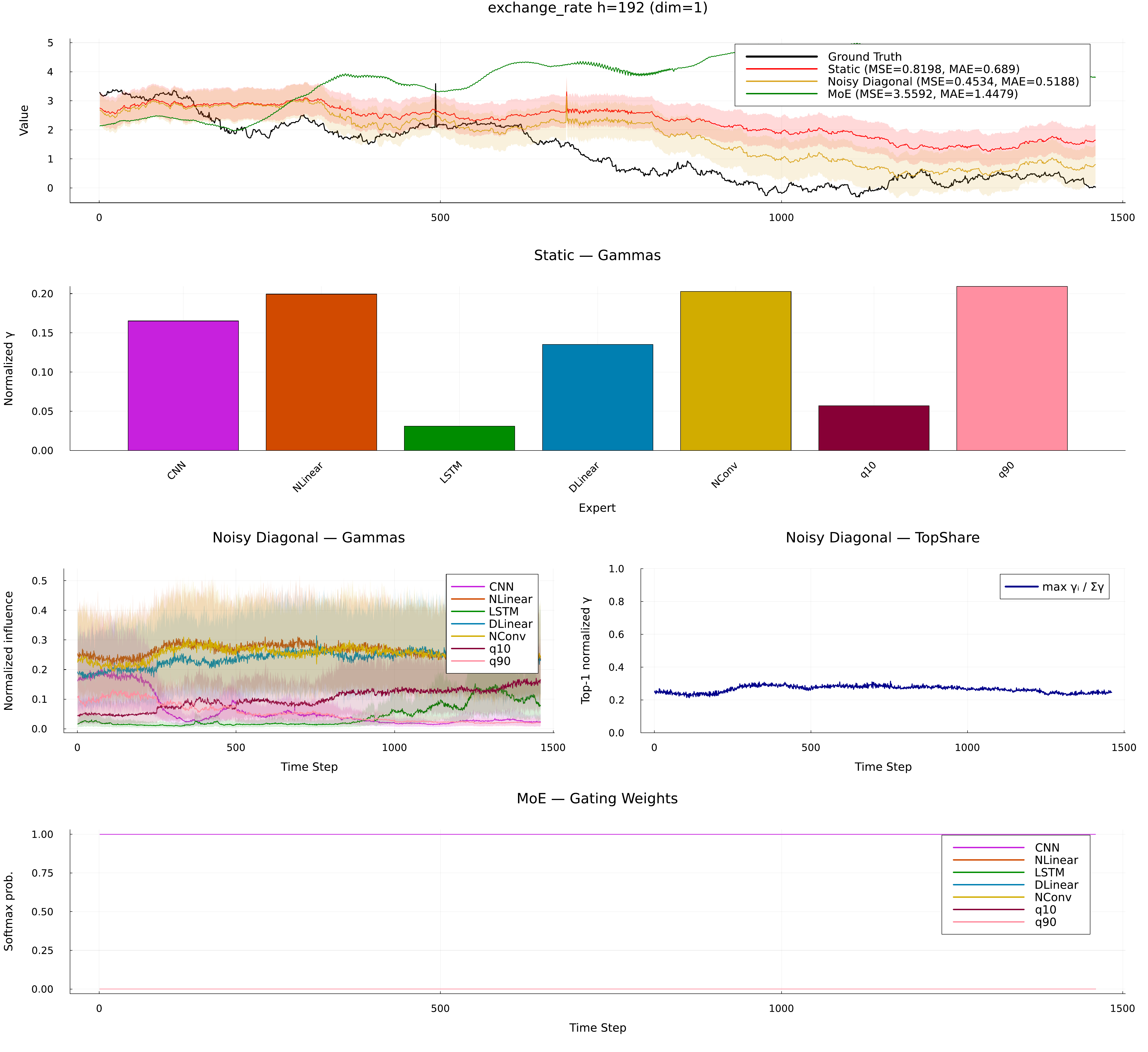}
    \caption{Comparison of forecasting with confidence interval of Static, Noisy Diagonal and MoE ensembles on exchange rate dataset and horizon 192. This example shows how classic neural network can collapse into one expert, which is not the best for this task. }
    \label{fig:forecast_compare_2}
\end{figure*}

%% file: main.bbl
\begin{thebibliography}{34}
\providecommand{\natexlab}[1]{#1}
\providecommand{\url}[1]{\texttt{#1}}
\expandafter\ifx\csname urlstyle\endcsname\relax
  \providecommand{\doi}[1]{doi: #1}\else
  \providecommand{\doi}{doi: \begingroup \urlstyle{rm}\Url}\fi

\bibitem[Bagaev and De~Vries(2023)]{bagaev_reactive_2023}
D.~Bagaev and B.~De~Vries.
\newblock Reactive {Message} {Passing} for {Scalable} {Bayesian} {Inference}.
\newblock \emph{Scientific Programming}, 2023:\penalty0 1--26, May 2023.
\newblock ISSN 1875-919X, 1058-9244.
\newblock \doi{10.1155/2023/6601690}.
\newblock URL \url{https://www.hindawi.com/journals/sp/2023/6601690/}.

\bibitem[Bergmann(2022)]{bergmann_manoptjl_2022}
R.~Bergmann.
\newblock Manopt.jl: {Optimization} on {Manifolds} in {Julia}.
\newblock \emph{Journal of Open Source Software}, 7\penalty0 (70):\penalty0
  3866, 2022.
\newblock \doi{10.21105/joss.03866}.

\bibitem[Cochran(1937)]{cochran_problems_1937}
W.~G. Cochran.
\newblock Problems arising in the analysis of a series of similar experiments.
\newblock \emph{Supplement to the Journal of the Royal Statistical Society},
  4\penalty0 (1):\penalty0 102--118, 1937.

\bibitem[Cybenko(1989)]{cybenko_approximation_1989}
G.~Cybenko.
\newblock Approximation by superpositions of a sigmoidal function.
\newblock \emph{Mathematics of Control, Signals and Systems}, 2\penalty0
  (4):\penalty0 303--314, Dec. 1989.
\newblock ISSN 1435-568X.
\newblock \doi{10.1007/BF02551274}.
\newblock URL \url{https://doi.org/10.1007/BF02551274}.

\bibitem[Dauwels(2007)]{dauwels_variational_2007}
J.~Dauwels.
\newblock On {Variational} {Message} {Passing} on {Factor} {Graphs}.
\newblock In \emph{{IEEE} {International} {Symposium} on {Information}
  {Theory}}, pages 2546--2550, Nice, France, June 2007.
\newblock \doi{10.1109/ISIT.2007.4557602}.
\newblock URL \url{http://ieeexplore.ieee.org/abstract/document/4557602}.

\bibitem[Forney(2001)]{forney_codes_2001}
G.~Forney.
\newblock Codes on graphs: normal realizations.
\newblock \emph{IEEE Transactions on Information Theory}, 47\penalty0
  (2):\penalty0 520--548, Feb. 2001.
\newblock ISSN 0018-9448.
\newblock \doi{10.1109/18.910573}.
\newblock URL \url{https://ieeexplore.ieee.org/abstract/document/910573}.

\bibitem[Ho et~al.(2020)Ho, Jain, and Abbeel]{ho_denoising_2020}
J.~Ho, A.~Jain, and P.~Abbeel.
\newblock Denoising diffusion probabilistic models.
\newblock \emph{Advances in neural information processing systems},
  33:\penalty0 6840--6851, 2020.

\bibitem[Hochreiter and Schmidhuber(1997)]{hochreiter_long_1997}
S.~Hochreiter and J.~Schmidhuber.
\newblock Long {Short}-{Term} {Memory}.
\newblock \emph{Neural Comput.}, 9\penalty0 (8):\penalty0 1735--1780, Nov.
  1997.
\newblock ISSN 0899-7667.
\newblock \doi{10.1162/neco.1997.9.8.1735}.
\newblock URL \url{https://doi.org/10.1162/neco.1997.9.8.1735}.

\bibitem[Hornik(1991)]{hornik_approximation_1991}
K.~Hornik.
\newblock Approximation capabilities of multilayer feedforward networks.
\newblock \emph{Neural Networks}, 4\penalty0 (2):\penalty0 251--257, 1991.
\newblock ISSN 0893-6080.
\newblock \doi{https://doi.org/10.1016/0893-6080(91)90009-T}.
\newblock URL
  \url{https://www.sciencedirect.com/science/article/pii/089360809190009T}.

\bibitem[Jacobs et~al.(1991)Jacobs, Jordan, Nowlan, and
  Hinton]{jacobs_adaptive_1991}
R.~A. Jacobs, M.~I. Jordan, S.~J. Nowlan, and G.~E. Hinton.
\newblock Adaptive mixtures of local experts.
\newblock \emph{Neural computation}, 3\penalty0 (1):\penalty0 79--87, 1991.

\bibitem[Khan(2025)]{khan_information_2025}
M.~E. Khan.
\newblock Information {Geometry} of {Variational} {Bayes}.
\newblock \emph{Information Geometry}, 8\penalty0 (S1):\penalty0 275--289, Nov.
  2025.
\newblock ISSN 2511-2481, 2511-249X.
\newblock \doi{10.1007/s41884-025-00174-3}.
\newblock URL \url{https://link.springer.com/10.1007/s41884-025-00174-3}.

\bibitem[Khan and Rue(2023)]{khan_bayesian_2023}
M.~E. Khan and H.~Rue.
\newblock The {Bayesian} learning rule.
\newblock \emph{Journal of Machine Learning Research}, 24\penalty0
  (281):\penalty0 1--46, 2023.

\bibitem[Kingma and Welling(2013)]{kingma_autoencoding_2013}
D.~P. Kingma and M.~Welling.
\newblock Auto-{Encoding} {Variational} {Bayes}.
\newblock \emph{arXiv:1312.6114 [cs, stat]}, Dec. 2013.
\newblock URL \url{http://arxiv.org/abs/1312.6114}.
\newblock arXiv: 1312.6114.

\bibitem[Kschischang et~al.(2001)Kschischang, Frey, and
  Loeliger]{kschischang_factor_2001}
F.~R. Kschischang, B.~J. Frey, and H.-A. Loeliger.
\newblock Factor graphs and the sum-product algorithm.
\newblock \emph{IEEE Transactions on information theory}, 47\penalty0
  (2):\penalty0 498--519, 2001.
\newblock \doi{10.1109/18.910572}.
\newblock URL
  \url{http://ieeexplore.ieee.org/xpls/abs_all.jsp?arnumber=910572}.

\bibitem[LeCun et~al.(2015)LeCun, Bengio, and Hinton]{lecun_deep_2015}
Y.~LeCun, Y.~Bengio, and G.~Hinton.
\newblock Deep learning.
\newblock \emph{Nature}, 521\penalty0 (7553):\penalty0 436--444, May 2015.
\newblock ISSN 0028-0836, 1476-4687.
\newblock \doi{10.1038/nature14539}.
\newblock URL \url{https://www.nature.com/articles/nature14539}.

\bibitem[Loeliger(2004)]{loeliger_introduction_2004}
H.-A. Loeliger.
\newblock An introduction to factor graphs.
\newblock \emph{Signal Processing Magazine, IEEE}, 21\penalty0 (1):\penalty0
  28--41, Jan. 2004.
\newblock \doi{10.1109/MSP.2004.1267047}.
\newblock URL \url{https://ieeexplore.ieee.org/document/1267047}.

\bibitem[Loeliger(2007)]{loeliger_factor_2007}
H.-A. Loeliger.
\newblock Factor {Graphs} and {Message} {Passing} {Algorithms} -- {Part} 1:
  {Introduction}, 2007.
\newblock URL \url{http://www.crm.sns.it/media/course/1524/Loeliger_A.pdf}.

\bibitem[Loshchilov and Hutter(2019)]{loshchilov_decoupled_2019}
I.~Loshchilov and F.~Hutter.
\newblock Decoupled {Weight} {Decay} {Regularization}.
\newblock In \emph{7th {International} {Conference} on {Learning}
  {Representations}, {ICLR} 2019, {New} {Orleans}, {LA}, {USA}, {May} 6-9,
  2019}. OpenReview.net, 2019.
\newblock URL \url{https://openreview.net/forum?id=Bkg6RiCqY7}.

\bibitem[Lukashchuk et~al.(2024)Lukashchuk, Senöz, and
  de~Vries]{lukashchuk_qconjugate_2024}
M.~Lukashchuk, I.~Senöz, and B.~de~Vries.
\newblock Q-conjugate message passing for efficient bayesian inference.
\newblock In \emph{International conference on probabilistic graphical models},
  pages 295--311. PMLR, 2024.

\bibitem[Lukashchuk et~al.(2025)Lukashchuk, Bagaev, Podusenko, {\c{S}}en{\"o}z,
  and de~Vries]{lukashchuk_exponentialfamilymanifoldsjl_2025}
M.~Lukashchuk, D.~Bagaev, A.~Podusenko, {\.{I}}.~{\c{S}}en{\"o}z, and
  B.~de~Vries.
\newblock {ExponentialFamilyManifolds}.jl: {Representing} exponential families
  as {Riemannian} manifolds.
\newblock \emph{Proceedings of the JuliaCon Conferences}, 7\penalty0
  (70):\penalty0 179, 2025.
\newblock \doi{10.21105/jcon.00179}.
\newblock URL \url{https://doi.org/10.21105/jcon.00179}.

\bibitem[Neal(2011)]{neal_mcmc_2011}
R.~M. Neal.
\newblock \emph{{MCMC} using {Hamiltonian} dynamics}.
\newblock May 2011.
\newblock \doi{10.1201/b10905}.
\newblock URL \url{http://arxiv.org/abs/1206.1901}.
\newblock arXiv:1206.1901 [physics, stat].

\bibitem[Nuijten et~al.(2024)Nuijten, Bagaev, and
  de~Vries]{nuijten_graphppljl_2024}
W.~W.~L. Nuijten, D.~Bagaev, and B.~de~Vries.
\newblock {GraphPPL}.jl: {A} {Probabilistic} {Programming} {Language} for
  {Graphical} {Models}.
\newblock \emph{Entropy}, 26\penalty0 (11), 2024.
\newblock ISSN 1099-4300.
\newblock \doi{10.3390/e26110890}.
\newblock URL \url{https://www.mdpi.com/1099-4300/26/11/890}.

\bibitem[Ranganath et~al.(2014)Ranganath, Gerrish, and
  Blei]{ranganath_black_2014}
R.~Ranganath, S.~Gerrish, and D.~Blei.
\newblock Black {Box} {Variational} {Inference}.
\newblock In S.~Kaski and J.~Corander, editors, \emph{Proceedings of the
  {Seventeenth} {International} {Conference} on {Artificial} {Intelligence} and
  {Statistics}}, volume~33 of \emph{Proceedings of {Machine} {Learning}
  {Research}}, pages 814--822, Reykjavik, Iceland, Apr. 2014. PMLR.
\newblock URL \url{https://proceedings.mlr.press/v33/ranganath14.html}.

\bibitem[Rezende and Mohamed(2015)]{rezende_variational_2015}
D.~J. Rezende and S.~Mohamed.
\newblock Variational {Inference} with {Normalizing} {Flows}.
\newblock \emph{arXiv:1505.05770 [cs, stat]}, May 2015.
\newblock URL \url{http://arxiv.org/abs/1505.05770}.
\newblock arXiv: 1505.05770.

\bibitem[Rudin(2013)]{rudin_real_2013}
W.~Rudin.
\newblock \emph{Real and complex analysis}.
\newblock {McGraw}-{Hill} international editions {Mathematics} series.
  McGraw-Hill, New York, NY, 3. ed., internat. ed., [nachdr.] edition, 2013.
\newblock ISBN 978-0-07-100276-9 978-0-07-054234-1.
\newblock OCLC: 957461070.

\bibitem[Senöz et~al.(2021)Senöz, van~de Laar, Bagaev, and
  de~Vries]{senoz_variational_2021}
I.~Senöz, T.~van~de Laar, D.~Bagaev, and B.~de~Vries.
\newblock Variational {Message} {Passing} and {Local} {Constraint}
  {Manipulation} in {Factor} {Graphs}.
\newblock \emph{Entropy}, 23\penalty0 (7):\penalty0 807, July 2021.
\newblock ISSN 1099-4300.
\newblock \doi{10.3390/e23070807}.
\newblock URL \url{https://www.mdpi.com/1099-4300/23/7/807}.

\bibitem[Smola et~al.(2003)Smola, Vishwanathan, and Eskin]{smola_laplace_2003}
A.~Smola, S.~Vishwanathan, and E.~Eskin.
\newblock Laplace propagation.
\newblock In \emph{Advances in neural information processing systems},
  volume~16. MIT Press, 2003.
\newblock URL
  \url{https://proceedings.neurips.cc/paper_files/paper/2003/file/7fd804295ef7f6a2822bf4c61f9dc4a8-Paper.pdf}.

\bibitem[Trindade(2015)]{trindade_electricityloaddiagrams20112014_2015}
A.~Trindade.
\newblock {ElectricityLoadDiagrams20112014}.
\newblock \emph{UCI Machine Learning Repository}, 10:\penalty0 C58C86, 2015.

\bibitem[van~de Laar et~al.(2018)van~de Laar, Cox, Senoz, Bocharov, and
  de~Vries]{vandelaar_forneylab_2018}
T.~van~de Laar, M.~Cox, I.~Senoz, I.~Bocharov, and B.~de~Vries.
\newblock {ForneyLab}: {A} {Toolbox} for {Biologically} {Plausible} {Free}
  {Energy} {Minimization} in {Dynamic} {Neural} {Models}.
\newblock In \emph{Conference on {Complex} {Systems} ({CCS})}, Thessaloniki,
  Greece, Sept. 2018.

\bibitem[Weber et~al.(2026)Weber, Waade, Legrand, Møller, Stephan, and
  Mathys]{weber_generalized_2026}
L.~A. Weber, P.~T. Waade, N.~Legrand, A.~H. Møller, K.~E. Stephan, and
  C.~Mathys.
\newblock The generalized {Hierarchical} {Gaussian} {Filter}.
\newblock Mar. 2026.
\newblock \doi{10.7554/elife.110174.1}.
\newblock URL \url{http://dx.doi.org/10.7554/eLife.110174.1}.

\bibitem[Winn and Bishop(2005)]{winn_variational_2005}
J.~Winn and C.~M. Bishop.
\newblock Variational {Message} {Passing}.
\newblock \emph{Journal of Machine Learning Research}, 6\penalty0
  (23):\penalty0 661--694, 2005.
\newblock ISSN 1533-7928.
\newblock URL \url{http://jmlr.org/papers/v6/winn05a.html}.

\bibitem[Yedidia et~al.(2005)Yedidia, Freeman, and
  Weiss]{yedidia_constructing_2005}
J.~S. Yedidia, W.~Freeman, and Y.~Weiss.
\newblock Constructing free-energy approximations and generalized belief
  propagation algorithms.
\newblock \emph{IEEE Transactions on Information Theory}, 51\penalty0
  (7):\penalty0 2282--2312, July 2005.
\newblock ISSN 0018-9448.
\newblock \doi{10.1109/TIT.2005.850085}.
\newblock URL \url{http://ieeexplore.ieee.org/abstract/document/1459044}.

\bibitem[Zeng et~al.(2023)Zeng, Chen, Zhang, and Xu]{zeng_are_2023}
A.~Zeng, M.~Chen, L.~Zhang, and Q.~Xu.
\newblock Are transformers effective for time series forecasting?
\newblock In \emph{Proceedings of the {AAAI} conference on artificial
  intelligence}, volume~37, pages 11121--11128, 2023.

\bibitem[Zhou et~al.(2021)Zhou, Zhang, Peng, Zhang, Li, Xiong, and
  Zhang]{zhou_informer_2021}
H.~Zhou, S.~Zhang, J.~Peng, S.~Zhang, J.~Li, H.~Xiong, and W.~Zhang.
\newblock Informer: {Beyond} {Efficient} {Transformer} for {Long} {Sequence}
  {Time}-{Series} {Forecasting}.
\newblock In \emph{The {Thirty}-{Fifth} {AAAI} {Conference} on {Artificial}
  {Intelligence}, {AAAI} 2021, {Virtual} {Conference}}, volume~35, pages
  11106--11115. AAAI Press, 2021.

\end{thebibliography}
